%% file: arxiv.tex
\documentclass[10pt,journal,compsoc]{IEEEtran}

\usepackage{amsmath,amsfonts}
\usepackage{algorithmic}
\usepackage{algorithm}
\usepackage{array}
\usepackage[caption=false,font=footnotesize,labelfont=rm,textfont=rm]{subfig}
\usepackage{textcomp}
\usepackage{stfloats}
\usepackage{url}
\usepackage{verbatim}
\usepackage{graphicx}
\usepackage{xspace}
\usepackage{amsthm}
\usepackage{booktabs}
\usepackage{caption}
\usepackage{makecell}
\usepackage{blindtext}
\usepackage{bm}
\usepackage{bbm} 
\usepackage{comment}
\usepackage{multirow}
\usepackage{enumitem}
\usepackage{setspace}
\usepackage[safe]{tipa} 
\usepackage{color,xcolor,colortbl}
\usepackage[misc,geometry]{ifsym} 
\usepackage{pifont}
\usepackage{wrapfig}
\usepackage{ulem}
\usepackage{cancel}
\usepackage{soul}
\usepackage{ragged2e}

\usepackage[hidelinks]{hyperref}

\theoremstyle{plain}
\newtheorem{theorem}{Theorem}[section]
\newtheorem{proposition}[theorem]{Proposition}

\theoremstyle{definition}

\theoremstyle{remark}
\newtheorem{remark}[theorem]{Remark}

\input{math_commands.tex}

\providecommand{\eg}{\textit{e.g.}\@\xspace}
\providecommand{\ie}{\textit{i.e.}\@\xspace}

\usepackage[capitalize,noabbrev]{cleveref}
\Crefname{section}{Section}{Sections}
\crefname{section}{Sec.}{Sec.}
\Crefname{table}{Table}{Tables}
\crefname{table}{Tab.}{Tab.}
\Crefname{figure}{Figure}{Figures}
\crefname{figure}{Fig.}{Fig.}
\Crefname{equation}{Equation}{Equations}
\crefname{equation}{Eq.}{Eq.}

\def\oP{{\operatorname{P}}}

\definecolor{mplgreen}{rgb}{0.17254901960784313, 0.6274509803921569, 0.17254901960784313}
\definecolor{myblue}{rgb}{0.2, 0.2, 0.9}
\definecolor{myred}{rgb}{1.0, 0.1, 0.1}

\urldef\rrcurl\url{https://pytorch.org/vision/stable/generated/torchvision.transforms.RandomResizedCrop.html}
\urldef\imageneturl\url{https://storage.googleapis.com/bit_models/imagenet21k_wordnet_lemmas.txt}
\urldef\cifarurl\url{https://www.cs.toronto.edu/%7Ekriz/cifar.html}
\urldef\adaptformerurl\url{https://github.com/ShoufaChen/AdaptFormer/blob/main/models/adapter.py}

\newcolumntype{L}[1]{>{\raggedright\arraybackslash}p{#1}}
\newcolumntype{R}[1]{>{\raggedleft\arraybackslash}p{#1}}
\newcolumntype{C}[1]{>{\centering\arraybackslash}p{#1}}

\captionsetup[figure]{aboveskip=0.05in, belowskip=0.05in}
\captionsetup[table]{aboveskip=0.05in, belowskip=0.05in}

\allowdisplaybreaks

\def\algo{{LIFT+}}

\ifCLASSOPTIONcompsoc
  \usepackage[nocompress]{cite}
\else
  \usepackage{cite}
\fi
\ifCLASSINFOpdf
\else
\fi

\hyphenation{op-tical net-works semi-conduc-tor IEEE-Xplore}

\begin{document}
%
\title{\algo: Lightweight Fine-Tuning for Long-Tail Learning}

%
%
\author{Jiang-Xin Shi,
        Tong Wei,
        and Yu-Feng Li,~\IEEEmembership{Senior Member,~IEEE}
\IEEEcompsocitemizethanks{%
\IEEEcompsocthanksitem Jiang-Xin Shi and Yu-Feng Li are with the National Key Laboratory for Novel Software Technology, Nanjing University, Nanjing 210023, China, and the School of Artificial Intelligence, Nanjing University, Nanjing 210023, China.\protect\\
E-mail: \{shijx, liyf\}@lamda.nju.edu.cn.
\IEEEcompsocthanksitem Tong Wei is with the School of Computer Science and Engineering, Southeast University, Nanjing 210096, China, and the Key Laboratory of Computer Network and Information Integration, Southeast University, Ministry of Education, China.\protect\\
E-mail: weit@seu.edu.cn.
}
\thanks{%
This work has been submitted to the IEEE for peer review.\\
(Corresponding author: Yu-Feng Li.)
}}

%
%

\markboth{}{}%
%


\IEEEtitleabstractindextext{%
\begin{abstract}
\justifying\let\raggedright\justifying
The fine-tuning paradigm has emerged as a prominent approach for addressing long-tail learning tasks in the era of foundation models. However, the impact of fine-tuning strategies on long-tail learning performance remains unexplored. In this work, we disclose that existing paradigms exhibit a profound misuse of fine-tuning methods, leaving significant room for improvement in both efficiency and accuracy.
Specifically, we reveal that heavy fine-tuning (fine-tuning a large proportion of model parameters) can lead to non-negligible performance deterioration on tail classes, whereas lightweight fine-tuning demonstrates superior effectiveness. Through comprehensive theoretical and empirical validation, we identify this phenomenon as stemming from inconsistent class conditional distributions induced by heavy fine-tuning.
Building on this insight, we propose \algo, an innovative lightweight fine-tuning framework to optimize consistent class conditions. Furthermore, \algo\ incorporates semantic-aware initialization, minimalist data augmentation, and test-time ensembling to enhance adaptation and generalization of foundation models.
Our framework provides an efficient and accurate pipeline that facilitates fast convergence and model compactness. Extensive experiments demonstrate that \algo\ significantly reduces both training epochs (from $\sim$100 to $\leq$15) and learned parameters (less than 1\%), while surpassing state-of-the-art approaches by a considerable margin.
The source code is available at \url{https://github.com/shijxcs/LIFT-plus}.
\end{abstract}

\begin{IEEEkeywords}
Long-tail learning, foundation model, class-imbalanced learning, vision-language model.
\end{IEEEkeywords}}

\let\oldtwocolumn\twocolumn
\renewcommand\twocolumn[1][]{%
\oldtwocolumn[{#1}{
    \vspace{-0.3in}
    \centering
    \includegraphics[width=0.325\linewidth]{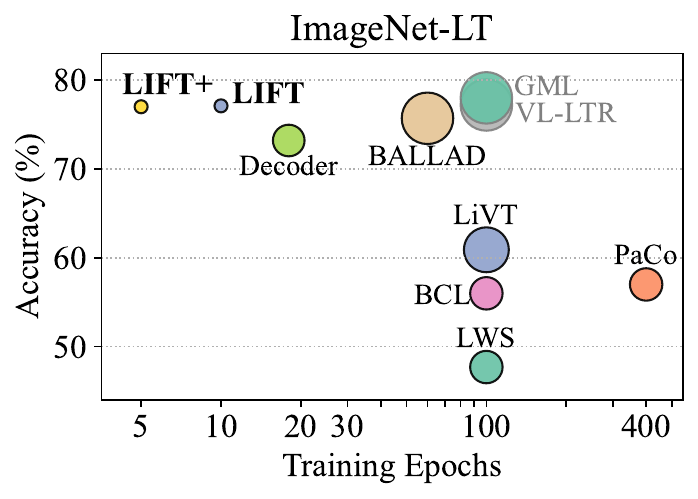}
    \hfill
    \includegraphics[width=0.325\linewidth]{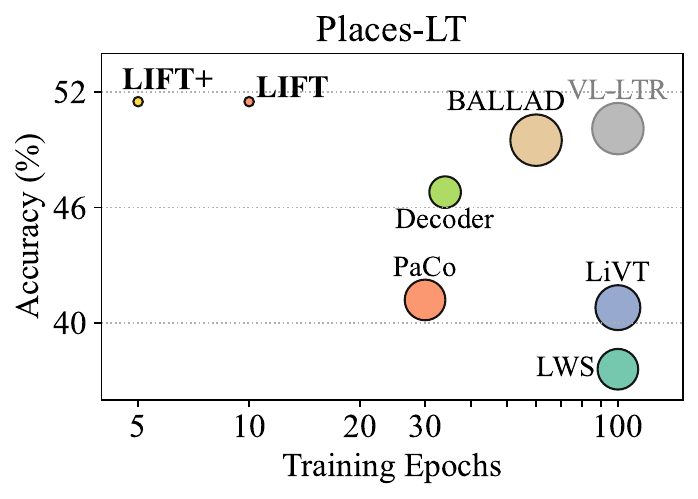}
    \hfill
    \includegraphics[width=0.325\linewidth]{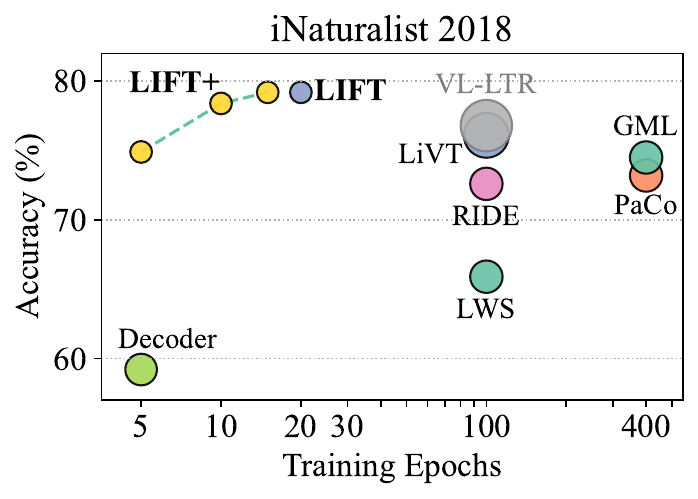}
    \captionof{figure}{Comparison of \algo\ with state-of-the-art methods (using CLIP as foundation model). The x-axis indicates the number of training epochs, and the y-axis represents the test accuracy. The size of each marker reflects the quantity of learned parameters. Gray labels denote methods that utilize external data. \algo\ consistently achieves higher performance with lower costs and is even comparable with methods that leverage external data.}
    \vspace{0.3in}
    \label{fig:sota}
}]
}

\maketitle

\IEEEdisplaynontitleabstractindextext

%
\IEEEpeerreviewmaketitle

\section{Introduction}\label{sec:introduction}

\IEEEPARstart{L}{ong-tail} learning tackles the challenge of learning from highly imbalanced data, where a small subset of dominant classes (head classes) possess abundant training samples, while the remaining classes (tail classes) have only limited training samples. The inherent challenge posed by long-tail data is prevalent across various scenarios, such as image classification \cite{wang2017learning,van2018inaturalist,liu2019large}, instance segmentation \cite{gupta2019lvis,wang2020devil,wang2021seesaw}, and object detection \cite{li2020overcoming,wang2021adaptive,li2022equalized}. This problem has attracted considerable research attention, leading to numerous approaches proposed to enhance generalization, particularly for tail classes. Existing approaches can be categorized into three main paradigms: \romannumeral 1) data manipulation techniques \cite{zhou2020bbn,chou2020remix,yang2020rethinking,he2021distilling,park2022majority,shi2023re,ahn2023cuda,gao2023enhancing}, \romannumeral 2) representation learning strategies \cite{kang2021exploring,wang2021contrastive,cui2021parametric,zhu2022balanced,liu2022selfsupervised,yang2022inducing,peifeng2023feature,ma2023curvature,kukleva2023temperature,gao2024distribution}, and \romannumeral 3) model output adjustment mechanisms \cite{cao2019learning,ren2020balanced,menon2021longtail,hong2021disentangling,zhang2021distribution,samuel2021distributional,wei2022robust,han2023wrapped}. Despite the considerable progress achieved by these methods, a significant performance gap persists compared to models trained on rich data, indicating room for substantial improvement.

Recent studies in long-tail learning have shown that fine-tuning foundation models such as CLIP \cite{radford2021clip} and ViT \cite{dosovitskiy2021an} can substantially improve performance, eliminating the necessity of training deep neural networks from scratch.
For instance, BALLAD \cite{ma2021simple}, VL-LTR \cite{tian2022vl}, Wang et al. \cite{wang2023exploring} employ the CLIP model, while RAC \cite{long2022retrieval} and LPT \cite{dong2023lpt} adopt the ViT model pre-trained on ImageNet-21K. However, although these works introduce a new paradigm for long-tail learning, they fail to provide a systematic analysis of the fine-tuning strategies under long-tail scenarios. This oversight potentially leads to suboptimal methodologies and unnecessary computational overhead. Specifically, existing approaches typically come at the cost of \romannumeral 1) prolonged training durations ($\approx 100$ epochs); \romannumeral 2) a two-stage procedure, which limits the deployment flexibility; \romannumeral 3) large external datasets ($\approx 10^6$ samples) to facilitate the training process.

To bridge this gap, we first provide a comprehensive analysis of fine-tuning strategies for long-tail learning. Our empirical evidence reveals that {heavy fine-tuning} (fine-tuning a large proportion of model parameters) can lead to \textit{non-negligible performance deterioration} on tail classes. Through systematic investigation, we identify that full fine-tuning distorts the intrinsic intra-class distance distributions. Theoretically, we prove that such distortions break the consistency assumption of class-conditional distributions, consequently resulting in biased predictions.
Building upon these theoretical and empirical findings, we recognize that optimizing a small proportion of pre-trained weights could simultaneously enhance model discriminability while preserving the intra-class distributions.
Motivated by this insight, we propose an efficient and accurate long-tail learning framework named \textit{LIghtweight Fine-Tuning} (\algo). The proposed approach facilitates rapid convergence and compact models through adaptive lightweight fine-tuning.

\algo\ is a single-stage framework that usually achieves convergence in fewer than 15 training epochs without requiring external data. Furthermore, we improve \algo\ with semantic-aware classifier initialization, which provides a robust starting point for the model, enabling better generalization while maintaining computational efficiency.
\algo\ also introduces a minimalist data augmentation strategy to eliminate trivial augmentation approaches and further accelerate the training period. 
Additionally, we incorporate test-time ensembling to mitigate the inherent limitations of foundation models and boost generalization capability.
\Cref{fig:sota} gives a performance comparison on three typical long-tail datasets, ImageNet-LT, Places-LT, and iNaturalist 2018. Notably, \algo\ consistently outperforms state-of-the-art methods without auxiliary data by an average margin of 2.1\% in accuracy, while requiring substantially fewer learned parameters and reduced training epochs. 

In conclusion, the contributions of this paper can be summarized as follows: 
\begin{itemize}
\item We first reveal the critical limitation of heavy fine-tuning that distorts the tail-class performance;
\item Through empirical and theoretical validation, we discover that optimizing a small proportion of parameters can effectively alleviate the performance degradation;
\item We propose \algo, an efficient and accurate long-tail learning framework that leverages lightweight fine-tuning with minimal computational overhead;
\item Extensive experiments demonstrate that \algo\ consistently outperforms state-of-the-art methods with lower computational costs on multiple benchmarks (an average of 2.1\% accuracy improvements, with fewer than 15 training epochs and less than 1\% learned parameters).
\end{itemize}

This paper presents an extended version of our previous work, LIFT, which was initially proposed at ICML 2024 \cite{shi2024longtail}. The current version offers four major advancements. First, we introduce \algo, a more systematic framework that addresses long-tail learning challenges through input processing, representation enhancement, and output optimization. This unified framework establishes a versatile research pipeline for future studies. Second, we propose a novel minimalist data augmentation strategy that fundamentally overcomes the limitations of conventional approaches, reducing training time by 5 epochs (25-50\% of the original training cost) while maintaining superior performance. Third, we extend our analysis to include more foundation models as baselines, specifically the ImageNet-21K pre-trained ViT model. We also conduct more comprehensive empirical studies to verify the advantages of our framework. Last but not least, we reconstruct our source code by rectifying some mistakes and further improving the computational efficiency, enabling full reproducibility using a single NVIDIA 4090 GPU with only 24GB of memory. We rerun all experiments and find that the performance is further improved in the updated pipeline.

\section{Related Work}
\textbf{Long-Tail Learning via Deep Learning.} Traditional approaches typically employ convolutional neural networks such as ResNet and ResNeXt for long-tail learning \cite{yang2022survey,zhang2023deep}. To address the long-tail challenge, existing methods follow three primary directions:
\romannumeral 1) data manipulation \cite{zhou2020bbn,chou2020remix,yang2020rethinking,he2021distilling,park2022majority,shi2023re,ahn2023cuda,gao2023enhancing}, which involves designing re-sampling and data augmentation strategies. Interestingly, Shi et al. \cite{shi2023re} discover that re-sampling is sensitive to irrelevant context and may not always help enhance representation. Ahn et al. \cite{ahn2023cuda} find that data augmentation may adversely affect the augmented class while benefiting the non-augmented classes.
\romannumeral 2) representation learning \cite{kang2021exploring,wang2021contrastive,cui2021parametric,zhu2022balanced,liu2022selfsupervised,yang2022inducing,peifeng2023feature,ma2023curvature,kukleva2023temperature,gao2024distribution}, which focuses on improving feature extraction. This direction incorporates various advanced techniques employed, such as supervised contrastive learning \cite{kang2021exploring,wang2021contrastive,cui2021parametric,zhu2022balanced}, self-supervised learning \cite{liu2022selfsupervised,kukleva2023temperature}, and neural collapse \cite{yang2022inducing,peifeng2023feature,gao2024distribution}. 
\romannumeral 3) model output adjustment \cite{cao2019learning,ren2020balanced,menon2021longtail,hong2021disentangling,zhang2021distribution,samuel2021distributional,wei2022robust,han2023wrapped}, which optimizes unbiased loss functions during training or applies post-hoc calibrations. 
Moreover, ensembling learning methods \cite{xiang2020learning,wang2021longtailed,cui2022reslt,zhang2022self,shi2024residual} aim to combine multiple diverse experts and optimization objectives to improve both head and tail classes.
Additionally, many studies adopt a two-stage training paradigm \cite{kang2020decoupling,zhong2021improving,wei2023towards,nam2023decoupled}, in which the first stage learns representations and the second stage learns the classifier .
In contrast to the aforementioned works, our work presents an end-to-end training approach that leverages the strengths of foundation models. Moreover, we propose novel techniques to enhance input processing, feature representation, and output optimization within a unified framework, providing a versatile and extensible pipeline for future research.

\textbf{Long-Tail Learning via Foundation Model.} Fine-tuning foundation models such as CLIP \cite{radford2021clip} and ViT \cite{dosovitskiy2021an} has attracted widespread attention \cite{steiner2022how,zhou2022learning,zhou2022conditional,yu2023visual,zhou2024decoop}, and has emerged as an effective strategy to mitigate class imbalance by leveraging their powerful representation learning capabilities \cite{ma2021simple,long2022retrieval,tian2022vl,iscen2023improving,dong2023lpt,xia2023lmpt,he2023uniformly,song2023long,wang2023exploring,li2024rectify,shi2024longtail}.
For instance, BALLAD \cite{ma2021simple} first employs full fine-tuning for the foundation model, then freezes the backbone and optimizes a linear adapter on re-sampled data.
VL-LTR \cite{tian2022vl} first incorporates auxiliary image-text web data to fine-tune the vision-language models, then freezes the text encoder, and optimizes the image encoder along with a language-guided head.
RAC \cite{long2022retrieval} jointly trains an encoder with a retrieval module to leverage external datasets (\eg, ImageNet-21K) for input augmentation.
LPT \cite{dong2023lpt} adopts a two-phase prompt learning methodology to adapt the foundation model.
Wang et al. \cite{wang2023exploring} propose to add a decoder after the pre-trained model to extract relevant features.
Although various approaches have been proposed for adapting foundation models, current research lacks a systematic analysis regarding the impact of fine-tuning strategies on long-tail learning. 
Moreover, it is important to note that existing methods typically suffer from prolonged training durations and often rely on auxiliary external data.
In contrast, our proposed approach exhibits a remarkable capability to achieve fast convergence without requiring external data. 
Furthermore, our method is versatile and inclusive, allowing for seamless integration with different lightweight fine-tuning methods and various backbones.

\begin{figure*}[!t]
    \centering
    \subfloat[]{
        \includegraphics[width=0.37\linewidth, trim=0 9 0 0]{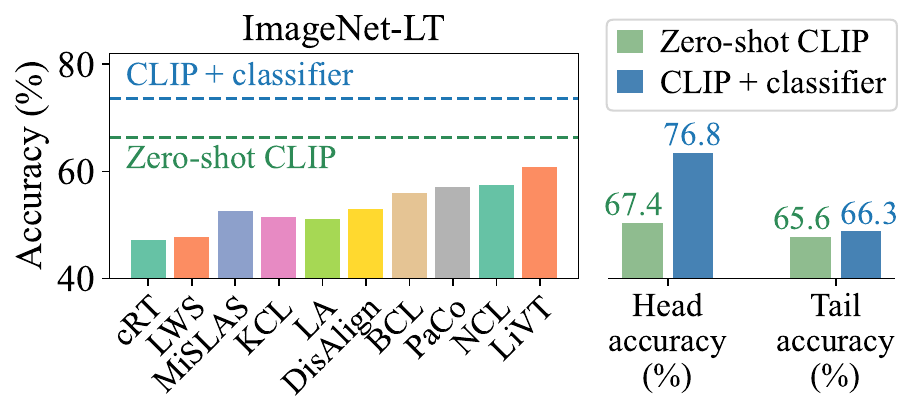}
        \label{fig:zsclip-a}
    }
    \hfill
    \subfloat[]{
        \includegraphics[width=0.352\linewidth, trim=0 9 0 0]{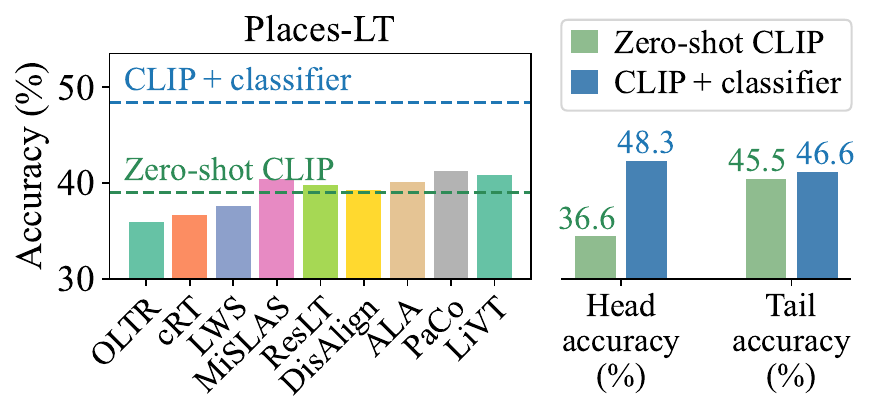}
        \label{fig:zsclip-b}
    }
    \hfill
    \subfloat[]{
        \includegraphics[width=0.202\linewidth, trim=0 0 0 0]{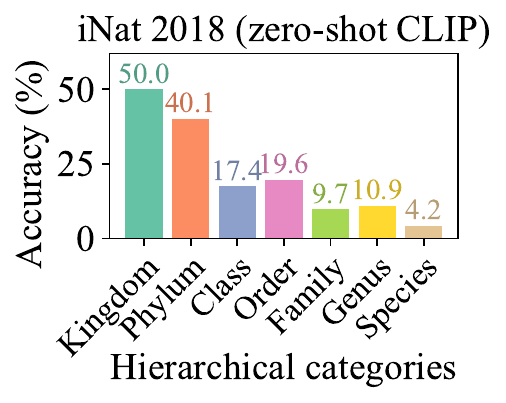}
        \label{fig:zsclip-c}
    }
    \caption{(a-b) On ImageNet-LT and Places-LT, zero-shot CLIP has surpassed many prior methods. By simply introducing an additional classifier, the accuracy further increases. However, the improvements mainly come from the head classes, while the tail classes only achieve marginal enhancements. (c) On iNaturalist 2018, zero-shot CLIP encounters challenges in achieving high accuracy for fine-grained long-tail categories.}
    \label{fig:zsclip}
\end{figure*}

\section{How Fine-Tuning Affects Long-Tail}

\subsection{Preliminary}
Unlike conventional neural networks, foundation models that employ the Transformer architecture \cite{vaswani2017attention,dosovitskiy2021an} are more elegantly designed while exhibiting exceptional generalization capabilities. It has demonstrated remarkable adaptability across various computer vision \cite{dosovitskiy2021an,he2022masked} and natural language processing tasks \cite{vaswani2017attention,kenton2019bert}.
Taking the image classification task as an example, given an image $\vx$, it first partitions the image into multiple patches, then feeds the patches into an extractor $\phi$ to obtain the corresponding feature $\phi(\vx)$. For a downstream $K-$classification task, the model computes the predicted logits for each class as $z_{k}=\phi(\vx)\vw_k^{\top}+b_k$, where $\vw_k$ and $b_k$ denote task-specific classifier weights and bias.
Due to space constraints, the comprehensive architecture and inference processes of foundation models are presented in \Cref{sec:tech_detail}.

The recent vision-language foundation model, CLIP, further underscores its efficacy by demonstrating impressive zero-shot performance \cite{radford2021clip}.
Given an image $\vx$ and considering $K$ candidate classes, CLIP first generates $K$ textual prompts $\vt_1,\cdots,\vt_K$, each representing a descriptive phrase, such as ``a photo of a cat'' or ``a photo of a dog''. It then extracts the image feature $\phi(\vx)$ and the prompt features $\psi(\vt_1),\cdots,\psi(\vt_K)$. To predict the label for the given image, CLIP compares the cosine similarity between the image and each class prompt:
\begin{equation}
\label{eq:clip}
    y_{\text{pred}}=\operatorname*{arg\  max}\limits_{k\in[K]} \langle\phi(\vx)\mP_I,\psi(\vt_k)\mP_T\rangle
\end{equation}
where $\langle\cdot,\cdot\rangle$ denotes cosine similarity. $\mP_I$ and $\mP_T$ are projection matrices that map the extracted features to a shared latent space with consistent dimensionality.

\subsection{Long-Tail Learning with Foundation Model}
We first evaluate the performance of CLIP on typical long-tail datasets, as illustrated in \Cref{fig:zsclip-a,fig:zsclip-b,fig:zsclip-c}. The results reveal that zero-shot CLIP outperforms most conventional methods. Furthermore, by freezing the backbone and training an additional classifier, the performance can be further improved, which validates the high-quality representations learned by CLIP. However, the majority of performance gains are dominated by the head classes, while the tail classes exhibit only marginal improvements.

Moreover, while zero-shot CLIP exhibits impressive performance in general scenarios, its effectiveness diminishes on specialized long-tail datasets. A notable case is the iNaturalist 2018 dataset, which presents a fine-grained long-tail challenge with a hierarchical categorization spanning from 7 kingdoms to 8142 species. Although zero-shot CLIP achieves high accuracy in predicting coarse-grained categories (\eg, ``kingdom'' and ``phylum''), it performs poorly in classifying fine-grained species, \ie, long-tail classes.

\subsection{Heavy Fine-Tuning Hurts}

Although the foundation model has shown commendable performance in downstream long-tail learning tasks, it has several limitations. As discussed previously, while CLIP achieves significant improvements on head classes, its performance gains on tail classes remain suboptimal. This observation naturally raises a critical question: \textit{Is the current fine-tuning strategy insufficient?}

\begin{figure}[!t]
\centering
\includegraphics[width=0.456\linewidth]{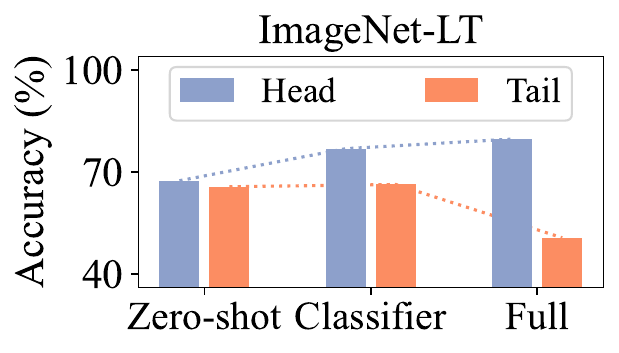}
\hspace{0.1in}
\includegraphics[width=0.437\linewidth]{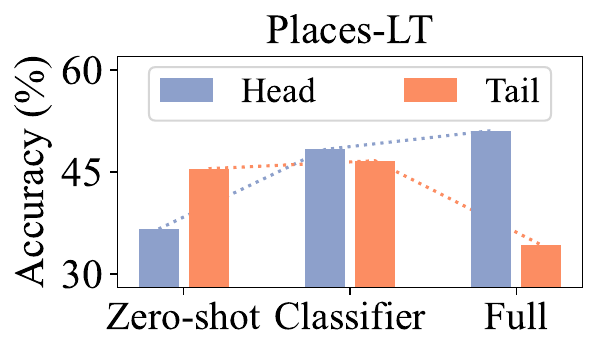}
\caption{Comparison of different fine-tuning manners. Full fine-tuning improves head-class accuracy but severely decreases tail-class performance, even when employing balanced loss and classifier initialization.}
\label{fig:deterioration}
\end{figure}

However, the results in \Cref{fig:deterioration} reveal a concerning finding: full fine-tuning enhances head-class accuracy while at the expense of degrading tail-class performance. Note that we have optimized the balanced logit-adjusted loss \cite{menon2021longtail} and applied a balanced classifier initialization (which will be introduced in \Cref{sec:sai}), the prediction bias still persists. Most strikingly, the performance deterioration of tail classes renders the overall accuracy of full fine-tuning even worse than that of classifier fine-tuning (71.6\% vs. 73.6\% on ImageNet-LT, and 46.6\% vs. 48.4\% on Places-LT).

\begin{figure}[!t]
\setlength{\tabcolsep}{0.2ex}
\subfloat[Classifier fine-tuning.]{
    \begin{tabular}{cc}
    \includegraphics[width=0.24\linewidth]{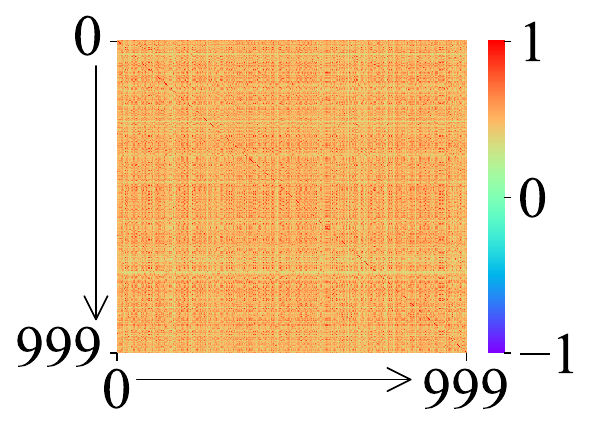} &
    \includegraphics[width=0.21\linewidth]{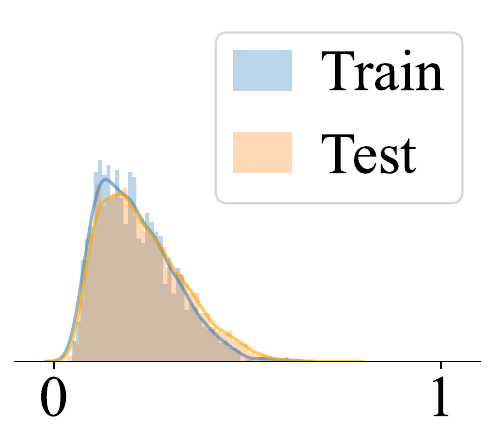} \\
    \color{red}{\fontsize{7.5pt}{7.5pt}\selectfont Head acc.: 76.8\%} &
    \color{mplgreen}{\fontsize{7.5pt}{7.5pt}\selectfont Tail acc.: 66.3\%} 
    \end{tabular}
    \label{fig:feature-classifier}
}
\subfloat[Full fine-tuning.]{
    \begin{tabular}{cc}
    \includegraphics[width=0.24\linewidth]{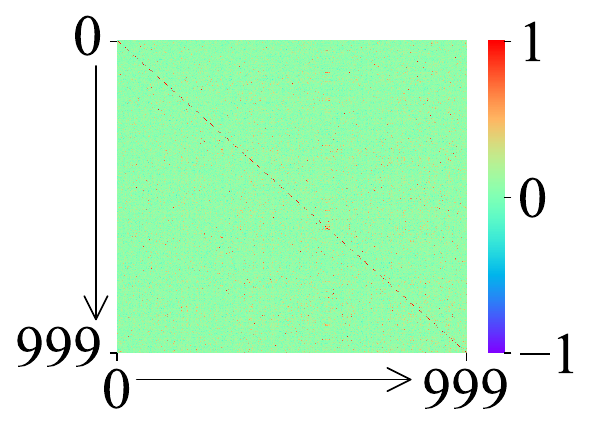} &
    \includegraphics[width=0.21\linewidth]{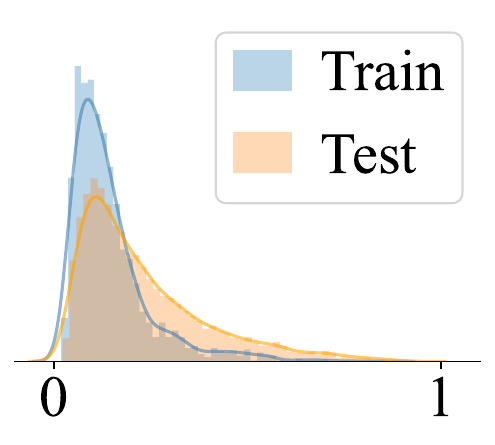} \\
    \color{mplgreen}{\fontsize{7.5pt}{7.5pt}\selectfont Head acc.: 79.7\%} &
    \color{red}{\fontsize{7.5pt}{7.5pt}\selectfont Tail acc.: 50.6\%} 
    \end{tabular}
    \label{fig:feature-full}
}
\caption{Inter-class feature similarities (heatmaps) and intra-class distance distributions from tail classes (histograms) on ImageNet-LT. Classifier fine-tuning limits head-class performance due to high inter-class similarities. Full fine-tuning optimizes inter-class similarities but leads to inconsistent distribution between train and test data on tail classes.}
\label{fig:ft-feature}
\end{figure}

To uncover the reasons behind it, we identify that full fine-tuning may distort the representation of tail classes. To validate this, we quantify the model representation from two perspectives: \romannumeral 1) \textit{inter-class feature similarities}, which is measured by the cosine similarities between class mean features; and \romannumeral 2) \textit{intra-class distance distributions}, which is characterized by the cosine similarities between individual samples and their corresponding class mean features. This approach is adopted since the class-conditional distribution cannot be directly observed.
We compute these metrics and present the results in \Cref{fig:ft-feature}. Notably, \Cref{fig:feature-full} reveals that full fine-tuning enhances feature discriminability by reducing inter-class similarities to approximately zero, rendering the features of different classes nearly orthogonal. However, it also distorts intra-class distributions, leading to a distribution shift between train and test data on tail classes. Consequently, using the fine-tuned model to estimate tail-class data will inevitably result in an underestimated class-conditional probability. We provide further theoretical evidence to analyze this matter.
\begin{proposition}
\label{prop:ccd}
The underestimated class-conditional probability $\oP(\phi(\vx)\mid y=j)$ leads to an underestimated loss on class $j$ and a biased prediction towards other classes.
\end{proposition}
\begin{proof}
Denote $\oP_s$ and $\oP_t$ as the probability distributions in the source (training) and target (test) domains, respectively. In long-tail learning scenarios, $\oP_s(y)$ follows a long-tail distribution and $\oP_t(y)$ follows a uniform distribution, \ie, $\oP_t(y=k)\equiv 1/K$. Consequently, we have
\begin{align}
&\oP_s(y=j\mid\phi(\vx))= \cfrac{\oP_s(y=j\mid\phi(\vx))}{\sum_{k=1}^{K}\oP_s(y=k\mid\phi(\vx))} \nonumber\\
&= \cfrac{\oP_t(y=j\mid\phi(\vx))\cdot\cfrac{\oP_s(y=j\mid\phi(\vx))}{\oP_t(y=j\mid\phi(\vx))}}{\sum_{k=1}^{K}\oP_t(y=k\mid\phi(\vx))\cdot\cfrac{\oP_s(y=k\mid\phi(\vx))}{\oP_t(y=k\mid\phi(\vx))}} \nonumber \\
&= \cfrac{\oP_t(y=j\mid\phi(\vx))\cdot\cfrac{\oP_s(\phi(\vx)\mid y=j)}{\oP_t(\phi(\vx)\mid y=j)}\cdot\cfrac{\oP_s(y=j)}{\oP_t(y=j)}}{\sum_{k=1}^{K}\oP_t(y=k\mid\phi(\vx))\cdot\cfrac{\oP_s(\phi(\vx)\mid y=k)}{\oP_t(\phi(\vx)\mid y=k)}\cdot\cfrac{\oP_s(y=k)}{\oP_t(y=k)}} \nonumber \\
&= \cfrac{\oP_t(y=j\mid\phi(\vx))\cdot\oP_s(y=j)\cdot\zeta_{s-t}(j)}{\sum_{k=1}^{K}\oP_t(y=k\mid\phi(\vx))\cdot\oP_s(y=k)\cdot\zeta_{s-t}(k)}
\end{align}
\begin{align}
&\gL(\phi(\vx),y=j)=-\log\oP_s(y=j\mid\phi(\vx)) \nonumber \\
&=-\log\cfrac{\exp\left(z_j+\log\oP_s(y=j)+\log\zeta_{s-t}(j)\right)}{\sum_{k=1}^{K}\exp\left(z_k+\log\oP_s(y=k)+\log\zeta_{s-t}(k)\right)}
\end{align}
where $\zeta_{s-t}(j)=\cfrac{\oP_s(\phi(\vx)\mid y=j)}{\oP_t(\phi(\vx)\mid y=j)}$ and $z_j$ is the predicted logit on class $j$.
For samples $\vx$ belonging to class $j$, the underestimated $\oP_t(\phi(\vx)\mid y=j)$ leads to an underestimation of $\gL(\phi(\vx),y=j)$, and consequently results in a biased optimization towards other classes.
\end{proof}

\begin{remark}
\Cref{prop:ccd} reveals that the performance degradation in full fine-tuning stems from inconsistent class-conditional distributions among tail classes. Existing approaches such as Balanced-Softmax \cite{ren2020balanced} and Logit-adjusted loss \cite{menon2021longtail,hong2021disentangling,zhao2022adaptive,li2022long} inherently assume a consistent class-conditional distribution between the source and target domains, \ie, $\zeta_{s-t}(j)=1$. However, as empirically demonstrated in \Cref{fig:feature-full}, full fine-tuning violates this assumption. While an ideal solution would involve estimating the underlying class-conditional distribution, this is infeasible in practice due to the scarcity of tail-class data. Alternatively, another viable approach is to mitigate such distribution distortions. To achieve this goal, we introduce lightweight fine-tuning in \Cref{sec:method}.
\end{remark}

To address generalization challenges with long-tail data, recent studies have explored two-stage training procedures \cite{ma2021simple} or incorporated external training data \cite{tian2022vl,long2022retrieval}. However, these strategies typically incur substantial training overhead or require additional data sources, which hinders their practical deployment. In response to this, we introduce \algo, an efficient and accurate lightweight fine-tuning framework tailored for long-tail learning.

\subsection{Lightweight Fine-Tuning Helps}
\label{sec:lightweight}

\begin{figure}[!t]
\centering
    \includegraphics[width=0.40\linewidth, trim=0 0 6 0]{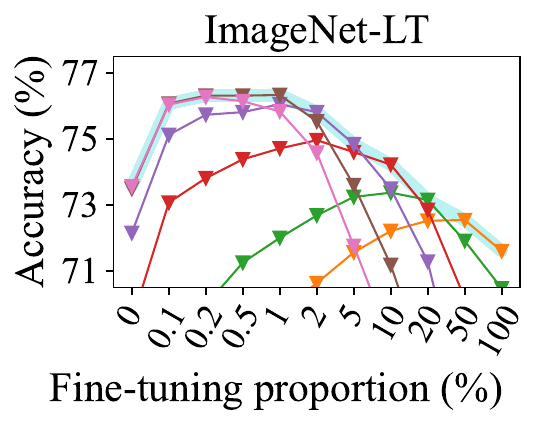}
    \includegraphics[width=0.40\linewidth, trim=6 0 0 0]{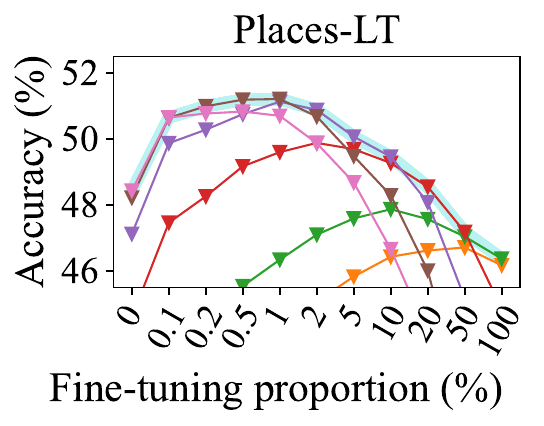}
    \includegraphics[width=0.12\linewidth, trim=10 0
 10 0]{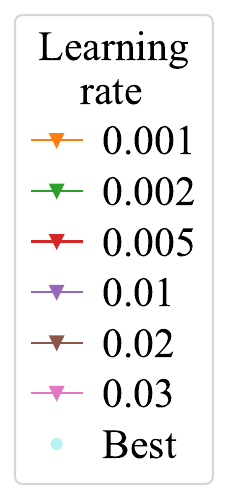}
    \caption{Fine-tuning a small proportion of all parameters (\eg, 0.1\%-2\%) yields superior performance. As the proportion increases, performance deteriorates even when we search for the best learning rate.}
    \label{fig:ft-proportion}
\end{figure}

To mitigate intra-class distribution distortion, a straightforward approach is to constrain the number of learned parameters. Formally, for each weight matrix $\mW\in\mathbb{R}^{d_1\times d_2}$ in the foundation model, we optimize a specified proportion $\alpha$ of parameters in $\mW$, while keeping the rest frozen. This approach allows only $\alpha d_1 d_2$ parameters to be optimized. In implementation, a sparse binary mask $\mM\in \{0,1\}^{d_1\times d_2}$ is employed to select the optimized parameters:
\begin{equation}
    \mX\mW \rightarrow \mX(\mW\circ \mM)+\underbrace{\mX(\mW\circ (1-\mM))}_{\textrm{gradient detached}}
\end{equation}
where $\Vert \mM\Vert_0=\alpha d_1 d_2$. 
In \Cref{fig:ft-proportion}, we investigate the impact of varying proportions of fine-tuned parameters. The result demonstrates the advantages of lightweight fine-tuning, since even a small proportion (\eg, 0.1\%) yields significant performance improvements. In contrast, as the proportion increases, the performance faces a risk of degradation, thereby underscoring the drawbacks of heavy fine-tuning. Moreover, heavy fine-tuning is sensitive to hyperparameters, as it requires careful searching of the learning rate to achieve optimal results.

\begin{figure}[!t]
\setlength{\tabcolsep}{0.2ex}
\subfloat[Arbitrary lightweight ft.]{
    \begin{tabular}{cc}
    \includegraphics[width=0.24\linewidth]{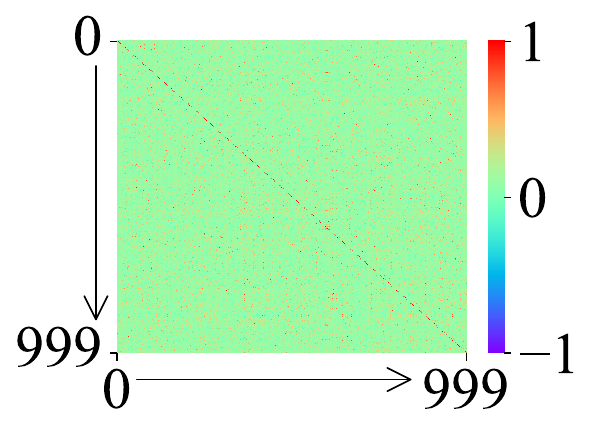} &
    \includegraphics[width=0.21\linewidth]{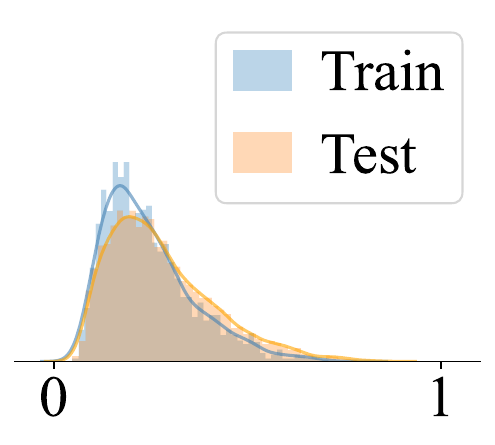} \\
    \color{mplgreen}{\fontsize{7.5pt}{7.5pt}\selectfont Head acc.: 79.4\%} &
    \color{mplgreen}{\fontsize{7.5pt}{7.5pt}\selectfont Tail acc.: 70.7\%} \\
    \end{tabular}
    \label{fig:feature-alf}
}
\subfloat[Structured lightweight ft.]{
    \begin{tabular}{cc}
    \includegraphics[width=0.24\linewidth]{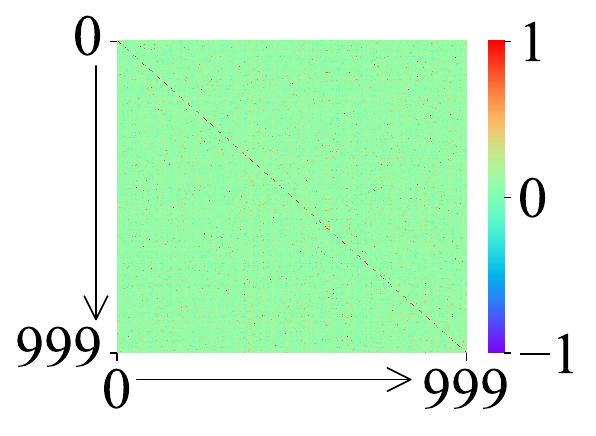} &
    \includegraphics[width=0.21\linewidth]{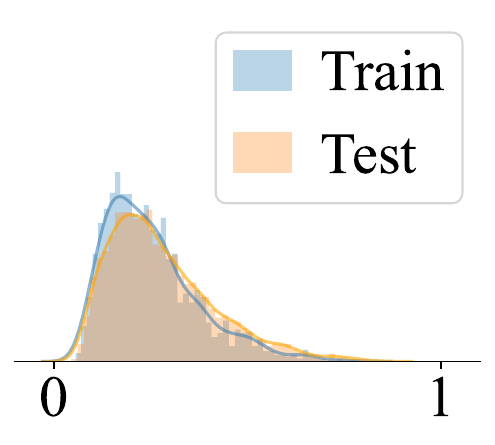} \\
    \color{mplgreen}{\fontsize{7.5pt}{7.5pt}\selectfont Head acc.: 79.9\%} &
    \color{mplgreen}{\fontsize{7.5pt}{7.5pt}\selectfont Tail acc.: 71.7\%} \\
    \end{tabular}
    \label{fig:feature-slf}
}
\caption{Inter-class feature similarities (heatmaps) and intra-class distributions from tail classes (histograms) on ImageNet-LT. Both arbitrary and structured lightweight fine-tuning perform well in optimizing inter-class similarities and preserving intra-class distributions.}
\label{fig:ft-lightweight}
\end{figure}

It is noteworthy that the optimized parameters are selected in an arbitrary manner, constituting what we term as \textit{arbitrary lightweight fine-tuning}. Despite its simplicity, the performance improvement is remarkably significant. This suggests that lightweight fine-tuning is crucial to enhance long-tail learning even without specific fine-tuning strategies. 
To gain a deeper understanding, we visualize both inter-class feature similarities and intra-class distributions in \Cref{fig:feature-alf}. The results demonstrate that arbitrary lightweight fine-tuning achieves feature separability comparable to that of full fine-tuning, while effectively maintaining consistent class conditions between training and test data. Furthermore, prediction accuracy substantiates the efficacy of arbitrary lightweight fine-tuning, as it matches the high performance of full-fine-tuning on head classes and even achieves superior performance on tail classes.

By investigating the above arbitrary strategy, we demonstrate that the key to effective fine-tuning lies in learning a small set of parameters rather than optimizing the entire model.
Motivated by this insight, we further explore \textit{structured lightweight fine-tuning}, which learns a small set of task-specific parameters in a structured manner. Some related ideas, although already used in foundation models, are not directly designed for long-tail learning. We presented more detailed introductions in \Cref{sec:tech_detail} due to space constraints.
\Cref{fig:feature-slf} illustrates the potential of a very simple idea of structured lightweight fine-tuning.
As shown, this method optimizes inter-class similarities to almost orthogonal while preserving intra-class distributions undistorted. Notably, the performance surpasses arbitrary lightweight fine-tuning on both head and tail classes. These results indicate that the well-designed lightweight module can further enhance the feature separability as well as maintain the consistency of class-conditional distributions, and will be more accurate.

\section{\algo: Efficient and Accurate Long-Tail Learning Framework}
\label{sec:method}
\subsection{The Proposed Fine-Tuning Method}

Our analyses justify that lightweight fine-tuning can significantly alleviate performance degradation and enhance generalization. Based on this finding, we propose \textit{LIghtweight Fine-Tuning} (\algo), an efficient and accurate framework for long-tail learning.
The \algo\ framework is versatile and inclusive, accommodating the integration of various lightweight fine-tuning methods, including arbitrary lightweight fine-tuning proposed in \Cref{sec:lightweight}, and various structured lightweight fine-tuning methods, such as BitFit, VPT, Adapter, LoRA, and AdaptFormer. Although these methods have shown efficacy in foundation models, their performance remains suboptimal to full fine-tuning in previous studies \cite{zaken2022bitfit,jia2022visual,houlsby2019parameter,hu2022lora,chen2022adaptformer}. In contrast, our work first identifies the limitations of full fine-tuning, thereby revealing the untapped potential of these lightweight fine-tuning methods in long-tail learning. Due to space constraints, the technical details are presented in \Cref{sec:tech_detail}.

Crucially, lightweight fine-tuning maintains the consistency of class-conditional distributions. This property enables us to adopt an unbiased LA loss for optimization:
\begin{equation}
    \gL(\vx,y=j)=-\log\frac{\exp({z_j}+\log \oP(y=j))}{\sum_{k\in[K]}\exp({z_k}+\log \oP(y=k))}
\end{equation}
where $y=j$ represents the ground-truth label of $\vx$ and $z_j$ denotes the predicted logit. $\oP(y=j)$ signifies the class prior probability, which can be estimated from the training data. 

In order to have better adaptability among tasks, in addition to lightweight fine-tuning, a classifier is essential for discerning the features for different tasks. The linear classifier is commonly adopted due to its simplicity and versatility. Given a feature vector $\phi(\vx)$, the predicted logit for class $j$ is computed as $z_j=\phi(\vx)\vw_j^{\top}+b_j$.
However, when training with long-tail data, the classifier weight norms $\Vert \vw_k\Vert_2\ (1\leq k\leq K)$ tend to exhibit an imbalanced distribution, leading to biased predictions \cite{kang2020decoupling,wei2021towards}. To overcome this issue and draw on the strengths of CLIP, which optimizes cosine distances in feature space, inspired by \cite{wei2021towards}, we propose to optimize an enhanced cosine classifier $z_j=\sigma\cdot\langle\phi(\vx), \vw_j\rangle$, where $\sigma$ is a scaling factor. This design inherently eliminates the influence of the classifier norm by employing cosine similarity, thereby effectively mitigating biased predictions.

\begin{figure}[!t]
    \centering
    \includegraphics[width=\linewidth]{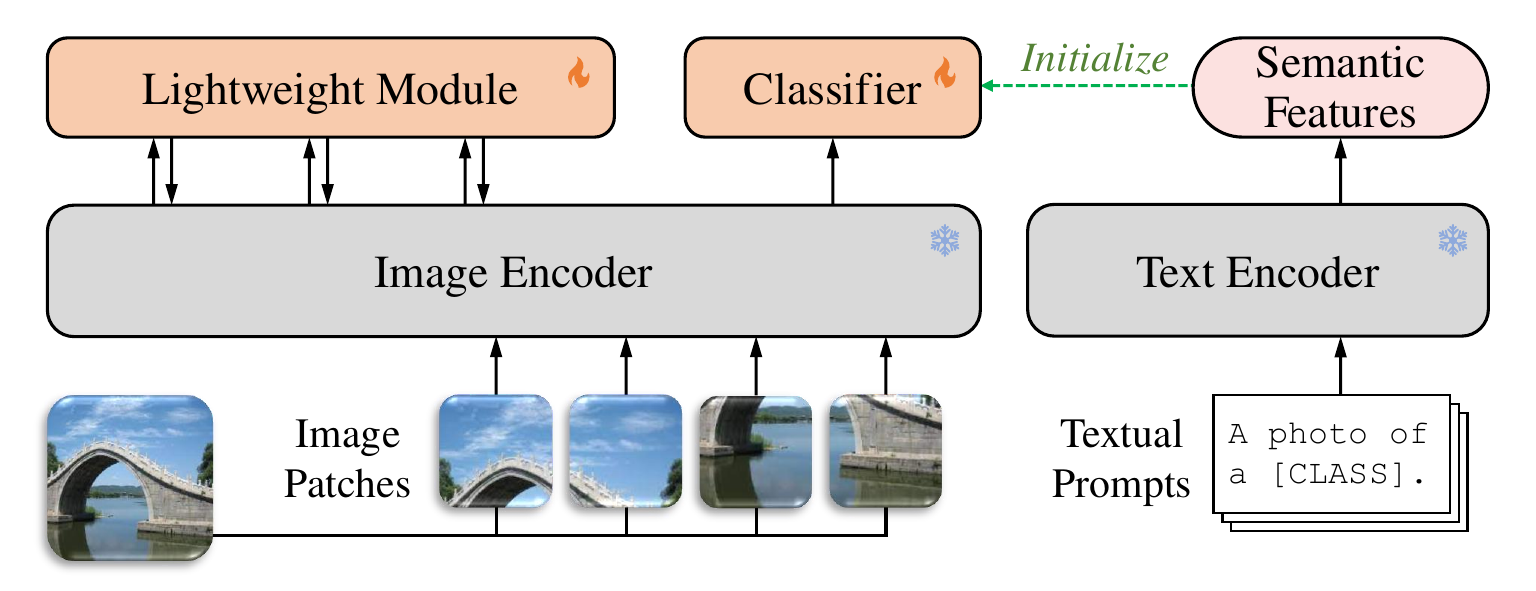}
    \caption{An overview of the proposed semantic-aware initialization. The text encoder is discarded after a single forward pass to initialize the classifier.}
    \label{fig:sai}
\end{figure}

\subsection{Semantic-Aware Initialization}
\label{sec:sai}
It has been demonstrated that an uninitialized classifier can adversely affect model fine-tuning \cite{kumar2022finetuning}. Therefore, it is crucial to set an appropriate initial state for the classifier. A straightforward method is to apply linear probing by optimizing balanced loss. Another approach is to utilize the class mean features. However, these two approaches not only require extracting features of training data, but are also unreliable with scarce tail-class data.
To overcome it, we propose to leverage the semantic knowledge from the text modality of CLIP.
Specifically, we generate hand-crafted textual prompts (\eg, ``\texttt{a photo of a [CLASS].}'') and compute their text features $\psi(\vt_1),\cdots,\psi(\vt_K)$, which are then utilized to initialize the classifier weights as $\vw_j=\psi(\vt_j)\mP_T\mP_I^{\top}$. In this way, we can converse image-text matching $\vz_j=\langle\phi(\vx)\mP_I, \psi(\vt_j)\mP_T\rangle$ to feature-classifier matching $\vz_j=\langle\phi(\vx), \vw_j\rangle$, where the initial performance of the model is still close to CLIP.
We call this way \textit{semantic-aware initialization} (SAI).
Unlike prior methods that fine-tune both the image encoder and the text encoder in optimization processes \cite{ma2021simple,tian2022vl}, SAI relies on a single forward pass of the text encoder for each class description. After that, the text encoder is discarded. This simple approach allows us to achieve a better initial state of the classifier with small computational overhead.
\Cref{fig:sai} gives an overview of the proposed semantic-aware initialization.

\subsection{Minimalist Data Augmentation}
The above methods incorporate tailored strategies at both the model representation and the output levels, including a lightweight fine-tuning module, an unbiased logit-adjusted loss, an enhanced cosine classifier, and a semantic-aware initialization method. However, such a framework still uses conventional data augmentation \cite{simonyan2014very,he2016deep,cubuk2019autoaugment,cubuk2020randaugment} at the input level. The conventional data augmentation strategy is proposed to train a deep model from scratch, which consists of random cropping, stretching, resizing, and flipping of a given image. The purpose is to generate more diverse samples and to cover the underlying data distribution. However, when it comes to fine-tuning foundation models, the key point is not to simulate richer data, but to adapt more efficiently and precisely to downstream data. Moreover, conventional data augmentation faces the risk of generating distorted or trivial images, hindering robust generalization, especially for tail classes where training samples are severely limited. In our learning framework, the number of training epochs is quite small (\eg, 5-15), which means that an image is augmented only a few times, and the quality of augmentation is quite crucial.

\begin{figure}[!t]
    \centering
    \subfloat[]{
        \includegraphics[width=0.63\linewidth]{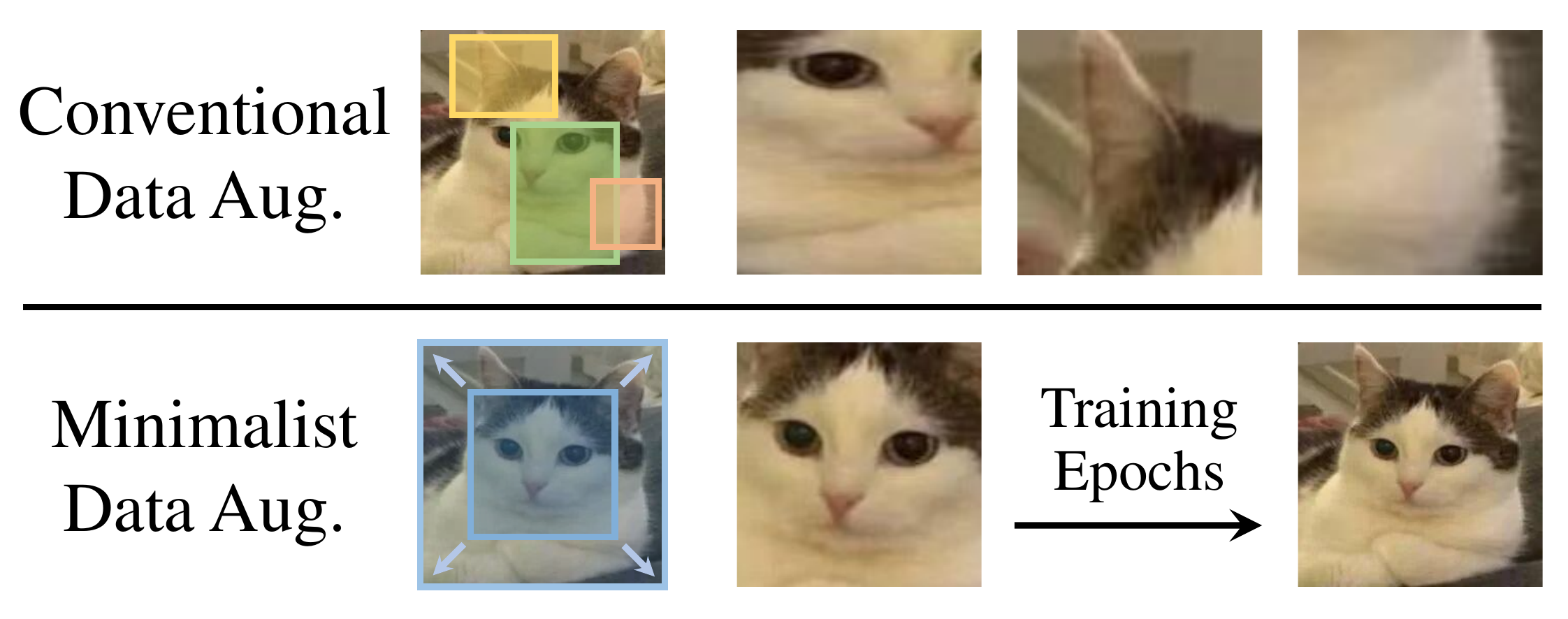}
        \label{fig:mda_overview}
    }
    \hfill
    \subfloat[]{
        \includegraphics[width=0.28\linewidth]{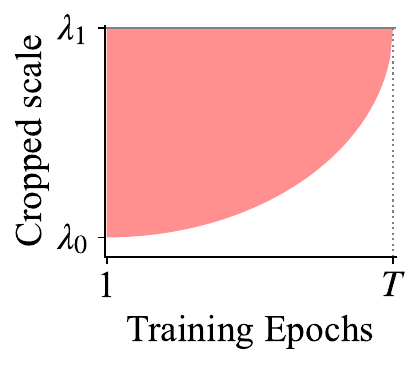}
        \label{fig:mda_func}
    }
    \caption{(a) Comparison of conventional and minimalist data augmentation. (b) The cropped scale scheduling function.}
    \label{fig:mda}
\end{figure}

\begin{algorithm}[!t]
    \renewcommand{\algorithmicrequire}{\textbf{Input:}}
    \renewcommand{\algorithmicensure}{\textbf{Output:}}
    \caption{Minimalist Data Augmentation}
    \label{alg:mda}
    \small
    \setstretch{1.2}
    \begin{algorithmic}[1]
    \REQUIRE Image $\vx$ (resolution $h\times w$), cropped scale range $(\lambda_0, \lambda_1)$, 
    {epoch $t\in [T]$, scheduling function $g(\cdot): [0,1]\rightarrow[0,1]$.}
    \ENSURE Image $\vx'$ (resolution $a\times a$).
    \STATE {$\Delta\lambda_0=(\lambda_1-\lambda_0)g(\frac{t-1}{T-1})$.}
    \STATE Sample $\lambda\sim(\lambda_0{+\Delta\lambda_0}, \lambda_1)$.
    \STATE $h'=w'=\sqrt{hw\lambda}$.
    \STATE Randomly crop $\vx'$ ($h'\times w'$) from $\vx$.
    \STATE Resize $\vx'$ to $a\times a$.
    \end{algorithmic}
\end{algorithm}

To address this dilemma, we further introduce a novel \textit{minimalist data augmentation} (MDA) strategy. Specifically, we eliminate stretching and flipping operations and only crop images with consistent proportions to ensure image fidelity. In addition, we design a scheduling function $g(\cdot)$ to control the scale of cropped images from local regions to the global level based on the training process. Through this approach, we ensure that critical information is not neglected, especially at the end of training. As illustrated in \Cref{fig:mda_overview}, conventional data augmentation generates many redundant or misleading training samples, while the proposed MDA is minimalist and essential. The detailed procedure is presented in \Cref{alg:mda}. 
Following the implementation of multiple previous works \cite{simonyan2014very,he2016deep,cubuk2019autoaugment,cubuk2020randaugment} and the official setting of PyTorch\protect\footnotemark[1], we set $(\lambda_0, \lambda_1)=(0.08, 1)$. The scheduling function is set as a simple circular convex function $g(t)=1-\sqrt{1-t^2}$ (plot in \Cref{fig:mda_func}). The analysis of different scheduling functions is given in \Cref{sec:component_analysis}. Compared to conventional data augmentation, MDA is even more efficient due to its pruning of redundant steps.
\footnotetext[1]{The official PyTorch implementation of image cropping: \rrcurl}

\subsection{Test-Time Ensembling}
It is well-established that applying random perturbations to each input can lead to improved generalization during tests \cite{sun2020test,wang2021tent,zhou2023ods}. Recent works further discover that scaling the test-time computation can improve the output and can be even more effective than scaling the model parameters \cite{snell2024scaling,chen2024expanding}. This principle is particularly suitable for enhancing Transformer-based models, where an image is split into multiple patches, potentially resulting in the partitioning of continuous patterns into discrete patches. Such partitioning aggravates the data scarcity issues, especially for tail classes with limited data. To enhance generalization robustness, we propose to aggregate the predictions from a set of perturbed versions of the image. Formally, given a test image $\boldsymbol{x}$, its predicted logits $\vz$ are obtained by averaging the predictions from $M$ perturbed versions:
\begin{equation}
\vz=\log\oP(\vy\mid \vx)=\frac{1}{M} \sum_{i=1}^M \log\oP(\vy \mid \alpha_{i}(\vx))
\end{equation}
Here, $\alpha_i(\boldsymbol{x})$ represents different perturbed versions of $\boldsymbol{x}$. 
We term this technique \textit{test-time ensembling} (TTE).

\begin{algorithm}[!t]
\renewcommand{\algorithmicrequire}{\textbf{Input:}}
\renewcommand{\algorithmicensure}{\textbf{Output:}}
\caption{Test-Time Ensembling}
\label{alg:tte}
\small
\setstretch{1.2}
\begin{algorithmic}[1]
\REQUIRE Image $\vx$ (resolution $a\times a$), expanded size $e$.
\ENSURE Predicted logits $\vz$.
\STATE Resize $\vx$ to $\vx'$ sized $(a+e)\times(a+e)$.
\STATE Crop the center $a\times a$ region of $\vx'$, denoted by $\vx^\text{c}$.
\STATE Split $\vx^\text{c}$ evenly into $m$ patches $[\vx^{\text{p}}_1;\cdots;\vx^{\text{p}}_m]$.
\STATE Calculate the feature $\phi(\vx^\text{c})$ and then the logits $\vz^\text{c}$.
\STATE {Crop the $a\times a$ region from each corner (top left, top right, bottom left, and bottom right) of $\vx'$, repeat procedure 3-4 and obtain the corresponding logits $\vz^\text{tl}$, $\vz^\text{tr}$, $\vz^\text{bl}$, $\vz^\text{br}$.}
\STATE $\vz=\operatorname{Average}([\vz^\text{c}{{, \vz^\text{tl}, \vz^\text{tr}, \vz^\text{bl}, \vz^\text{br}}}])$.
\end{algorithmic}
\end{algorithm}

In practice, we employ a simple but effective perturbation approach, \ie, using different cropping regions.
This approach helps mitigate the bias introduced by image patches.
The detailed procedure is presented in \Cref{alg:tte}.
Conventionally, an image is first resized, center-cropped, and then split into patches before being fed into the Transformer model. However, this approach inevitably leads to the partitioning of important patterns across different patches, thus impeding generalization. By applying diverse croppings, patterns that were separated in one cropping can be preserved in another. It is crucial to emphasize that the expanded size $e$ should not be a multiple of the patch size; otherwise, the five cropped images will share a large portion of the same patches, rendering the expected diversity unattainable. For example, when the patch size is 16, we suggest setting $e$ to 8, 24, 40 rather than 16, 32, 48. We provide an empirical analysis of different expanded sizes in \Cref{sec:component_analysis}.
The proposed TTE strategy can be seamlessly integrated into \algo\ with minimal computational overhead while consistently improving performance.
For fair comparison, we will make necessary statements when equipping any method with TTE.

\section{Empirical Study}
\label{sec:empirical_study}

\subsection{Experimental Settings}
\label{sec:exp_settings}
We evaluate our approach on four widely-used long-tail datasets, including ImageNet-LT \cite{liu2019large}, Places-LT \cite{liu2019large}, iNaturalist 2018 \cite{van2018inaturalist}, and CIFAR-100-LT \cite{cao2019learning}.
ImageNet-LT has 115.8K images from 1000 classes, with class frequencies ranging from 1280 to 5 samples.
Places-LT contains 62.5K images from 365 classes, with class frequencies ranging from 4980 to 5 samples.
iNaturalist 2018 consists of 437.5K images distributed across 8142 species, with the number of images per species varying from 1000 to 2.
CIFAR-100-LT is constructed with multiple imbalance ratio options, including 100, 50, and 10.
The models are trained on long-tail datasets and subsequently assessed on their corresponding balanced test datasets. 
In addition to measuring overall accuracy, we follow the evaluation protocol introduced by \cite{liu2019large} to report accuracy among three splits of classes: head classes (containing more than 100 images), medium classes (20 to 100 images), and tail classes (fewer than 20 images).

\subsection{Implementation Details}
\label{sec:imp_details}

For all experiments, we use the SGD optimizer with a batch size of 128, a weight decay of $5\times 10^{-4}$, and a momentum of 0.9. For \algo, the learning rate is 0.02; and for full fine-tuning, we search for the optimal learning rate from \{0.03, 0.02, 0.01, 0.005, 0.002, 0.001\} considering its weak stability. For ImageNet-LT, Places-LT, and CIFAR-100-LT, we train the model for only 5 epochs; and for iNaturalist 2018, we train 15 epochs since it has much more data. We set the bottleneck dimensionality $r=2^{\lfloor\log_{2}{(\frac{K}{2L})}\rfloor}$ for Adapter and AdaptFormer such that it learns even fewer parameters than the classifier (please refer to \Cref{sec:tech_detail} for detailed analysis). The scaling factor $\sigma$ of the cosine classifier is set to 25 (please refer to \Cref{table:comp_classifier} for analysis). The expanded size $e$ for TTE is set to 24 (please refer to \Cref{fig:expand_size} for analysis). All experiments in this paper are conducted on a single NVIDIA 4090 GPU with only 24GB of memory.

\subsection{Comparison with State-of-the-art Methods}

\begin{table}[!t]
\caption{Comparison with state-of-the-arts methods on ImageNet-LT. $\dagger$ denotes methods using external data. Empty cells ``-'' indicate unreported results in their original papers.}
\label{table:comp_imagenetlt}
\setlength{\tabcolsep}{0.12ex} 
\centering
\begin{small}
\begin{tabular}{l|c|c|c|cccc}
\toprule
\bf Methods &\bf Backbone &\makecell{\bf Learned \\ \bf Params.} &\bf Epochs &\bf All &\bf Head &\bf Med. &\bf Tail \\
\midrule
\multicolumn{5}{l}{\bf Training from scratch} \\
\midrule
cRT \cite{kang2020decoupling} & RN-50 & 23.51M & 90+10 & 47.3 & 58.8 & 44.0 & 26.1 \\
LWS \cite{kang2020decoupling} & RN-50 & 23.51M & 90+10 & 47.7 & 57.1 & 45.2 & 29.3 \\
MiSLAS \cite{zhong2021improving} & RN-50 & 23.51M & 180+10 & 52.7 & 62.9 & 50.7 & 34.3 \\
KCL \cite{kang2021exploring} & RN-50 & 23.51M & 200 & 51.5 & 61.8 & 49.4 & 30.9 \\
LA \cite{menon2021longtail} & RN-50 & 23.51M & 90 & 51.1 & - & - & - \\
DisAlign \cite{zhang2021distribution} & RN-50 & 23.51M & 90 & 52.9 & 61.3 & 52.2 & 31.4 \\
RIDE \cite{wang2021longtailed} & RN-50 & 23.51M & 100 & 55.4 & 66.2 & 52.3 & 36.5 \\
rwSAM \cite{liu2022selfsupervised} & RN-50 & 23.51M & 100 & 55.5 & - & - & - \\
BCL \cite{zhu2022balanced} & RN-50 & 23.51M & 100 & 56.0 & - & - & - \\
PaCo \cite{cui2021parametric} & RN-50 & 23.51M & 400 & 57.0 & - & - & - \\
NCL \cite{li2022nested} & RN-50 & 23.51M & 400 & 57.4 & - & - & -  \\
LiVT \cite{xu2023learning} & ViT-B/16 & 85.80M & 100 & 60.9 & 73.6 & 56.4 & 41.0 \\
\midrule
\multicolumn{8}{l}{\bf Fine-tuning foundation model (from CLIP)} \\
\midrule
BALLAD \cite{ma2021simple} & ViT-B/16 & 149.62M & 50+10 & 75.7 & 79.1 & 74.5 & 69.8 \\
VL-LTR \cite{tian2022vl}$^\dagger$ & ViT-B/16 & 149.62M & 100 & 77.2 &\bf 84.5 & 74.6 & 59.3 \\
GML \cite{suh2023long}$^\dagger$ & ViT-B/16 & 149.62M & 100 & 78.0 & - & - & - \\
Decoder \cite{wang2023exploring} & ViT-B/16 & 21.26M & $\sim$18 & 73.2 & - & - & - \\
\algo\ (Ours) & ViT-B/16 &\bf 0.62M &\bf 5 & 77.0 & 79.9 & 76.1 & 71.7 \\
\ w/ TTE (Ours) & ViT-B/16 &\bf 0.62M &\bf 5 &\bf 78.3 & 81.0 &\bf 77.5 &\bf 73.4 \\
\bottomrule
\end{tabular}
\end{small}
\end{table}

\textbf{Results on ImageNet-LT.} We report the test accuracy in \Cref{table:comp_imagenetlt}. While existing approaches such as VL-LTR \cite{tian2022vl} and GML \cite{suh2023long} rely on extensive auxiliary data to facilitate fine-tuning, our method \algo\ achieves competitive performance with significantly lower costs. Reliance on external data not only incurs substantial computational overhead but also limits practical applicability, as such data is often unavailable in real-world scenarios. Notably, \algo\ requires only 5 epochs of training and fine-tunes far fewer model parameters (\ie, from 21.26M to 0.62M). Compared to methods that do not use auxiliary data, \algo\ demonstrates even more pronounced advantages, surpassing the best existing method by 1.3\% in accuracy. It is worth noting that results using the ImageNet-21K pre-trained foundation model are not reported, considering the class overlap between it and ImageNet-LT. RAC \cite{long2022retrieval} and LPT \cite{dong2023lpt} are not included, since they default to employing the ImageNet-21K pre-trained foundation model, and the results on ImageNet-LT were not reported in their original papers.

\begin{table}[!t]
\caption{Comparison with state-of-the-arts methods on Places-LT. $\dagger$ denotes methods using external data. Empty cells ``-'' indicate unreported results in their original papers.}
\label{table:comp_placeslt}
\setlength{\tabcolsep}{0.12ex} 
\centering
\begin{small}
\begin{tabular}{l|c|c|c|cccc}
\toprule
\bf Methods &\bf Backbone &\makecell{\bf Learned \\ \bf Params.} &\bf Epochs &\bf All &\bf Head &\bf Med. &\bf Tail\\ 
\midrule
\multicolumn{8}{l}{\bf Training from scratch (with ImageNet pre-trained weights)} \\
\midrule
OLTR \cite{liu2019large} & RN-152 & 58.14M & 30 & 35.9 & 44.7 & 37.0 & 25.3 \\
cRT \cite{kang2020decoupling} & RN-152 & 58.14M & 90+10 & 36.7 & 42.0 & 37.6 & 24.9 \\
LWS \cite{kang2020decoupling} & RN-152 & 58.14M & 90+10 & 37.6 & 40.6 & 39.1 & 28.6 \\
MiSLAS \cite{zhong2021improving} & RN-152 & 58.14M & 90+10 & 40.4 & 39.6 & 43.3 & 36.1 \\
ResLT \cite{cui2022reslt} & RN-152 & 58.14M & 30 & 39.8 & 39.8 & 43.6 & 31.4 \\
DisAlign \cite{zhang2021distribution} & RN-152 & 58.14M & 30 & 39.3 & 40.4 & 42.4 & 30.1 \\
ALA \cite{zhao2022adaptive} & RN-152 & 58.14M & 30 & 40.1 & 43.9 & 40.1 & 32.9 \\
PaCo \cite{cui2021parametric} & RN-152 & 58.14M & 30 & 41.2 & 36.1 & 47.9 & 35.3 \\
LiVT \cite{xu2023learning} & ViT-B/16 & 85.80M & 100 & 40.8 & 48.1 & 40.6 & 27.5 \\
\midrule
\multicolumn{8}{l}{\bf Fine-tuning foundation model (from CLIP)} \\
\midrule
BALLAD \cite{ma2021simple} & ViT-B/16 & 149.62M & 50+10 & 49.5 & 49.3 & 50.2 & 48.4 \\
VL-LTR \cite{tian2022vl}$^\dagger$ & ViT-B/16 & 149.62M & 100 & 50.1 &\bf 54.2 & 48.5 & 42.0 \\
Decoder \cite{wang2023exploring} & ViT-B/16 & 21.26M & $\sim$34 & 46.8 & - & - & - \\
\algo\ (Ours) & ViT-B/16 &\bf 0.18M &\bf 5 & 51.5 & 50.8 & 52.0 & 51.6 \\
\ w/ TTE (Ours) & ViT-B/16 &\bf 0.18M &\bf 5 &\bf 52.1 & 51.5 &\bf 52.8 &\bf 51.7 \\
\midrule
\multicolumn{8}{l}{\bf Fine-tuning foundation model (from ImageNet-21K)} \\
\midrule
RAC \cite{long2022retrieval}$^\dagger$ & ViT-B/16 & 85.80M & 30 & 47.2 & 48.7 & 48.3 & 41.8 \\
LPT \cite{dong2023lpt} & ViT-B/16 & 1.01M & 40+40 & 50.1 & 49.3 & 52.3 & 46.9 \\
\algo\ (Ours) & ViT-B/16 &\bf 0.18M &\bf 5 & 48.1 & 48.4 & 48.7 & 46.3 \\
\ w/ TTE (Ours) & ViT-B/16 &\bf 0.18M &\bf 5 & 49.0 & 49.2 & 49.7 & 47.3 \\
\bottomrule
\end{tabular}
\end{small}
\end{table}

\textbf{Results on Places-LT.} \Cref{table:comp_placeslt} shows that \algo\ achieves significantly superior performance over existing methods on Places-LT. In particular, \algo\ surpasses VL-LTR by 1.4\%, despite its use of external training data. When taking a closer look at the performance on tail classes, we can see that \algo\ achieves particularly significant improvements compared to state-of-the-art methods, \ie, from 48.4 to 51.6. It is worth noting that, when applying to ImageNet-21K pre-trained foundation model, we use class mean features instead of semantic-aware initialization due to the absence of a corresponding text encoder. Although LPT achieves higher performance, it needs to adopt a two-stage framework with a total of 80 epochs (40 in each stage). In comparison, \algo\ requires only 5 training epochs. Moreover, the number of learned parameters is substantially smaller (0.18M).

\begin{table}[!t]
\caption{Comparison with state-of-the-arts methods on iNaturalist 2018. $\dagger$ denotes using external data. Empty cells ``-'' indicate unreported results in their original papers.}
\label{table:comp_inat18}
\setlength{\tabcolsep}{0.12ex} 
\centering
\begin{small}
\begin{tabular}{l|c|c|c|cccc}
\toprule
\bf Methods &\bf Backbone &\makecell{\bf Learned \\ \bf Params.} &\bf Epochs &\bf All &\bf Head &\bf Med. &\bf Tail \\
\midrule
\multicolumn{8}{l}{\bf Training from scratch} \\
\midrule
cRT \cite{kang2020decoupling} & RN-50 & 23.51M & 90+10 & 65.2 & 69.0 & 66.0 & 63.2 \\
LWS \cite{kang2020decoupling} & RN-50 & 23.51M & 90+10 & 65.9 & 65.0 & 66.3 & 65.5 \\
MiSLAS \cite{zhong2021improving} & RN-50 & 23.51M & 200+30 & 71.6 & 73.2 & 72.4 & 70.4 \\
DiVE \cite{he2021distilling} & RN-50 & 23.51M & 90 & 69.1 & 70.6 & 70.0 & 67.6 \\
DisAlign \cite{zhang2021distribution} & RN-50 & 23.51M & 90 & 69.5 & 61.6 & 70.8 & 69.9 \\
ALA \cite{zhao2022adaptive} & RN-50 & 23.51M & 90 & 70.7 & 71.3 & 70.8 & 70.4 \\
RIDE \cite{wang2021longtailed} & RN-50 & 23.51M & 100 & 72.6 & 70.9 & 72.4 & 73.1 \\
\ +CR \cite{ma2023curvature} & RN-50 & 23.51M & 200 & 73.5 & 71.0 & 73.8 & 74.3 \\
\ +OTmix \cite{gao2023enhancing} & RN-50 & 23.51M & 210 & 73.0 & 71.3 & 72.8 & 73.8 \\
BCL \cite{zhu2022balanced} & RN-50 & 23.51M & 100 & 71.8 & - & - & - \\
PaCo \cite{cui2021parametric} & RN-50 & 23.51M & 400 & 73.2 & 70.4 & 72.8 & 73.6 \\
NCL \cite{li2022nested} & RN-50 & 23.51M & 400 & 74.2 & 72.0 & 74.9 & 73.8 \\
GML \cite{suh2023long} & RN-50 & 23.51M & 400 & 74.5 & - & - & - \\
LiVT \cite{xu2023learning} & ViT-B/16 & 85.80M & 100 & 76.1 &\bf 78.9 & 76.5 & 74.8 \\
\midrule
\multicolumn{8}{l}{\bf Fine-tuning foundation model (from CLIP)} \\
\midrule
VL-LTR \cite{tian2022vl}$^\dagger$ & ViT-B/16 & 149.62M & 100 & 76.8 & - & - & - \\
Decoder \cite{wang2023exploring} & ViT-B/16 & 21.26M &\bf $\sim$5 & 59.2 & - & - & - \\
\algo\ (Ours) & ViT-B/16 & 4.75M & 15 & 79.2 & 73.0 & 79.1 & 80.8 \\
\ w/ TTE (Ours) & ViT-B/16 & 4.75M & 15 & 80.5 & 74.4 & 80.4 & 82.3 \\
\midrule
\multicolumn{8}{l}{\bf Fine-tuning foundation model (from ImageNet-21K)} \\
\midrule
RAC \cite{long2022retrieval}$^\dagger$ & ViT-B/16 & 85.80M & 20 & 80.2 & 75.9 & 80.5 & 81.1 \\
LPT \cite{dong2023lpt} & ViT-B/16 &\bf 1.01M & 80+80 & 76.1 & - & - & 79.3 \\
\algo\ (Ours) & ViT-B/16 & 4.75M & 15 & 81.1 & 74.5 & 81.5 & 82.4 \\
\ w/ TTE (Ours) & ViT-B/16 & 4.75M & 15 &\bf 82.1 & 75.4 &\bf 82.3 &\bf 83.6 \\
\bottomrule
\end{tabular}
\end{small}
\end{table}

\textbf{Results on iNaturalist 2018.} We present the results in \Cref{table:comp_inat18}. In general, \algo\ achieves state-of-the-art performance on this challenging dataset across different foundation models. It consistently outperforms a series of existing approaches, including VL-LTR, LPT, and RAC. Although Decoder \cite{wang2023exploring} employs fewer training epochs, its performance trails far behind \algo. Compared to LPT, \algo\ surpasses by 5.0\% in accuracy and reduces the training cost from 160 epochs (80 per stage) to only 15 epochs. Although LPT uses fewer learned parameters, we can adjust the parameters of \algo\ to reach a lower quantity (\eg, reduce the bottleneck dimensionality $r$ to 32, requiring only 0.62M parameters). In this case, \algo\ achieves an accuracy of 80.4\% (without TTE) / 81.5\% (with TTE), which outperforms LPT by at least 4.3\%. In fact, due to the large number of classes in iNaturalist 2018, the classifier itself already comprises 6.25M parameters. Therefore, the parameters introduced by \algo\ do not impose too much overhead.

\begin{table}[!t]
\caption{Comparison with state-of-the-art methods on CIFAR-100-LT with various imbalance ratios. Empty cells ``-'' indicate unreported results in their original papers.}
\label{table:comp_cifar100lt}
\setlength{\tabcolsep}{0.12ex} 
\centering
\begin{small}
\begin{tabular}{l|c|c|c| C{6ex} C{6ex} C{6ex} }
\toprule
\multirow{2}{*}{\bf Methods} & \multirow{2}{*}{\bf Backbone} &\multirow{2}{*}{\makecell{\bf Learned \\ \bf Params.}} & \multirow{2}{*}{\bf Epochs} & \multicolumn{3}{c}{\bf Imbalance Ratio} \\ 
& & & & \bf 100 &\bf 50 &\bf 10 \\
\midrule
\multicolumn{7}{l}{\bf Training from scratch} \\
\midrule
LDAM \cite{cao2019learning} & RN-32 & 0.46M & 200 & 42.0 & 46.6 & 58.7 \\
BBN \cite{zhou2020bbn} & RN-32 & 0.46M & 200 & 42.6 & 47.0 & 59.1 \\
DiVE \cite{he2021distilling} & RN-32 & 0.46M & 200 & 45.4 & 51.1 & 62.0 \\
MiSLAS \cite{zhong2021improving} & RN-32 & 0.46M & 200+10 & 47.0 & 52.3 & 63.2 \\
ResLT \cite{cui2022reslt} & RN-32 & 0.46M & 200 & 45.3 & 50.0 & 60.8 \\
BS \cite{ren2020balanced} & RN-32 & 0.46M & 400 & 50.8 & 54.2 & 63.0 \\
PaCo \cite{cui2021parametric} & RN-32 & 0.46M & 400 & 52.0 & 56.0 & 64.2 \\
BCL \cite{zhu2022balanced} & RN-32 & 0.46M & 200 & 51.9 & 56.6 & 64.9 \\
\midrule
\multicolumn{7}{l}{\bf Fine-tuning foundation model (from CLIP)} \\
\midrule
LiVT \cite{xu2023learning} & ViT-B/16 & 85.80M & 100 & 58.2 & - & 69.2 \\
BALLAD \cite{ma2021simple} & ViT-B/16 & 149.62M & 50+10 & 77.8 & - & - \\
\algo\ (Ours) & ViT-B/16 &\bf 0.10M &\bf 5 & 81.7 & 83.1 & 84.7 \\
\ w/ TTE (Ours) & ViT-B/16 &\bf 0.10M &\bf 5 & 83.4 & 84.3 & 85.9 \\
\midrule
\multicolumn{7}{l}{\bf Fine-tuning foundation model (from ImageNet-21K)} \\
\midrule
LPT \cite{dong2023lpt} & ViT-B/16 & 1.01M & 40+40 & 89.1 & 90.0 & 91.0 \\
\algo\ (Ours) & ViT-B/16 &\bf 0.10M &\bf 5 & 89.3 & 90.1 & 91.6 \\
\ w/ TTE (Ours) & ViT-B/16 &\bf 0.10M &\bf 5 &\bf 89.7 &\bf 90.7 &\bf 92.0 \\
\bottomrule
\end{tabular}
\end{small}
\end{table}

\textbf{Results on CIFAR-100-LT.} The results in \Cref{table:comp_cifar100lt} demonstrate that \algo\ consistently outperforms LiVT, BALLAD, and various training-from-scratch approaches across all imbalance ratios. In particular, our method achieves a 3.9\% accuracy improvement over BALLAD, while requiring only 5 training epochs. When adapting to foundation models pre-trained from ImageNet-21K, we employ class mean features to initialize the classifier due to the lack of a corresponding text encoder. It should be noted that the inherent class overlaps between CIFAR-100-LT and ImageNet-21K\protect\footnotemark[2] can elevate the baseline performance to an exceptionally high level. Despite this, our method maintains its superior performance, surpassing LPT with fewer training epochs and learned parameters.

\footnotetext[2]{Classes in CIFAR-100-LT: \cifarurl; Classes in ImageNet-21K: \imageneturl}

\subsection{Component Analysis and Ablation Study}
\label{sec:component_analysis}


\begin{table}[!t]
\caption{Ablation study of each component in \algo. The baseline involves learning a cosine classifier using LA loss.}
\label{table:ablation}
\setlength{\tabcolsep}{0.25ex} 
\centering
\begin{small}
\begin{tabular}{cccc | cccc | cccc}
\toprule
\multirow{2}{*}{\bf SLF} & \multirow{2}{*}{\bf SAI} & \multirow{2}{*}{\bf MDA} & \multirow{2}{*}{\bf TTE} & \multicolumn{4}{c|}{\bf ImageNet-LT} & \multicolumn{4}{c}{\bf Places-LT} \\ 
& & & &\bf All &\bf Head &\bf Med. &\bf Tail &\bf All &\bf Head &\bf Med. &\bf Tail \\
\midrule
& & & & 69.3 & 77.1 & 71.0 & 41.2 & 43.9 & 49.4 & 47.7 & 25.0 \\
$\checkmark$ & & & & 74.9 & 79.6 & 75.0 & 61.6 & 48.7 & 50.4 & 50.3 & 42.0 \\
$\checkmark$ & $\checkmark$ & & & 76.9 & 79.8 & 76.1 & 71.4 & 51.0 & 50.6 & 51.6 & 50.5 \\
$\checkmark$ & $\checkmark$ & $\checkmark$ & & 77.0 & 79.9 & 76.1 & 71.7 & 51.5 & 50.8 & 52.0 & 51.6 \\
$\checkmark$ & $\checkmark$ & $\checkmark$ & $\checkmark$ &\bf 78.3 &\bf 81.0 &\bf 77.5 &\bf 73.4 &\bf 52.1 &\bf 51.5 &\bf 52.8 &\bf 51.7 \\
\bottomrule
\end{tabular}
\end{small}
\end{table}

\textbf{Effect of Each Component.} To assess the effectiveness of each component, we conduct a systematic ablation study on key components in \algo\, including \romannumeral 1) structured lightweight fine-tuning (SLF), \romannumeral 2) semantic-aware initialization (SAI), \romannumeral 3) minimalist data augmentation (MDA), and \romannumeral 4) test-time ensembling (TTE). The results are presented in \Cref{table:ablation}. Specifically, \romannumeral 1) SLF substantially improves performance across all classes; \romannumeral 2) SAI consistently improves generalization, particularly on tail classes; \romannumeral 3) MDA further boosts model performance. More crucially, MDA significantly reduces training durations, as will be discussed in \Cref{sec:component_analysis}; \romannumeral 4) TTE provides additional performance gains for both head and tail classes.

\begin{figure}[!t]
\centering
\subfloat[ImageNet-LT]{
    \includegraphics[width=0.32\linewidth]{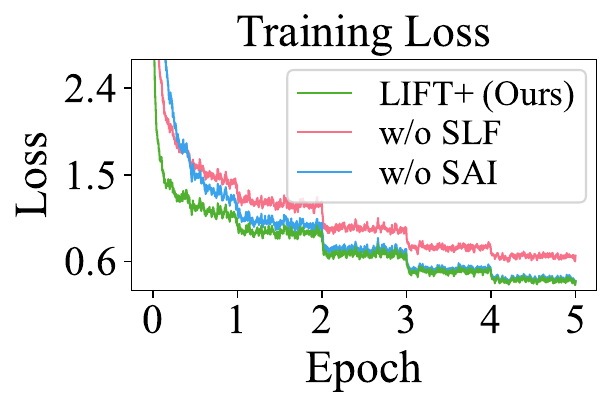}
    \hfill
    \includegraphics[width=0.32\linewidth]{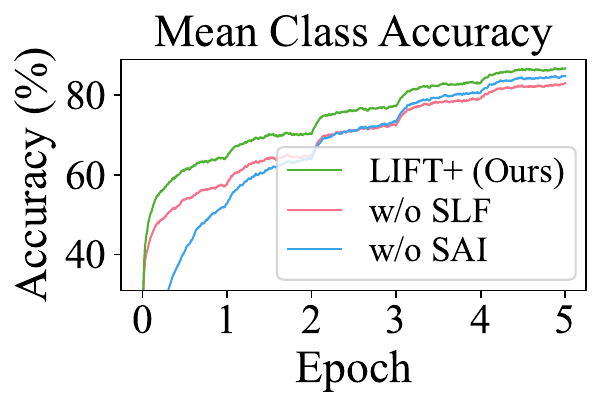}
    \hfill
    \includegraphics[width=0.32\linewidth]{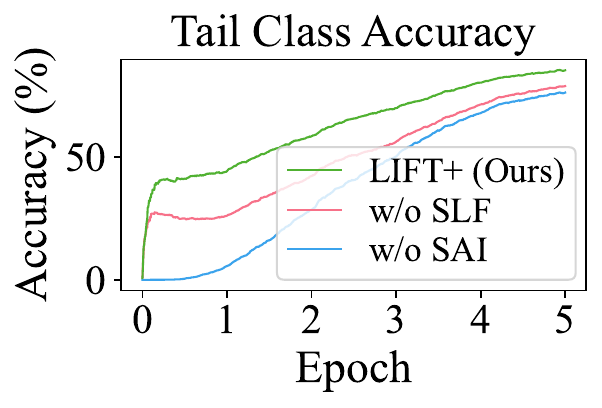}
}
\\
\subfloat[Places-LT]{
    \includegraphics[width=0.32\linewidth]{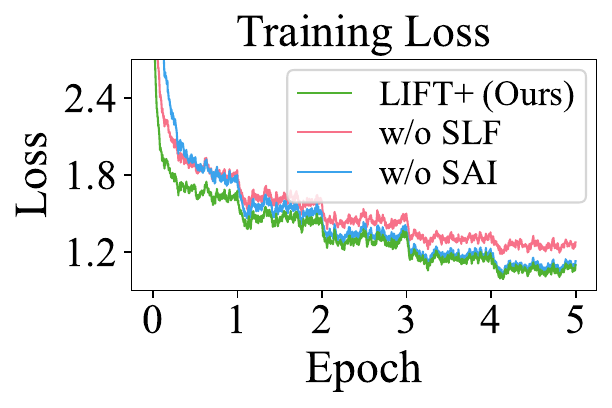}
    \hfill
    \includegraphics[width=0.32\linewidth]{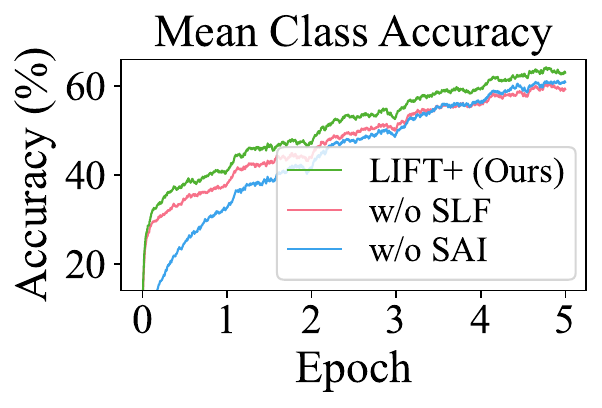}
    \hfill
    \includegraphics[width=0.32\linewidth]{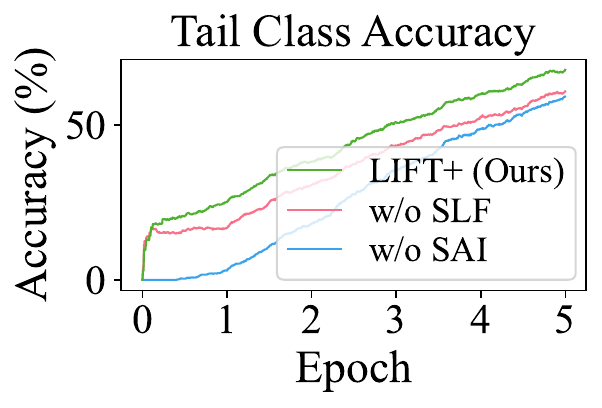}
}
\caption{Convergence curves of training loss and accuracy.}    
\label{fig:convergence}
\end{figure}

\textbf{Impact on Convergence Efficacy.} In \Cref{fig:convergence}, we illustrate the convergence curve of training loss, mean class accuracy, and tail class accuracy. The results show that, without structured lightweight fine-tuning (SLF), the training loss and accuracy converge suboptimally in all cases. Without semantic-aware initialization (SAI), the training loss is slightly affected, while the class accuracy decreases by a large margin, especially on tail classes. These results underscore the significance of SLF and SAI in ensuring optimal convergence.

\begin{figure}[!t]
    \centering
    \includegraphics[width=0.325\linewidth,trim=6 0 6 0]{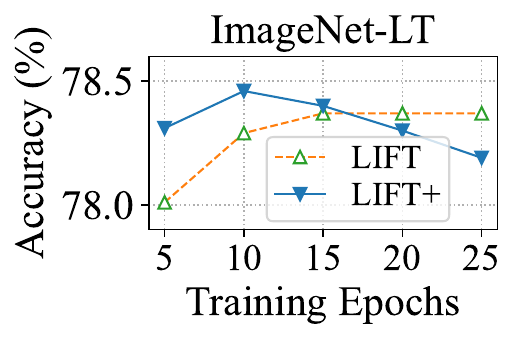}
    \hfill
    \includegraphics[width=0.325\linewidth,trim=6 0 6 0]{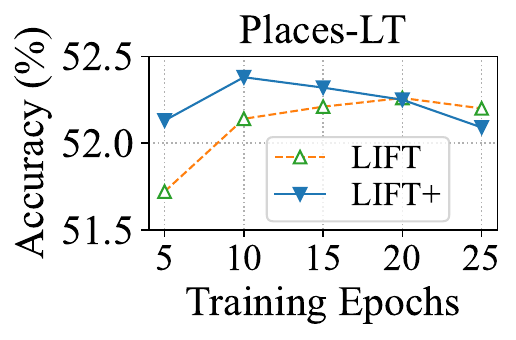}
    \hfill
    \includegraphics[width=0.309\linewidth,trim=6 0 0 0]{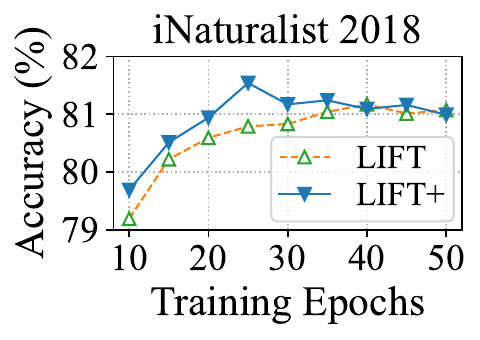}
    \caption{Performance of \algo\ with different training epochs.}
    \label{fig:comp_epochs}
\end{figure}

\textbf{Comparison of Different Training Epochs.}
To investigate the training efficiency of \algo, we run LIFT and \algo\ with different epochs and report the results in \Cref{fig:comp_epochs}. The learning rate for each experiment is modified to $\mathrm{lr}=(\mathrm{base\_epochs}\cdot\mathrm{base\_lr})/\mathrm{epochs}$. On ImageNet-LT and Places-LT, while LIFT maintains high accuracy under prolonged training durations, it struggles to achieve competitive performance within a few training epochs. In contrast, \algo\ demonstrates superior efficiency, attaining an accuracy comparable to LIFT with only 5 epochs.
On iNaturalist 2018, when training for 5 epochs, \algo\ achieves an overall accuracy of 74.9\% (without TTE) / 76.2\% (with TTE), surpassing Decoder \cite{wang2023exploring} by more than 15\%. We default to train \algo\ for 15 epochs on this dataset; however, further performance gains remain attainable. As illustrated in \Cref{fig:comp_epochs}, when training for more epochs (\eg, 25 epochs), \algo\ achieves an additional performance improvement of 1\%. However, this will increase the computational overhead, so we abort this approach in \algo.

\begin{table*}[!t]
\caption{Performance of \algo\ with different losses.}
\label{table:comp_loss}
\setlength{\tabcolsep}{1ex} 
\centering
\begin{small}
\begin{tabular}{L{8ex}| cccc | cccc | cccc}
\toprule
\multirow{2}{*}{\bf Losses} & \multicolumn{4}{c|}{\bf ImageNet-LT} & \multicolumn{4}{c|}{\bf Places-LT}  & \multicolumn{4}{c}{\bf iNaturalist 2018} \\ 
&\bf All &\bf Head &\bf Med. &\bf Tail &\bf All &\bf Head &\bf Med. &\bf Tail &\bf All &\bf Head &\bf Med. &\bf Tail \\
\midrule
CE & 72.1 & 86.0 & 69.0 & 43.3 & 41.8 & 56.3 & 37.6 & 24.8 & 74.7 & 82.8 & 75.8 & 71.2 \\
Focal & 72.2 & 85.3 & 69.3 & 45.3 & 42.3 & 55.8 & 38.2 & 26.7 & 72.4 & 80.9 & 73.7 & 68.4 \\
LDAM & 70.0 &\bf 86.1 & 67.1 & 35.0 & 40.5 &\bf 56.7 & 36.1 & 20.8 & 76.3 &\bf 84.2 & 77.3 & 72.9 \\
CB & 76.8 & 82.2 & 76.3 & 63.5 & 50.0 & 52.0 & 51.3 & 43.0 & 78.6 & 73.5 & 79.1 & 79.3 \\
GRW & 76.9 & 82.2 & 76.3 & 63.6 & 50.0 & 51.9 & 51.5 & 43.0 & 78.7 & 73.7 & 79.2 & 79.4 \\
LADE & 78.0 & 80.6 & 77.1 &\bf 74.0 & 51.9 & 51.7 & 52.1 & 51.5 &\bf 80.9 & 75.5 &\bf 80.7 &\bf 82.4 \\
LA &\bf 78.3 & 81.0 &\bf 77.5 & 73.4 &\bf 52.1 & 51.5 &\bf 52.7 &\bf 51.7 & 80.5 & 74.4 & 80.4 & 82.3 \\
\bottomrule
\end{tabular}
\end{small}
\end{table*}

\textbf{Comparison of Different Losses.} Building upon our theoretical analysis, we propose to optimize the unbiased logit-adjusted loss (LA). To further validate our proposal, we compare the performance of \algo\ with different losses, including cross-entropy loss (CE), focal loss \cite{lin2017focal}, label-distribution-aware margin loss (LDAM) \cite{cao2019learning}, class-balanced loss (CB) \cite{cui2019class}, generalized re-weighting loss (GRW) \cite{zhang2021distribution}, and label distribution disentangling loss (LADE) \cite{hong2021disentangling}. The results are shown in \Cref{table:comp_loss}, where the LA loss achieves the highest average overall accuracy across these three datasets. In contrast, other losses, such as LDAM, CB, and GRW, cannot achieve satisfactory performance in all cases. The LADE performs well on iNaturalist 2018 but is suboptimal on the other two datasets. Moreover, its form is far more complex compared with LA.

\begin{table*}[!t]
\caption{Performance of \algo\ with different classifiers.}
\label{table:comp_classifier}
\setlength{\tabcolsep}{1ex} 
\centering
\begin{small}
\begin{tabular}{l|c| cccc | cccc | cccc}
\toprule
\multicolumn{2}{l|}{\multirow{2}{*}{\bf Classifiers}} & \multicolumn{4}{c|}{\bf ImageNet-LT} & \multicolumn{4}{c|}{\bf Places-LT} & \multicolumn{4}{c}{\bf iNaturalist 2018} \\ 
\multicolumn{2}{l|}{} &\bf All &\bf Head &\bf Med. &\bf Tail &\bf All &\bf Head &\bf Med. &\bf Tail &\bf All &\bf Head &\bf Med. &\bf Tail  \\
\midrule
\multicolumn{2}{l|}{Linear} & 78.2 & 81.1 & 77.3 & 73.1 &\bf 52.1 & 51.7 & 52.4 & 51.9 & 76.7 &\bf 75.7 & 77.7 & 75.8 \\
\multicolumn{2}{l|}{L2-normalized} & 78.3 & 81.0 & 77.2 &\bf 74.6 &\bf 52.1 & 51.4 & 52.5 &\bf 52.5 & 80.1 & 74.9 & 79.9 & 81.7 \\
\multirow{5}{*}{Cosine} & $\sigma=15$ & 75.7 & 80.4 & 76.5 & 59.5 & 49.6 &\bf 52.4 & 53.1 & 36.7 & 76.8 & 73.2 & 77.0 & 77.3 \\
 & $\sigma=20$ & 77.7 & 80.7 & 77.4 & 70.2 & 51.6 & 51.7 &\bf 53.3 & 47.7 & 79.9 & 74.7 & 79.6 & 81.7 \\
 & $\sigma=25$ & 78.3 & 81.0 &\bf 77.5 & 73.4 &\bf 52.1 & 51.5 & 52.7 & 51.7 &\bf 80.5 & 74.4 &\bf 80.4 &\bf 82.3 \\
 & $\sigma=30$ &\bf 78.4 & 81.3 & 77.4 & 73.8 & 51.9 & 51.2 & 52.3 & 52.2 & 80.4 & 75.1 &\bf 80.4 & 81.7 \\
 & $\sigma=35$ & 78.3 &\bf 81.6 & 77.0 & 73.6 & 51.5 & 51.0 & 51.9 & 51.7 & 79.8 & 74.9 & 79.8 & 80.9 \\
\bottomrule
\end{tabular}
\end{small}
\end{table*}

\begin{figure}[!t]
    \centering
    \includegraphics[width=0.32\linewidth,trim=6 0 6 0]{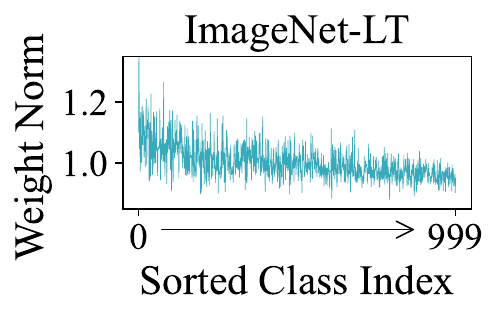}
    \hfill
    \includegraphics[width=0.32\linewidth,trim=6 0 6 0]{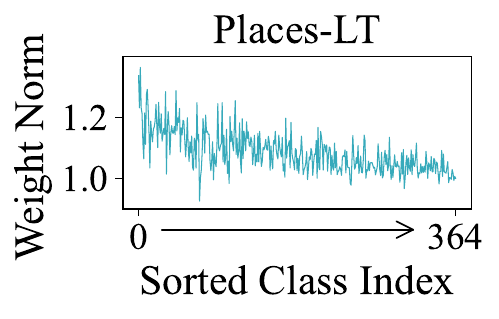}
    \hfill
    \includegraphics[width=0.325\linewidth,trim=6 0 6 0]{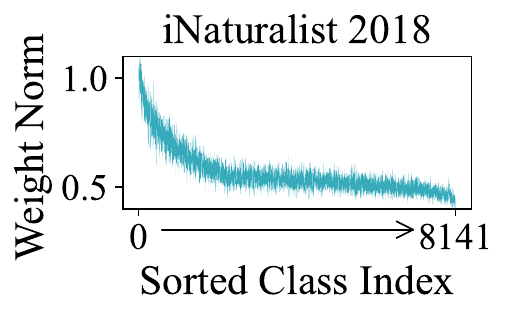}
    \caption{Weight norms of the linear classifier on three long-tail datasets. Classes are sorted by their frequency in the training dataset. On iNaturalist 2018, the weight norms are much more skewed, leading to a suboptimal performance.}
    \label{fig:weight_norm}
\end{figure}

\textbf{Comparison of Different Classifiers.}
We propose to use an enhanced cosine classifier in \algo.
Furthermore, we investigate the linear classifier $z_k=\vw_k^{\top}\vf+b$, the L2-normalized classifier $z_j=\frac{\vw_j^{\top}}{\Vert \vw_j\Vert_2}\vf$, as well as the cosine classifier with $\sigma\in\{15, 20, 25, 30, 35\}$. 
The results in \Cref{table:comp_classifier} show that the linear classifier performs well on ImageNet-LT and Places-LT, but diminishes on the more challenging iNaturalist 2018 dataset. This can be inferred from the skewed distribution of classifier weight norms illustrated in \Cref{fig:weight_norm}.
By removing the impact of weight norms, the L2-normalized classifier and the cosine classifier achieve higher performance, especially on tail classes. Moreover, when adopting the cosine classifier, setting $\sigma$ to $25$ or $30$ leads to superior performance.

\begin{table}[!t]
\caption{Comparison of \algo\ with different classifier initialization methods.}
\label{table:comp_clf_init}
\setlength{\tabcolsep}{0.25ex} 
\centering
\begin{small}
\begin{tabular}{l| cccc | cccc}
\toprule
\multirow{2}{*}{\bf Methods} & \multicolumn{4}{c|}{\bf ImageNet-LT} & \multicolumn{4}{c}{\bf Places-LT} \\ 
&\bf All &\bf Head &\bf Med. &\bf Tail &\bf All &\bf Head &\bf Med. &\bf Tail \\
\midrule
Random & 76.5 & 80.9 & 76.4 & 64.4 & 50.2 & 51.3 & 51.8 & 44.3 \\
Linear probing & 77.1 &\bf 81.8 & 76.4 & 66.3 & 49.9 &\bf 51.5 & 50.8 & 44.8 \\
Class mean features & 77.5 &\bf 81.3 & 76.8 & 69.4 & 51.2 & 51.3 & 52.2 & 48.6 \\
Semantic-aware &\bf 78.3 & 81.0 &\bf 77.5 &\bf 73.4 &\bf 52.1 &\bf 51.5 &\bf 52.8 &\bf 51.7 \\
\bottomrule
\end{tabular}
\end{small}
\end{table}

\textbf{Comparison of Different Initialization Strategies.} In \Cref{table:comp_clf_init}, we evaluate four types of classifier initialization strategies. The linear probing strategy yields limited improvements due to the inherent challenge posed by long-tail data. Using class mean features to initialize the classifier achieves a notable improvement, particularly on tail classes. In \algo, we adopt textual features because they transfer semantic relations between classes during fine-tuning. The results show that our proposed strategy is significantly superior to other methods. This experiment also demonstrates that a good starting point for parameter optimization can lead to a better solution.

\begin{table}[!t]
\caption{Definitions of different scheduling functions.}
\label{tab:mda_func}
\begin{minipage}{0.60\linewidth}
\setlength{\tabcolsep}{1.5ex} 
\centering
\begin{small}
\begin{tabular}{ll}
\toprule
No. of func. & Definitions \\
\midrule
\ding{192} (Minimal) & $g(t)=0$ \\
\ding{193} (Convex) & $g(t)=1-\sqrt{1-t^2}$ \\
\ding{194} (Linear) & $g(t)=t$ \\
\ding{195} (Concave) & $g(t)=\sqrt{1-(1-t)^2}$ \\
\ding{196} (Maximal) & $g(t)=1$ \\
\bottomrule
\end{tabular}
\end{small}
\end{minipage}
\hfill
\begin{minipage}{0.36\linewidth}
\centering
\includegraphics[width=\linewidth,trim=0 6 0 0]{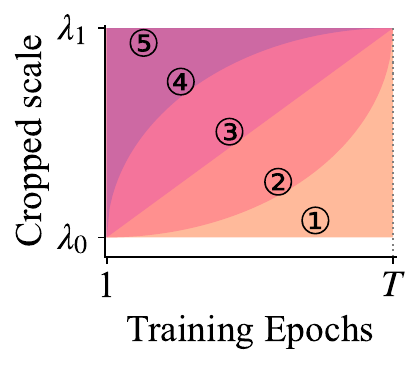}
\end{minipage}
\end{table}

\begin{table}[!t]
\caption{Comparison of \algo\ with different scheduling functions for MDA.}
\label{tab:comp_mda_func}
\setlength{\tabcolsep}{0.23ex} 
\centering
\begin{small}
\begin{tabular}{c| cccc | cccc | cccc}
\toprule
\multirow{2}{*}{\footnotesize\bf No.} & \multicolumn{4}{c|}{\bf ImageNet-LT} & \multicolumn{4}{c|}{\bf Places-LT} & \multicolumn{4}{c}{\bf iNaturalist 2018} \\ 
&\footnotesize\bf All &\footnotesize\bf Head &\footnotesize\bf Med. &\footnotesize\bf Tail &\footnotesize\bf All &\footnotesize\bf Head &\footnotesize\bf Med. &\footnotesize\bf Tail &\footnotesize\bf All &\footnotesize\bf Head &\footnotesize\bf Med. &\footnotesize\bf Tail \\
\midrule
\ding{192} & 78.1 & 80.8 & 77.2 & 73.1 & 51.9 & 51.5 & 52.7 & 50.7 & 80.1 & 73.6 & 79.9 & 82.2 \\
\ding{193} &\bf 78.3 & 81.0 &\bf 77.5 & 73.4 &\bf 52.1 & 51.5 &\bf 52.8 &\bf 51.7 &\bf 80.5 & 74.4 &\bf 80.4 &\bf 82.3 \\
\ding{194} &\bf 78.3 & 81.2 & 77.3 &\bf 73.5 &\bf 52.1 &\bf 51.7 & 52.7 &\bf 51.7 & 80.0 & 74.7 & 79.8 & 81.7 \\
\ding{195} &\bf 78.3 & 81.3 & 77.4 &\bf 73.5 & 52.0 & 51.3 & 52.7 & 51.6 & 79.9 & 75.1 & 79.7 & 81.4 \\
\ding{196} &\bf 78.3 &\bf 81.4 & 77.3 & 73.1 & 51.8 & 51.5 & 52.2 & 51.4 & 79.8 &\bf 75.8 & 79.7 & 80.9 \\
\bottomrule
\end{tabular}
\end{small}
\end{table}

\textbf{Comparison of Different Scheduling Functions.} In the proposed MDA module, the cropped scale range starts from $(\lambda_0, \lambda_1)$ and gradually converges to the maximum value $\lambda_1$ by adding a variable $\Delta\lambda_0$. We compare five different scheduling functions for $\Delta\lambda_0$, including minimal, convex, linear, concave, and maximal functions. The definition and curves are shown in \Cref{tab:mda_func}, and the comparison results are reported in \Cref{tab:comp_mda_func}. On ImageNet-LT and Places-LT, the performance difference is not significant. On iNaturalist 2018, the convex function exhibits an obvious enhancement.

\begin{figure}[!t]
    \centering
    \includegraphics[width=0.32\linewidth, trim=6 0 6 0]{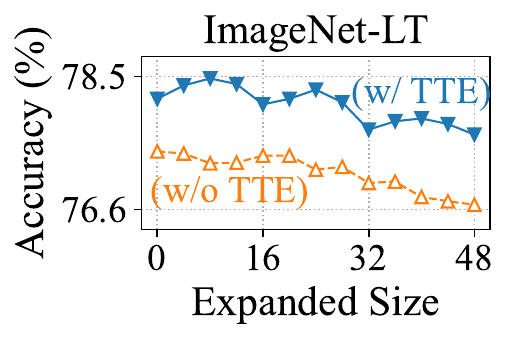}
    \hfill
    \includegraphics[width=0.32\linewidth, trim=6 0 6 0]{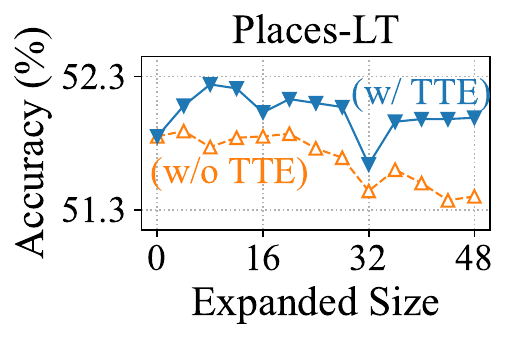}
    \hfill
    \includegraphics[width=0.32\linewidth, trim=6 0 6 0]{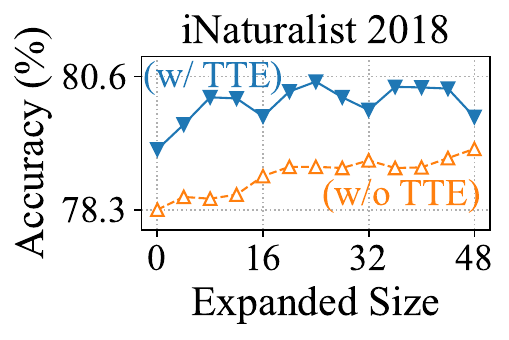}
    \caption{Performance with different expanded size $e$. When applying TTE, setting $e$ to a multiple of the patch size $16$ yields suboptimal improvements.}
    \label{fig:expand_size}
\end{figure}

\textbf{Comparison of Different Expanded Sizes.} To analyze the effect of TTE, we compare different expanded sizes $e$ and report the results in \Cref{fig:expand_size}. Remarkably, TTE consistently improves performance with different expanded sizes. Note that $e=0$ still achieves improvements on ImageNet-LT and iNaturalist 2018, because the aspect ratios of the images in these two datasets are not always 1:1, which enables TTE to generate partially diverse croppings. More crucially, setting the expanded size as a multiple of the patch size $16$ results in suboptimal enhancement, as the generated images will share many repeated patches, thus reducing patch diversity. Generally, setting $e=24$ achieves significant improvements on different datasets.

\subsection{More Advantages of \algo}

\begin{table}[!t]
\caption{Comparison of different fine-tuning methods. All baselines use SAI, MDA, and TTE for a fair comparison.}
\label{table:comp_peft_module}
\setlength{\tabcolsep}{0.22ex} 
\centering
\begin{small}
\begin{tabular}{l|l | cccc | cccc}
\toprule
\multicolumn{2}{l|}{\multirow{2}{*}{\bf Methods}} & \multicolumn{4}{c|}{\bf ImageNet-LT} & \multicolumn{4}{c}{\bf Places-LT} \\ 
\multicolumn{2}{l|}{}&\bf All &\bf Head &\bf Med. &\bf Tail &\bf All &\bf Head &\bf Med. &\bf Tail \\
\midrule
\multicolumn{2}{l|}{Zero-shot CLIP} & 67.9 & 69.0 & 67.3 & 67.2 & 39.9 & 37.5 & 39.1 & 45.9 \\
\multicolumn{2}{l|}{Classifier fine-tuning} & 74.6 & 77.9 & 73.8 & 67.7 & 49.6 & 49.3 & 50.7 & 47.8 \\
\multicolumn{2}{l|}{Full fine-tuning} & 73.2 &\bf 81.3 & 72.8 & 51.9 & 47.5 &\bf 51.9 & 49.4 & 35.0 \\
\multicolumn{2}{l|}{\algo\ (Arbitrary)} & 77.7 & 80.9 & 76.8 & 72.5 & 51.8 & 51.7 & 52.4 & 50.8 \\
\arrayrulecolor{gray}\midrule\arrayrulecolor{black}
\multirow{5}{*}{\makecell[l]{\algo\\~w/}} & BitFit & 77.0 & 79.8 & 76.3 & 71.4 & 51.6 & 50.8 & 52.6 & 50.7 \\
& VPT & 77.6 & 80.1 & 76.8 & 72.9 & 51.9 & 51.4 & 52.7 & 51.2 \\
& Adapter & 78.0 & 81.0 & 77.1 & 72.6 & 51.9 & 51.4 & 52.7 & 51.1 \\
& LoRA & 77.9 & 80.4 & 76.9 &\bf 73.9 & 51.6 & 51.4 & 52.1 & 50.9 \\ 
& AdaptFormer &\bf 78.3 & 81.0 &\bf 77.5 & 73.4 &\bf 52.1 & 51.5 &\bf 52.8 &\bf 51.7 \\
\bottomrule
\end{tabular}
\end{small}
\end{table}

\textbf{Adaptability to Varying Fine-Tuning Methods.} \algo\ is a general framework in which many lightweight fine-tuning methods can be integrated. In addition to zero-shot CLIP, classifier fine-tuning, and full fine-tuning, we test \algo\ with arbitrary lightweight fine-tuning and five structured lightweight methods,
BitFit \cite{zaken2022bitfit}, 
VPT \cite{jia2022visual}, 
Adapter \cite{houlsby2019parameter}, 
LoRA \cite{hu2022lora}, 
and AdaptFormer \cite{chen2022adaptformer}.
The results in \Cref{table:comp_peft_module} demonstrate that arbitrary lightweight fine-tuning surpasses the baseline methods by a large margin across both overall and tail-class performance. This underscores the efficacy of lightweight fine-tuning even in the absence of specific strategies. Furthermore, the integration of structured methods leads to improved performance. Specifically, \algo\ with AdaptFormer performs best on both datasets.

\begin{table}[!t]
\caption{Results on ImageNet-LT with ViT-L/14.}
\label{table:vitl_imagenetlt}
\setlength{\tabcolsep}{0.12ex} 
\centering
\begin{small}
\begin{tabular}{l|c|c|c|cccc}
\toprule
\bf Methods &\bf Backbone &\makecell{\bf Learned \\ \bf Params.} &\bf Epochs &\bf All &\bf Head &\bf Med. &\bf Tail \\
\midrule
CLIP & ViT-L/14 & - & - & 73.6 & 74.6 & 73.1 & 72.3 \\
Decoder \cite{wang2023exploring} & ViT-L/14 & 39.79M & $\sim$18 & 79.3 & - & - & - \\
\algo\ (Ours) & ViT-L/14 &\bf 0.86M &\bf 5 & 82.5 & 84.7 & 81.7 & 78.8 \\
\ w/ TTE (Ours) & ViT-L/14 &\bf 0.86M &\bf 5 &\bf 83.1 &\bf 85.4 &\bf 82.3 &\bf 79.4 \\
\bottomrule
\end{tabular}
\end{small}
\vskip -0.05in
\end{table}

\begin{table}[!t]
\caption{Results on Places-LT with ViT-L/14.}
\label{table:vitl_placeslt}
\setlength{\tabcolsep}{0.12ex} 
\centering
\begin{small}
\begin{tabular}{l|c|c|c|cccc}
\toprule
\bf Methods &\bf Backbone &\makecell{\bf Learned \\ \bf Params.} &\bf Epochs &\bf All &\bf Head &\bf Med. &\bf Tail \\
\midrule
CLIP & ViT-L/14 & - & - & 39.9 & 38.1 & 39.2 & 45.1 \\
Decoder \cite{wang2023exploring} & ViT-L/14 & 39.79M & $\sim$34 & 48.4 & - & - & - \\
\algo\ (Ours) & ViT-L/14 &\bf 0.27M &\bf 5 & 53.3 & 52.3 & 53.9 & 53.6 \\
\ w/ TTE (Ours) & ViT-L/14 &\bf 0.27M &\bf 5 &\bf 53.7 &\bf 52.9 &\bf 54.3 &\bf 54.0 \\
\bottomrule
\end{tabular}
\end{small}
\vskip -0.05in
\end{table}

\begin{table}[!t]
\caption{Results on iNaturalist 2018 with ViT-L/14.}
\label{table:vitl_inat18}
\setlength{\tabcolsep}{0.12ex} 
\centering
\begin{small}
\begin{tabular}{l|c|c|c|cccc}
\toprule
\bf Methods &\bf Backbone &\makecell{\bf Learned \\ \bf Params.} &\bf Epochs &\bf All &\bf Head &\bf Med. &\bf Tail \\
\midrule
CLIP & ViT-L/14 & - & - & 5.8 & 11.1 & 5.4 & 4.9 \\
Decoder \cite{wang2023exploring} & ViT-L/14 & 39.79M &\bf $\sim$5 & 72.3 & 65.5 & 73.2 & 73.0 \\
\algo\ (Ours) & ViT-L/14 &\bf 6.37M & 15 & 84.0 & 79.1 & 83.9 & 85.4 \\
\ w/ TTE (Ours) & ViT-L/14 &\bf 6.37M & 15 &\bf 84.9 &\bf 79.3 &\bf 84.7 &\bf 86.5 \\
\bottomrule
\end{tabular}
\end{small}
\vskip -0.05in
\end{table}

\begin{table}[!t]
\caption{Results on iNaturalist 2018 with ViT-L/14 (336$\times$336 pixels resolution).}
\label{table:vitl_inat18_336}
\setlength{\tabcolsep}{0.12ex} 
\centering
\begin{small}
\begin{tabular}{l|c|c|c|cccc}
\toprule
\bf Methods &\bf Backbone &\makecell{\bf Learned \\ \bf Params.} &\bf Epochs &\bf All &\bf Head &\bf Med. &\bf Tail \\
\midrule
CLIP & \multirow{3}{*}[-0.33ex]{\makecell{ViT-L/14 \\ (336px)}} & - & - & 6.2 & 11.5 & 5.8 & 5.3 \\
\algo\ (Ours) & &\bf 6.37M &\bf 15 & 86.8 & 82.4 & 86.9 & 87.7\\
\ w/ TTE (Ours) & &\bf 6.37M &\bf 15 &\bf 87.3 &\bf 83.4 &\bf 87.2 &\bf 88.3\\
\bottomrule
\end{tabular}
\end{small}
\end{table}

\begin{table}[!t]
\caption{Results on ImageNet-LT with ResNet-50.}
\label{table:resnet_imagenetlt}
\setlength{\tabcolsep}{0.12ex} 
\centering
\begin{small}
\begin{tabular}{l|c|c|c|cccc}
\toprule
\bf Methods &\bf Backbone &\makecell{\bf Learned \\ \bf Params.} &\bf Epochs &\bf All &\bf Head &\bf Med. &\bf Tail \\
\midrule
CLIP & RN-50 & - & - & 57.6 & 58.7 & 56.9 & 56.8 \\
\algo\ w/ SSF & RN-50 & 0.002M & 10 & 65.8 &\bf 70.9 & 65.0 & 53.9 \\
\algo\ w/ BitFit & RN-50 & 0.026M & 10 &\bf 65.9 & 70.1 &\bf 65.1 &\bf 57.2 \\
\bottomrule
\end{tabular}
\end{small}
\vskip -0.05in
\end{table}

\begin{table}[!t]
\caption{Results on Places-LT with ResNet-50.}
\label{table:resnet_placeslt}
\setlength{\tabcolsep}{0.12ex} 
\centering
\begin{small}
\begin{tabular}{l|c|c|c|cccc}
\toprule
\bf Methods &\bf Backbone &\makecell{\bf Learned \\ \bf Params.} &\bf Epochs &\bf All &\bf Head &\bf Med. &\bf Tail \\
\midrule
CLIP & RN-50 & - & - & 35.2 & 33.1 & 34.6 & 40.5 \\
\algo\ w/ SSF & RN-50 & 0.002M & 5 & 45.8 & 46.6 & 47.4 & 40.9 \\
\algo\ w/ BitFit & RN-50 & 0.026M & 10 &\bf 46.6 &\bf 47.3 &\bf 47.9 &\bf 42.2 \\
\bottomrule
\end{tabular}
\end{small}
\end{table}

\textbf{Extension to Different Backbones.} \algo\ can be extended to various backbones. In addition to ViT-B/16, we also assess \algo\ based on the larger ViT-L/14. The results in \Cref{table:vitl_imagenetlt,table:vitl_placeslt,table:vitl_inat18,table:vitl_imagenetlt} show that \algo\ surpasses the state-of-the-art method Decoder \cite{wang2023exploring} by a considerable margin, showcasing improvements of 3.2\% on ImageNet-LT, 4.9\% on Places-LT, and 11.7\% on iNaturalist 2018. Furthermore, the model with higher resolution (336$\times$336 pixels) yields better performance. These results underscore the adaptability of \algo\ across variant backbones.

The adaptation of \algo\ to the widely used ResNet \cite{he2016deep} is also worth exploring. However, due to the absence of a dedicated lightweight fine-tuning method tailored for ResNet, it is challenging to directly integrate \algo\ with ResNet. Despite this limitation, we investigate some intuitive strategies, including \romannumeral 1) incorporating a scaling and shifting (SSF) \cite{lian2022scaling} module after the backbone, and \romannumeral 2) fine-tuning solely the bias terms (BitFit) of ResNet. The results are presented in \Cref{table:resnet_imagenetlt,table:resnet_placeslt}. \algo\ delivers remarkable performance enhancements over zero-shot CLIP. Furthermore, when compared with conventional ResNet-based approaches in \Cref{table:comp_imagenetlt,table:comp_placeslt}, \algo\ achieves superior results with fewer learned parameters and training epochs.

\section{Technical Details of Lightweight Fine-Tuning For Long-Tail Learning}
\label{sec:tech_detail}

\textbf{Preliminary to the Transformer Architecture.} 
The Transformer model architecture comprises an embedding layer followed by a series of Transformer blocks. Specifically, it first partitions an input image $\vx$ into $m$ discrete patches $\{\vx_i\}_{i=1}^m$. These patches are then embedded into sequences of $d$-dimensional tokens $\mE^0=\mathrm{Embedding}([\vx_1;\cdots;\vx_m])\in\sR^{m\times d}$. The embeddings are subsequently passed through $L$ successive Transformer blocks $\{\Phi^l\}_{l=1}^L$ within the model:
\begin{equation}
\label{eq:block}
\begin{split}
\mX^l = \Phi^l(\mX^{l-1})
\left \{
\begin{array}{ll}
\!\hat{\mX}^l=\mathrm{MSA}(\mathrm{LN}(\mX^{l-1}))+\mX^{l-1}\\
\!\mX^l=\mathrm{FFN}(\mathrm{LN}(\hat{\mX}^l))+\hat{\mX}^l\\
\end{array}
\right.
\end{split}
\end{equation}
In practice, $\mX^0$ is normally set as $\mE^0$, and will prepend a learnable token $\vc^0\in\sR^{d}$ when performing classification tasks, \ie, $\mX^0=[\vc^0;\mE^0]$. The feature is extracted from the same position of the last-layer sequence, which is $\phi(\vx)=\mathrm{LN}(\vc^L)$.
Moreover, $\mathrm{MSA}$ represents multi-head self-attention, $\mathrm{FFN}$ indicates feed-forward network, and $\mathrm{LN}$ denotes layer normalization \cite{ba2016layer}. The definitions are
\begin{gather}
\label{eq:msa}
\mathrm{MSA}(\mX)=\mathrm{Concat}_{h=1}^{H}\left(\mathrm{Softmax}\left(\frac{\mQ_h\mK_h^\top}{\sqrt{d/H}}\right)\mV_h\right)\mW^O\notag\\
\mathrm{where}\;\left(\mQ_h,\mK_h,\mV_h\right)=\left(\mX\mW^Q_h,\mX\mW^K_h,\mX\mW^V_h\right)
\end{gather}
\begin{gather}
\label{eq:ffn}
\mathrm{FFN}(\mX)=\mathrm{ReLU}(\mX\mW_1)\mW_2
\end{gather}
\begin{gather}
\label{eq:ln}
\mathrm{LN}(\mX_i)=\frac{\mX_i-\mathrm{E}[\mX_i]}{\sqrt{\mathrm{Var}[\mX_i]}}\cdot\boldsymbol{\gamma}+\boldsymbol{\beta}
\end{gather}
where $H$ is the number of heads, $\mW^Q_h,\mW^K_h,\mW^V_h\in\sR^{d\times \frac{d}{H}}$, $\mW^O\in\sR^{d\times d}$, $\mW_1\in\sR^{d\times 4d}$, $\mW_2\in\sR^{4d\times d}$, $\boldsymbol{\gamma}\in\sR^d$ and $\boldsymbol{\beta}\in\sR^d$ are model parameters. Note that the bias terms are omitted for simplification. 

\begin{table}[!t]
\caption{Components and parameter quantity of ViT.}
\label{table:vit_quantity}
\setlength{\tabcolsep}{1ex} 
\centering
\begin{small}
\begin{tabular}{c|c|c|c}
\toprule
\bf Layers &\bf Components & \bf Variables &\bf \#Params. \\
\midrule
\multirow{3}{*}[-1.33ex]{Embedding} & $\mathrm{Projection}$ & - & $(d_0+1)d$ \\
\cmidrule{2-4}
 & $\mathrm{Class~Token}$ & $\vc^0$ & $d$ \\
\cmidrule{2-4}
 & $\mathrm{Positional}$ & - & $(m+1)d$ \\
\midrule
\multirow{8}{*}[-4.66ex]{\makecell{Block-$1$ \\ ($l=1$)}} & $\mathrm{LN}$ & $\boldsymbol{\gamma},\boldsymbol{\beta}$ & $2d$ \\
\cmidrule{2-4}
 & \multirow{4}{*}[-2ex]{$\mathrm{MSA}$} & $\{\mW^Q_h,\vb^Q_h\}_{h=1}^{H}$ & \multirow{4}{*}[-2ex]{$4d^2+4d$} \\ [1ex]
 & & $\{\mW^K_h,\vb^K_h\}_{h=1}^{H}$ \\ [1ex]
 & & $\{\mW^V_h,\vb^V_h\}_{h=1}^{H}$ \\ [1ex]
 & & $\mW^O,\vb^O$ \\
\cmidrule{2-4}
 & $\mathrm{LN}$ & $\boldsymbol{\gamma},\boldsymbol{\beta}$ & $2d$ \\
\cmidrule{2-4}
 & \multirow{2}{*}[-0.66ex]{$\mathrm{FFN}$} & $\mW_1,\vb_1$ & \multirow{2}{*}[-0.66ex]{$8d^2+5d$} \\ [1ex]
 & & $\mW_2,\vb_2$ \\ 
\midrule
$\vdots$ & $\vdots$ & $\vdots$ & $\vdots$ \\
\midrule
Block-$L$ & $\cdots$ & $\cdots$ & $12d^2+13d$ \\
\midrule
Normalization & $\mathrm{LN}$ & $\boldsymbol{\gamma},\boldsymbol{\beta}$ & $2d$\\
\bottomrule
\end{tabular}
\end{small}
\end{table}

The architectural components and parameter quantities of a Vision Transformer (ViT) are summarized in \Cref{table:vit_quantity}.
The total number of parameters in ViT can be calculated as $12Ld^2+(13L+m+d_0+5)d$, where $L$ represents the number of blocks, $m$ corresponds to the number of image patches, $d_0$ denotes the dimensionality of each image patch, and $d$ indicates the dimensionality of the embedded features. For example, in ViT-B/16 with $224\times 224$ resolution, where $L=12$, $m=\frac{224}{16}\times\frac{224}{16}=196$, $d_0=16\times 16\times 3=768$, and $d=768$, the parameter quantity amounts to approximately 85.80M.
Note that models from different sources (\eg, CLIP or ImageNet-21K pre-trained ViT) may have slight differences due to implementation details, while their results are always within the standard configuration \cite{radford2021clip,dosovitskiy2021an}.

\textbf{\algo\ with Structured Lightweight Fine-Tuning Methods.} In this paper, we explore lightweight fine-tuning by learning only a small set of parameters for adaptation. Concretely, we propose a versatile and inclusive framework \algo, allowing the incorporation of a range of lightweight modules, including but not limited to:
\begin{itemize}
    \item \textit{Bias-terms Fine-tuning} (\textit{BitFit}) \cite{zaken2022bitfit}, which aims to fine-tune only the bias terms of the model. Formally, given a projection $\mX\mW+\vb$, it freezes $\mW$ and optimizes $\vb$.
    \item \textit{Visual Prompt Tuning} (\textit{VPT}) \cite{jia2022visual}, which prepends learnable prompts $\mP^l\in\sR^{p\times d}$ at the entrance of each layer to extend $\mX^l=[\vc^l;\mE^l]$ to $[\vc^l;\mP^l;\mE^l]$. 
    \item \textit{Adapter} \cite{houlsby2019parameter}, which attaches a bottleneck module to $\mathrm{FFN}$ layer to reconstruct $\mathrm{FFN}(\cdot)$ to $\mathrm{Adapter}(\mathrm{FFN}(\cdot))$. The definition is $\mathrm{Adapter}(\mX)=\mathrm{ReLU}(\mathrm{LN}(\mX)\mW_{\text{down}})\mW_{\text{up}}$, where $\mW_{\text{down}}\in\sR^{d\times r}$ and $\mW_{\text{up}}\in\sR^{r\times d}\ (r\ll d)$.
    \item \textit{Low-Rank Adapter} (\textit{LoRA}) \cite{hu2022lora}, which optimizes low-rank matrices $\mW_{\text{down}}$ and $\mW_{\text{up}}$ to reparameterize a matrix $\mW$ to $\mW+\mW_{\text{down}}\mW_{\text{up}}$. In practice, it is often applied to the weights in the $\mathrm{MSA}$ module, such as $\mW_Q$ and $\mW_V$. 
    \item \textit{AdaptFormer} \cite{chen2022adaptformer}, which extends the sequential Adapter to a parallel one. Formally, it computes $s\cdot\mathrm{Adapter}(\hat{\mX}^l)$ and adds it to $\mX^l$ in \Cref{eq:block}, where $s$ is a scaling parameter that can be either specified or learnable\protect\footnotemark[3].
\end{itemize}
\footnotetext[3]{The usage of learnable scaling parameter is not proposed in the main paper but is implemented in their source code: \footnotesize\adaptformerurl}

\begin{table}[!t]
\caption{Parameter quantities for structured lightweight fine-tuning modules in a Transformer block.}
\label{table:peft_quantity}
\setlength{\tabcolsep}{0.7ex} 
\centering
\begin{small}
\begin{tabular}{c|c|c|c}
\toprule
\bf Modules &\bf Components & \bf Variables &\bf \#Params. \\
\midrule
\multirow{8}{*}[-4.66ex]{BitFit} & $\mathrm{LN\text{-}bias}$ & $\boldsymbol{\beta}$ & $d$ \\
\cmidrule{2-4}
 &\multirow{4}{*}[-2ex]{$\mathrm{MSA\text{-}bias}$} & $\{\vb^Q_h\}_{h=1}^{H}$ & \multirow{4}{*}[-2ex]{$4d$} \\ [1ex]
 & & $\{\vb^K_h\}_{h=1}^{H}$ & \\ [1ex]
 & & $\{\vb^V_h\}_{h=1}^{H}$ & \\ [1ex]
 & & $\vb^O$ & \\
\cmidrule{2-4}
 & $\mathrm{LN\text{-}bias}$ & $\boldsymbol{\beta}$ & $d$ \\
\cmidrule{2-4}
 & \multirow{2}{*}[-0.66ex]{$\mathrm{FFN\text{-}bias}$} & $\vb^{l}_1$ & \multirow{2}{*}[-0.66ex]{$5d$} \\ [1ex]
 & & $\vb^{l}_2$ & \\ 
\midrule
VPT & $\mathrm{Prompts}$ & $\mP^l$ & $pd$ \\
\midrule
\multirow{3}{*}[-1.33ex]{Adapter} & $\mathrm{LN}$ & $\boldsymbol{\gamma},\boldsymbol{\beta}$ & $2d$ \\
\cmidrule{2-4}
 & \multirow{2}{*}[-0.66ex]{$\mathrm{Projection}$} & $\mW_{\text{down}}, \vb_{\text{down}}$ & \multirow{2}{*}[-0.66ex]{$(2r+1)d+r$} \\ [1ex]
 & & $\mW_{\text{up}}, \vb_{\text{up}}$ & \\
\midrule
\multirow{4}{*}[-2ex]{LoRA} & \multirow{4}{*}[-2ex]{$\mathrm{Projection}$} & $\mW_{\text{down}}$ (for $\mW_Q$) & \multirow{4}{*}[-2ex]{$4rd$} \\ [1ex]
 & & $\mW_{\text{up}}$ (for $\mW_Q$) \\ [1ex]
 & & $\mW_{\text{down}}$ (for $\mW_V$) \\ [1ex]
 & & $\mW_{\text{up}}$ (for $\mW_V$) \\
\midrule
\multirow{4}{*}[-2ex]{AdaptFormer} & $\mathrm{LN}$ & $\boldsymbol{\gamma},\boldsymbol{\beta}$ & $2d$ \\
\cmidrule{2-4}
 & \multirow{2}{*}[-0.66ex]{$\mathrm{Projection}$} & $\mW_{\text{down}}, \vb_{\text{down}}$ & \multirow{2}{*}[-0.66ex]{$(2r+1)d+r$} \\ [1ex]
 & & $\mW_{\text{up}}, \vb_{\text{up}}$ & \\
\cmidrule{2-4}
 & $\mathrm{Scaling}$ & $s$ & $1$ \\
\bottomrule
\end{tabular}
\end{small}
\end{table}

We summarize the parameter quantities of the above lightweight fine-tuning modules in \Cref{table:peft_quantity}. For all modules, the quantities are at the polynomial level of $d$. Note that the prompt length $p$ for VPT and the bottleneck dimensionality $r$ for Adapter, LoRA, and AdaptFormer are much smaller than $d$. Compared to a standard Transformer block, these lightweight fine-tuning modules are significantly low-complexity ($\gO(d)$ vs. $\gO(d^2)$).

In addition to the lightweight fine-tuning module, a task-specific classifier is required, which contains approximately $Kd$ parameters (omitting bias terms), where $K$ represents the number of classes. In \algo, we set the bottleneck dimensionality $r=2^{\lfloor\log_{2}{(\frac{K}{2L})}\rfloor}\leq\frac{K}{2L}$ for the AdaptFormer module, so that the total parameter quantity is $L\cdot 2rd\leq Kd$ (ignoring constant terms). As a result, it learns even fewer parameters than the classifier.

\begin{figure}[!t]
\centering
\subfloat[CLIP.]{
    \includegraphics[width=0.144\linewidth,trim=12 0 12 12]{figures/heatmap_imagenet_lt_zs.pdf}
    \includegraphics[width=0.144\linewidth,trim=12 12 12 -12]{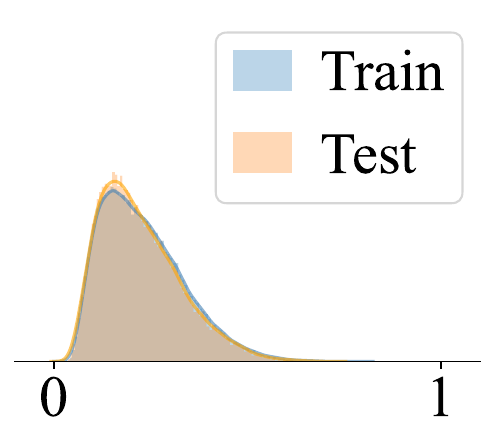}
    \includegraphics[width=0.144\linewidth,trim=12 12 12 -12]{figures/class_scatter_imagenet_lt_zs_tail.pdf}
}
\hfill
\subfloat[Full fine-tuning.]{
    \includegraphics[width=0.144\linewidth,trim=12 0 12 12]{figures/heatmap_imagenet_lt_fft.pdf}
    \includegraphics[width=0.144\linewidth,trim=12 12 12 -12]{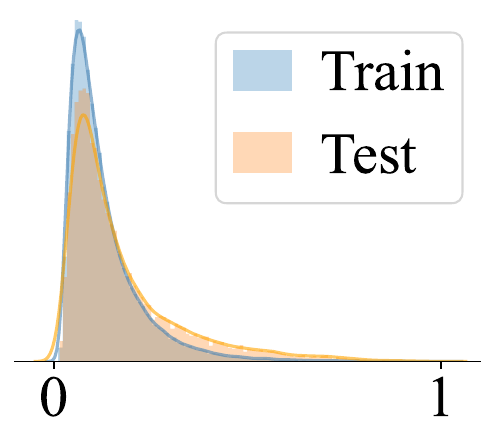}
    \includegraphics[width=0.144\linewidth,trim=12 12 12 -12]{figures/class_scatter_imagenet_lt_fft_tail.pdf}
}
\\
\subfloat[Arbitrary lightweight ft.]{
    \includegraphics[width=0.144\linewidth,trim=12 0 12 12]{figures/heatmap_imagenet_lt_aft.pdf}
    \includegraphics[width=0.144\linewidth,trim=12 12 12 -12]{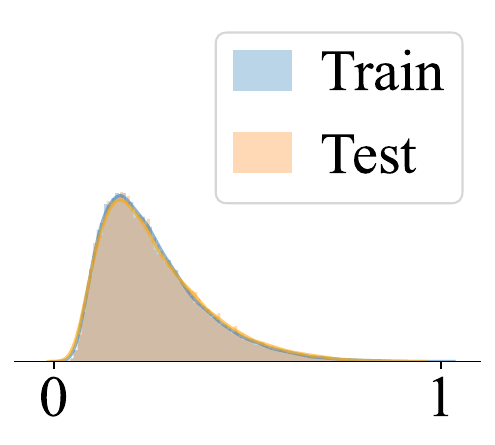}
    \includegraphics[width=0.144\linewidth,trim=12 12 12 -12]{figures/class_scatter_imagenet_lt_aft_tail.pdf}
}
\hfill
\subfloat[BitFit.]{
    \includegraphics[width=0.144\linewidth,trim=12 0 12 12]{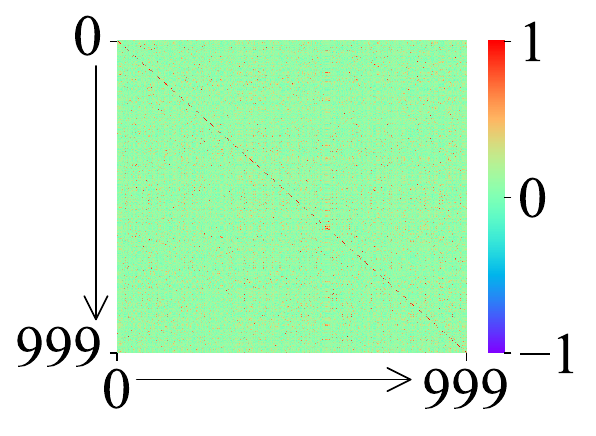}
    \includegraphics[width=0.144\linewidth,trim=12 12 12 -12]{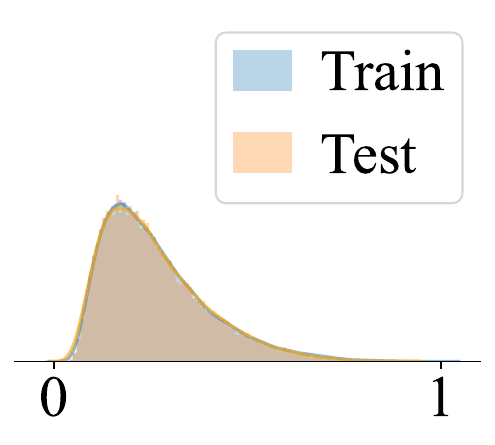}
    \includegraphics[width=0.144\linewidth,trim=12 12 12 -12]{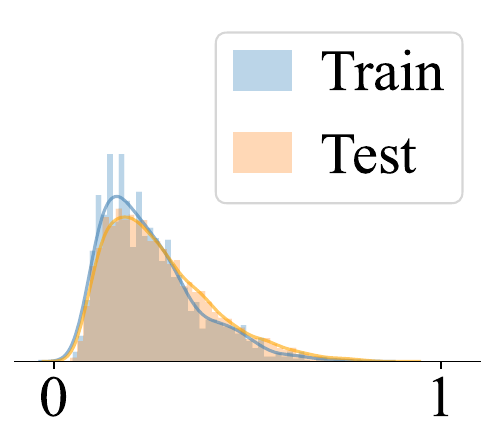}
}
\\
\subfloat[VPT-deep.]{
    \includegraphics[width=0.144\linewidth,trim=12 0 12 12]{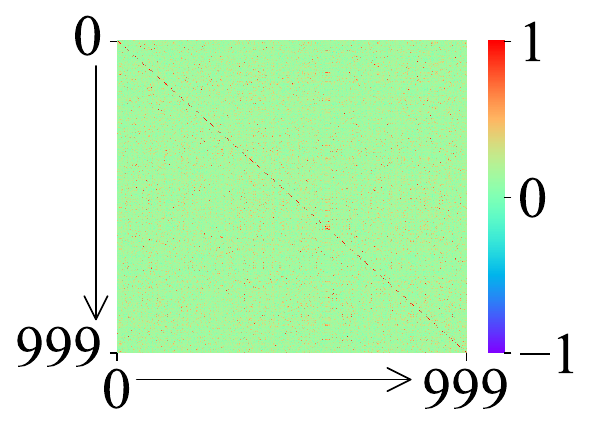}
    \includegraphics[width=0.144\linewidth,trim=12 12 12 -12]{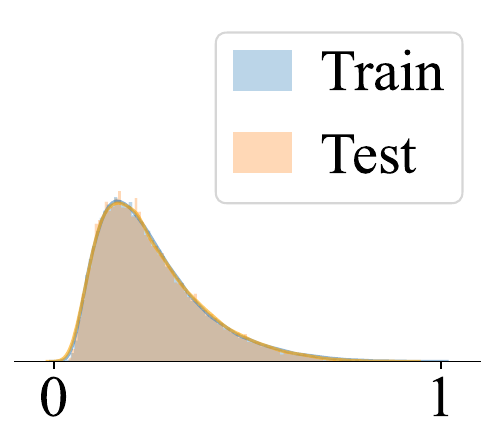}
    \includegraphics[width=0.144\linewidth,trim=12 12 12 -12]{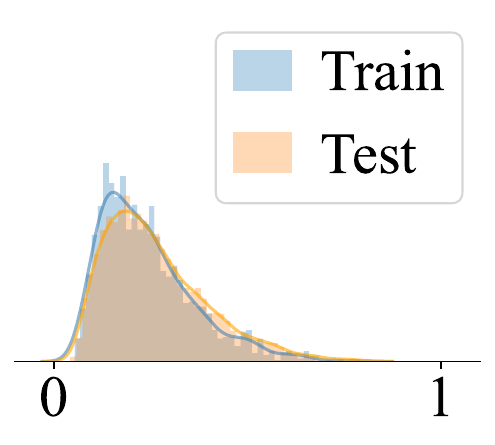}
}
\hfill
\subfloat[LoRA.]{
    \includegraphics[width=0.144\linewidth,trim=12 0 12 12]{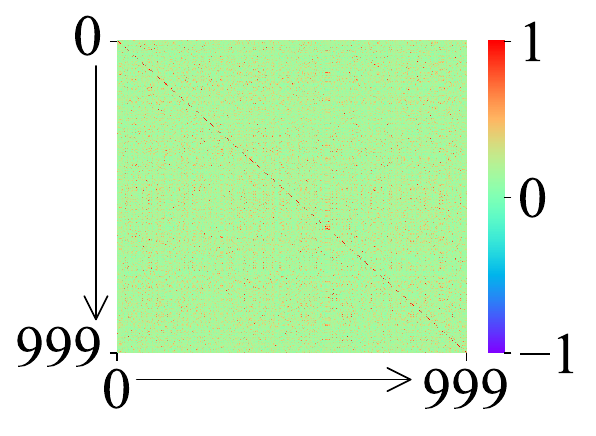}
    \includegraphics[width=0.144\linewidth,trim=12 12 12 -12]{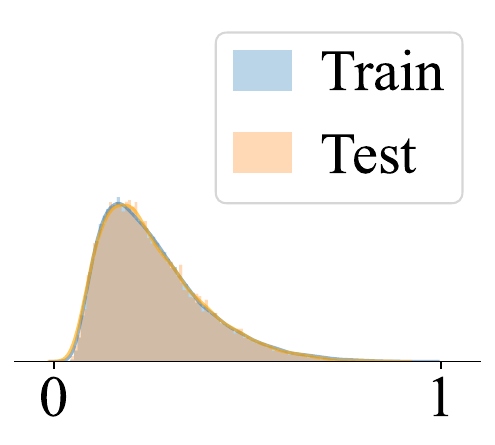}
    \includegraphics[width=0.144\linewidth,trim=12 12 12 -12]{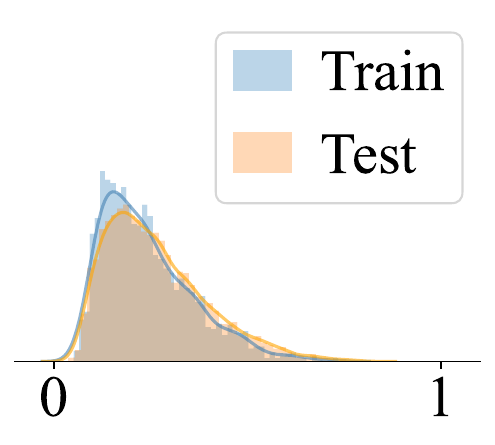}
}
\\
\subfloat[Adapter.]{
    \includegraphics[width=0.144\linewidth,trim=12 0 12 12]{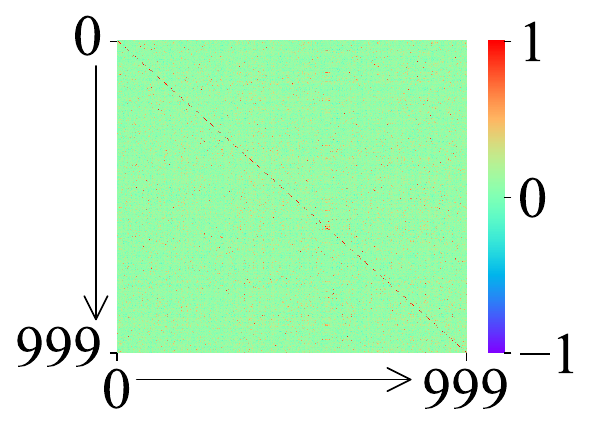}
    \includegraphics[width=0.144\linewidth,trim=12 12 12 -12]{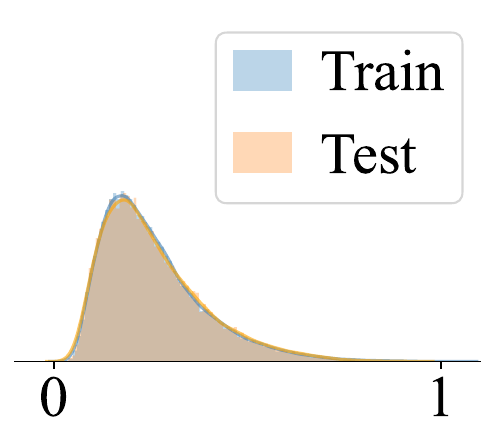}
    \includegraphics[width=0.144\linewidth,trim=12 12 12 -12]{figures/class_scatter_imagenet_lt_lp_adapter_head.pdf}
}
\hfill
\subfloat[AdaptFormer.]{
    \includegraphics[width=0.144\linewidth,trim=12 0 12 12]{figures/heatmap_imagenet_lt_lift.pdf}
    \includegraphics[width=0.144\linewidth,trim=12 12 12 -12]{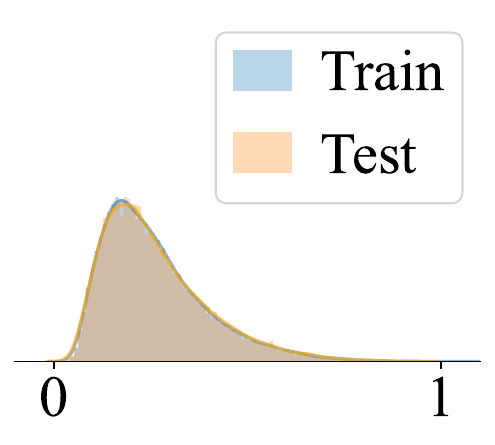}
    \includegraphics[width=0.144\linewidth,trim=12 12 12 -12]{figures/class_scatter_imagenet_lt_lift_tail.pdf}
}
\caption{Visualization of the inter-class feature similarities (the heatmaps) and intra-class distance distributions from head classes (the left histograms) and tail classes (the right histograms) on ImageNet-LT.}
\label{fig:heatmap-and-scatter-imagenetlt}
\end{figure}

\begin{figure}[!t]
\centering
\subfloat[CLIP.]{
    \includegraphics[width=0.144\linewidth,trim=12 0 12 12]{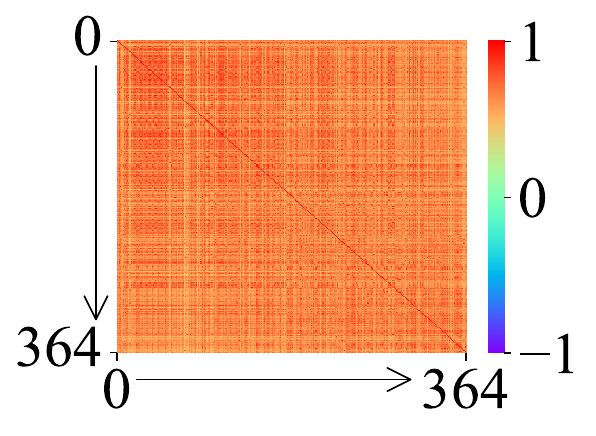}
    \includegraphics[width=0.144\linewidth,trim=12 12 12 -12]{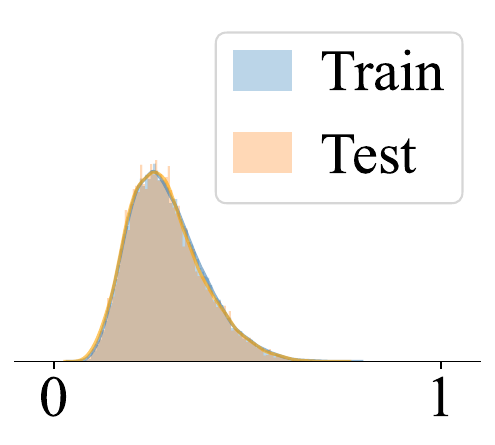}
    \includegraphics[width=0.144\linewidth,trim=12 12 12 -12]{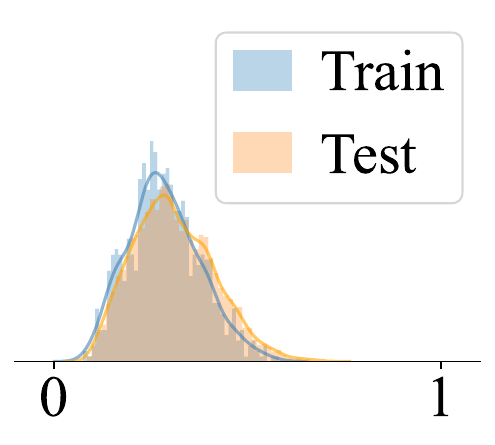}
}
\hfill
\subfloat[Full fine-tuning.]{
    \includegraphics[width=0.144\linewidth,trim=12 0 12 12]{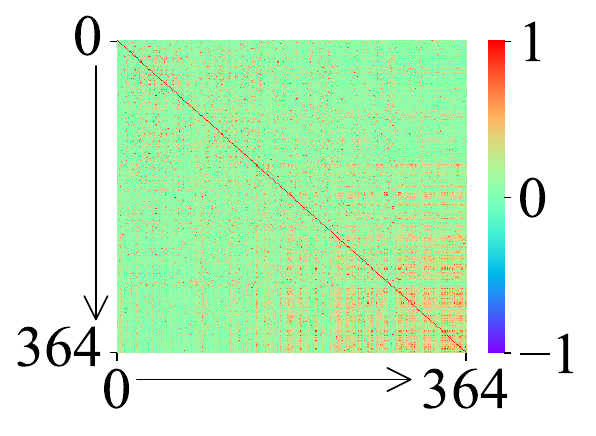}
    \includegraphics[width=0.144\linewidth,trim=12 12 12 -12]{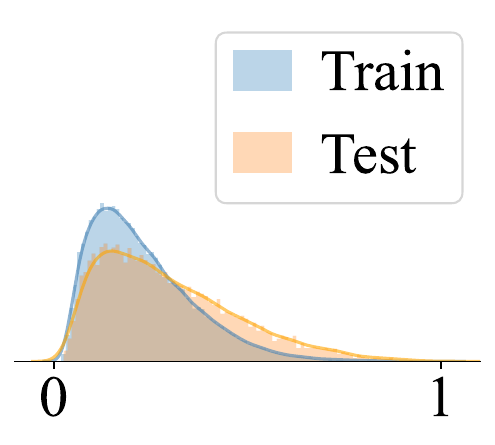}
    \includegraphics[width=0.144\linewidth,trim=12 12 12 -12]{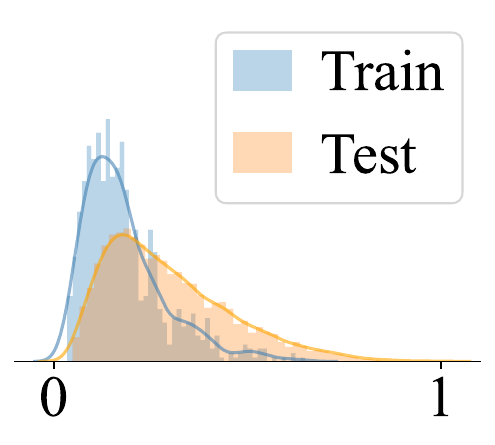}
}
\\
\subfloat[Arbitrary lightweight ft.]{
    \includegraphics[width=0.144\linewidth,trim=12 0 12 12]{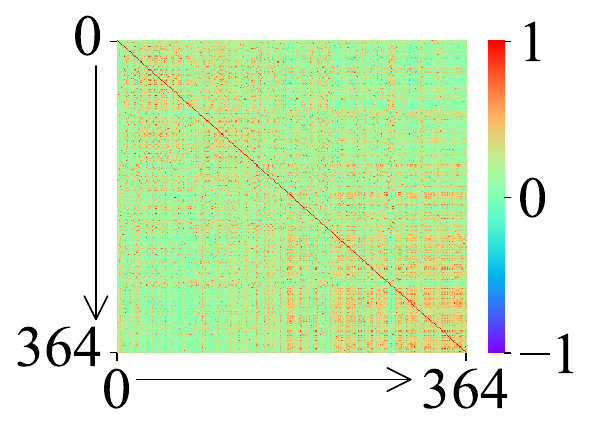}
    \includegraphics[width=0.144\linewidth,trim=12 12 12 -12]{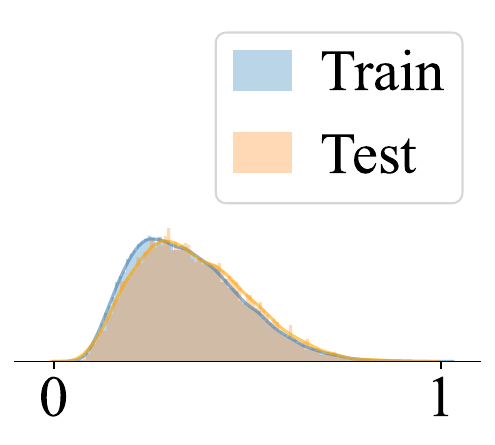}
    \includegraphics[width=0.144\linewidth,trim=12 12 12 -12]{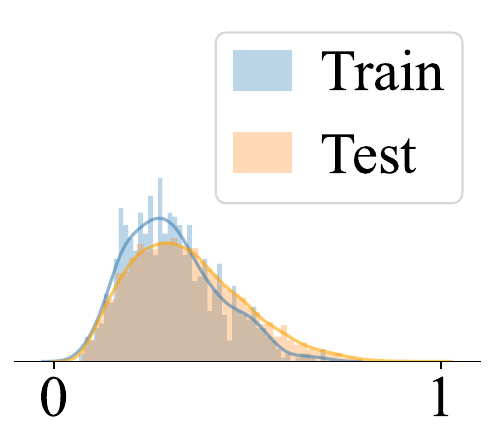}
}
\hfill
\subfloat[BitFit.]{
    \includegraphics[width=0.144\linewidth,trim=12 0 12 12]{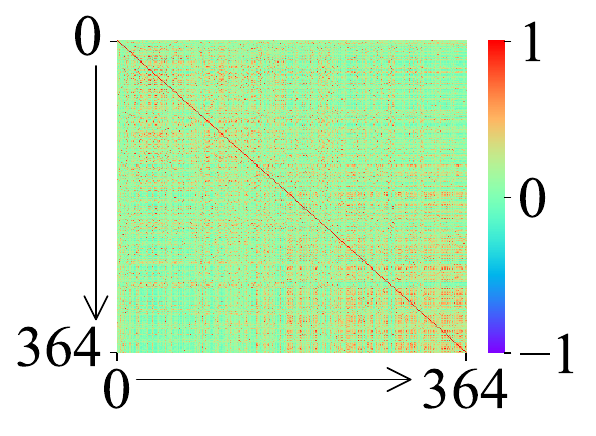}
    \includegraphics[width=0.144\linewidth,trim=12 12 12 -12]{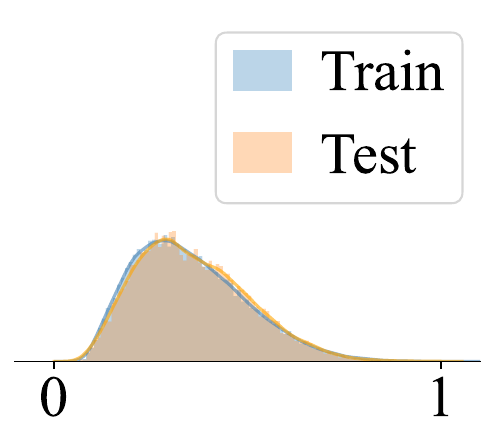}
    \includegraphics[width=0.144\linewidth,trim=12 12 12 -12]{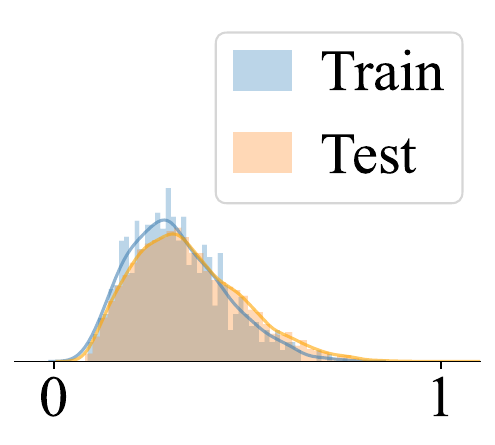}
}
\\
\subfloat[VPT-deep.]{
    \includegraphics[width=0.144\linewidth,trim=12 0 12 12]{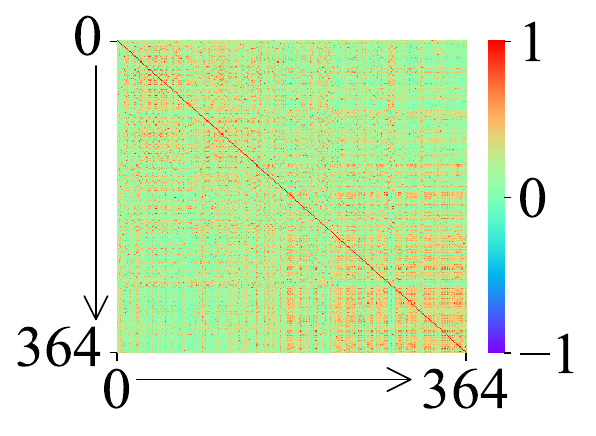}
    \includegraphics[width=0.144\linewidth,trim=12 12 12 -12]{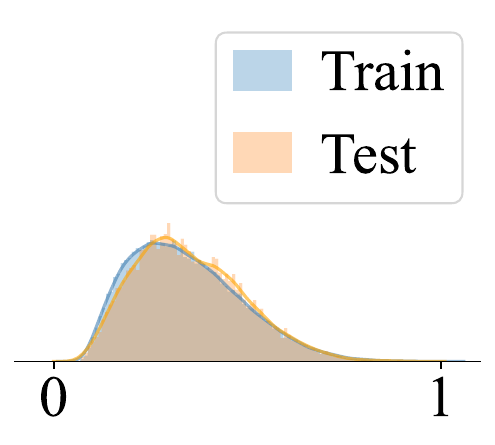}
    \includegraphics[width=0.144\linewidth,trim=12 12 12 -12]{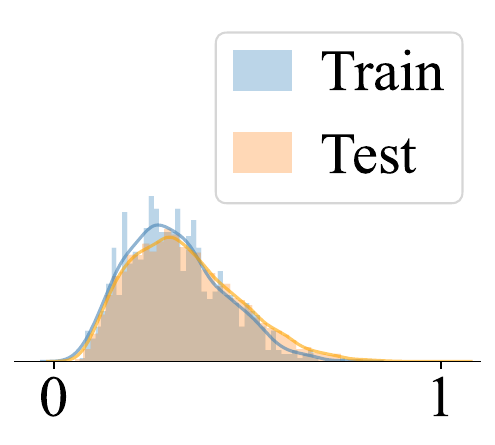}
}
\hfill
\subfloat[LoRA.]{
    \includegraphics[width=0.144\linewidth,trim=12 0 12 12]{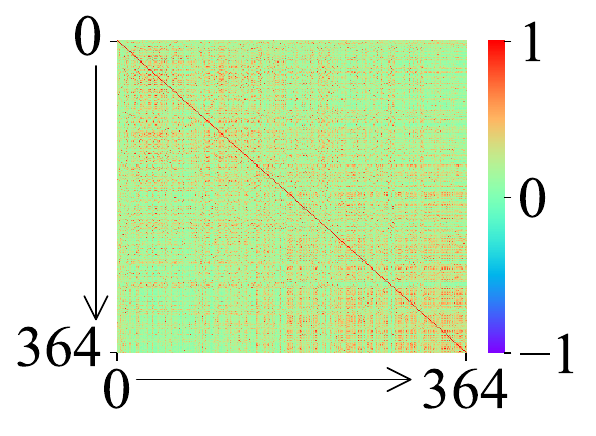}
    \includegraphics[width=0.144\linewidth,trim=12 12 12 -12]{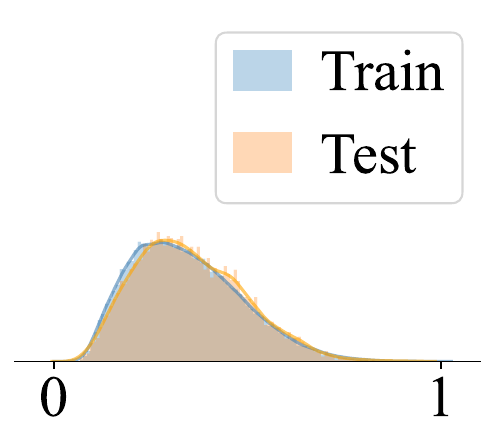}
    \includegraphics[width=0.144\linewidth,trim=12 12 12 -12]{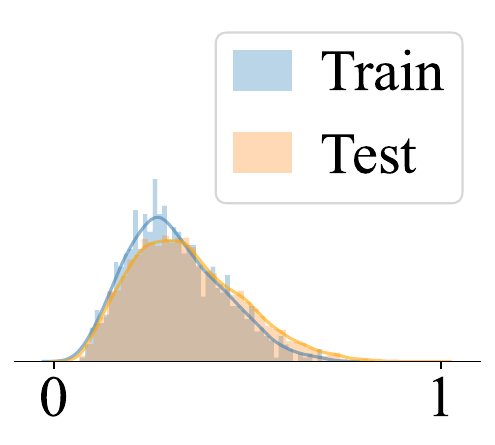}
}
\\
\subfloat[Adapter.]{
    \includegraphics[width=0.144\linewidth,trim=12 0 12 12]{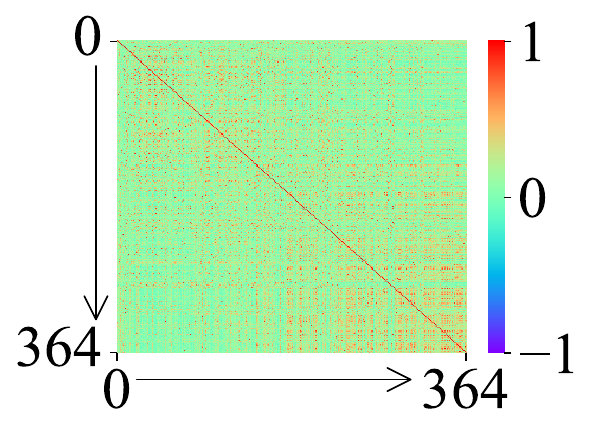}
    \includegraphics[width=0.144\linewidth,trim=12 12 12 -12]{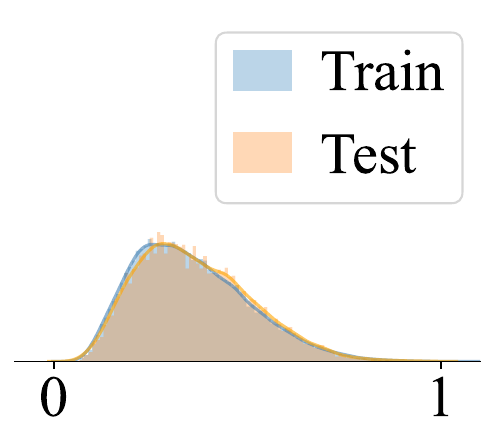}
    \includegraphics[width=0.144\linewidth,trim=12 12 12 -12]{figures/class_scatter_places_lt_lp_adapter_head.pdf}
}
\hfill
\subfloat[AdaptFormer.]{
    \includegraphics[width=0.144\linewidth,trim=12 0 12 12]{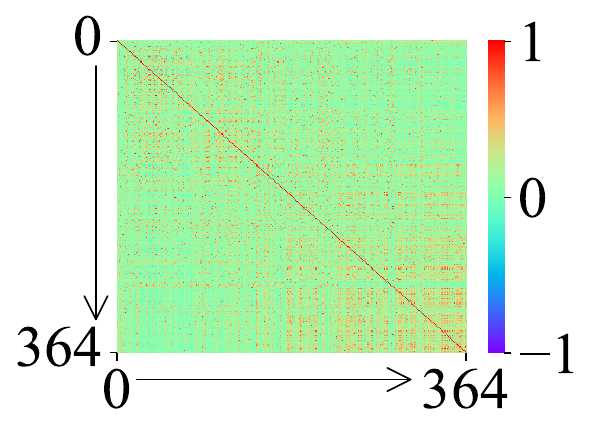}
    \includegraphics[width=0.144\linewidth,trim=12 12 12 -12]{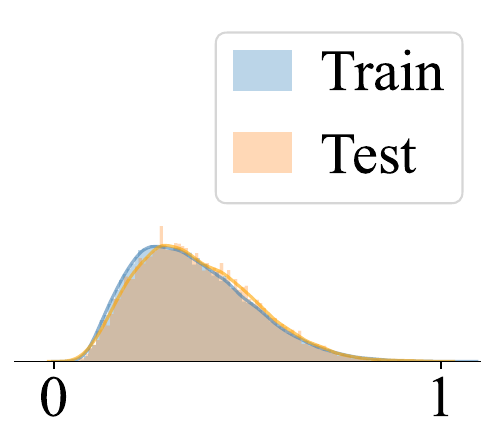}
    \includegraphics[width=0.144\linewidth,trim=12 12 12 -12]{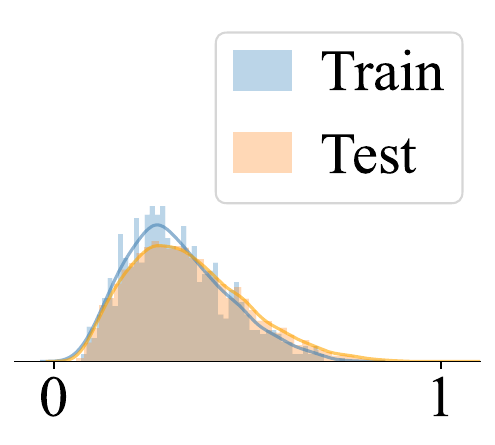}
}
\caption{Visualization of the inter-class feature similarities (the heatmaps) and intra-class distance distributions from head classes (the left histograms) and tail classes (the right histograms) on Places-LT.}
\label{fig:heatmap-and-scatter-placeslt}
\end{figure}

\textbf{Learned Representations of Different Fine-Tuning Methods.}
In \Cref{fig:heatmap-and-scatter-imagenetlt,fig:heatmap-and-scatter-placeslt}, we visualize the inter-class feature separabilities and intra-class distance distributions based on the representation learned by (a) original CLIP, (b) full fine-tuning, (c) arbitrary lightweight fine-tuning, and (d-i) structured lightweight fine-tuning methods.
Compared to the original CLIP, all of these lightweight fine-tuning methods yield more discriminative representations, achieving feature separability comparable to full fine-tuning. More importantly, in contrast to full fine-tuning, lightweight methods preserve undistorted intra-class distributions. For both head and tail classes, the features of training and test data exhibit nearly overlapping distributions. This property contributes to their stable performance improvements, particularly on tail classes, as demonstrated in \Cref{table:comp_peft_module}.

\begin{figure}[!t]
    \centering
    \includegraphics[width=0.32\linewidth,trim=12 0 12 0]{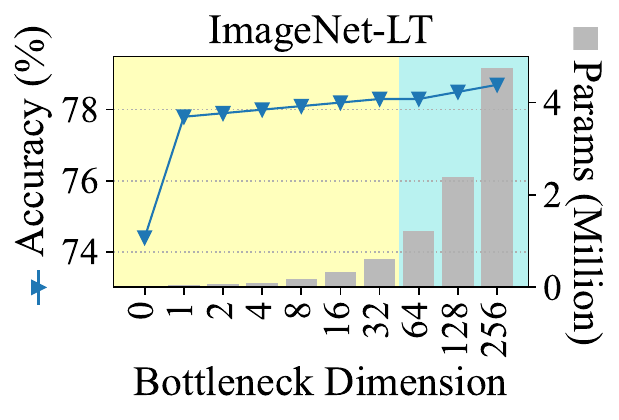}
    \hfill
    \includegraphics[width=0.32\linewidth,trim=12 0 12 0]{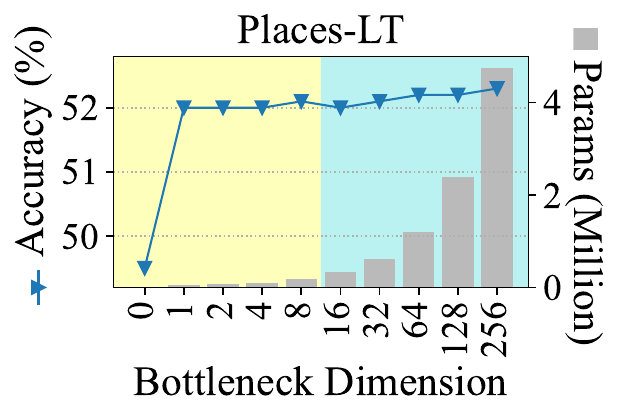}
    \hfill
    \includegraphics[width=0.32\linewidth,trim=12 0 12 0]{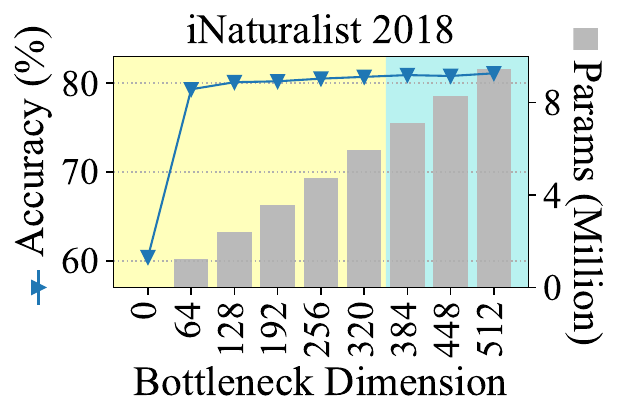}
    \caption{Comparative analysis of learned parameters by adjusting the bottleneck dimensionality $r$. In the yellow region, the incorporated module contains fewer learned parameters than the classifier. The blue region is just the opposite.}
    \label{fig:bottle_dim}
\end{figure}

\textbf{Impact of the Quantity of Learned Parameters.} In \algo, the number of learned parameters can be flexibly adjusted. We investigate the impact of parameter quantity by adjusting bottleneck dimensionality $r$ and present the results in \Cref{fig:bottle_dim}. The results reveal that the performance is robust to the adjustment of dimensionalities. Generally, when the lightweight parameters approximate the classifier parameters, the model achieves adequate improvements without incurring substantial computational overhead.

\begin{figure}[!t]
    \centering
    \includegraphics[width=0.325\linewidth]{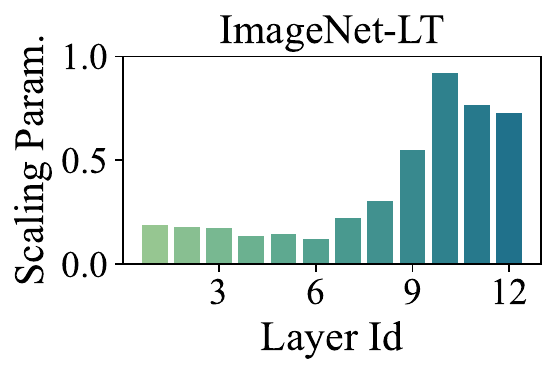}
    \hfill
    \includegraphics[width=0.325\linewidth]{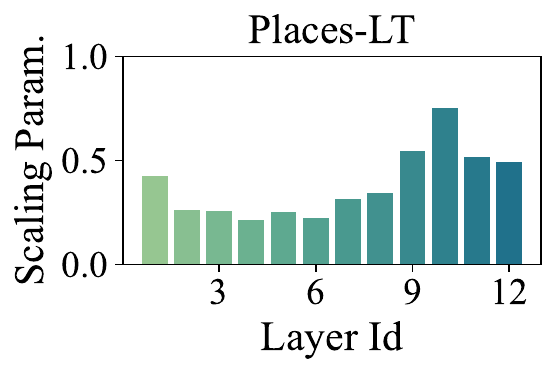}
    \hfill
    \includegraphics[width=0.313\linewidth]{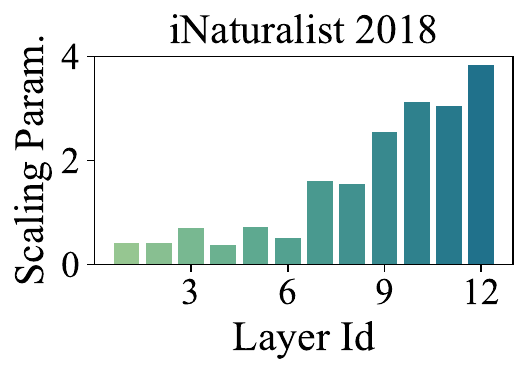}
    \caption{The learned scaling parameters of the AdaptFormer modules in different layers. \algo\ with AdaptFormer adaptively learns a distinct scaling parameter for each layer.}
    \label{fig:adaptformer_scale}
\end{figure}

\textbf{Advantage of \algo\ with AdaptFormer.} 
In each layer, the output of the AdaptFormer module can be multiplied by a learnable scaling parameter $s$ before being added to the corresponding block. Therefore, we can compare the values of $s$ to analyze the effects of the module for different layers. The comparison results are presented in \Cref{fig:adaptformer_scale}. It is inspiring that \algo\ with AdaptFormer adaptively learns suitable scaling parameters for different layers. For example, the values of the last layers tend to be larger, which indicates that the adaptations of the last several layers contribute more significantly to downstream classification tasks.

\begin{figure}[!t]
    \centering
    \includegraphics[width=0.9\linewidth]{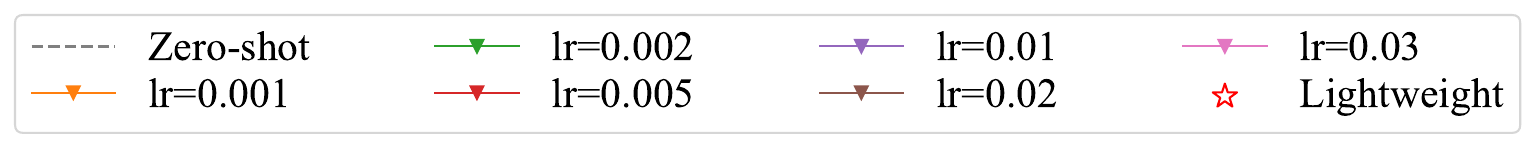}
    \\
    \includegraphics[width=0.32\linewidth]{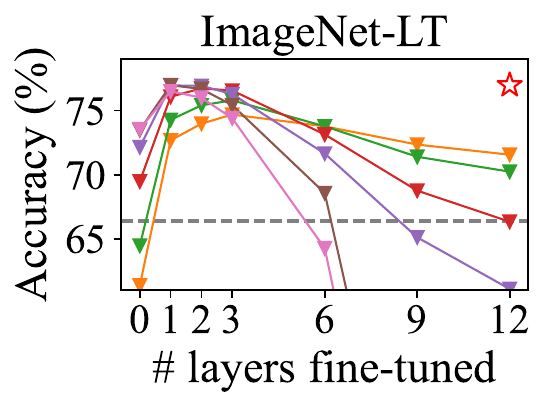}
    \hfill
    \includegraphics[width=0.32\linewidth]{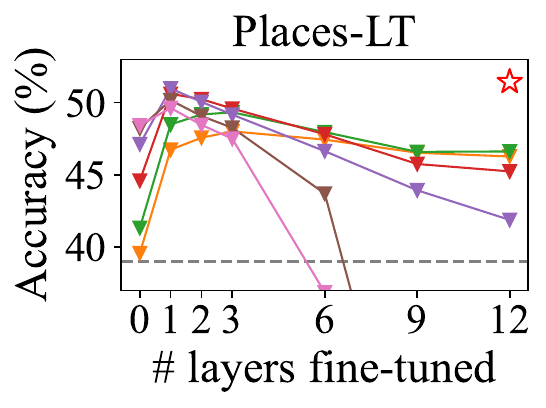}
    \hfill
    \includegraphics[width=0.32\linewidth]{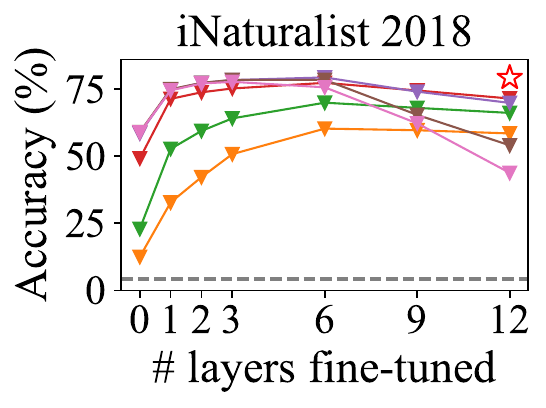}
    \hfill
    \caption{Partially fine-tuning the last $k$ layers.
    Similar to full fine-tuning, we search the learning rate from \{0.03, 0.02, 0.01, 0.005, 0.002, 0.001\}. For \algo, the learning rate is fixed to 0.02. The optimal learning rate and fine-tuned layers need to be elaborately selected for best performance. In contrast, \algo\ consistently demonstrates superior performance.}
    \label{fig:partial}
\end{figure}

\textbf{Lightweight Fine-Tuning vs. Partial Fine-Tuning.}
Partial fine-tuning \cite{he2022masked} offers an intuitive approach to reduce the number of learned parameters. Specifically, it fine-tunes only the last $k$ layers while keeping the remaining frozen. In \Cref{fig:partial}, we compare partial fine-tuning and lightweight fine-tuning (\algo) on ImageNet-LT, Places-LT, and iNaturalist 2018. 
Similar to full fine-tuning, partial fine-tuning is also sensitive to the learning rate. For shallow fine-tuned layers (\eg, $k=0, 1, 2$), higher learning rates produce better performance. When fine-tuning more layers (\eg, $k=9, 12$), high learning rates adversely lead to severe accuracy degradation. 
Moreover, identifying the optimal number of fine-tuned layers presents another non-trivial challenge, even if using the ideal learning rate, as the optimal $k$ varies dramatically across datasets: $1$ for ImageNet-LT and Places-LT, and $6$ for iNaturalist 2018. Remarkably, \algo\ maintains strong performance across all datasets without requiring hyperparameter adjustments.

\section{Conclusion and Limitation}
This paper investigates long-tail learning with the foundation model. We first identify the critical limitation of heavy fine-tuning in distorting tail-class performance and discover that lightweight fine-tuning can effectively mitigate this issue. Based on this insight, we propose \algo, a versatile and inclusive framework tailored for long-tail learning. The proposed framework facilitates efficient and accurate fine-tuning, achieving state-of-the-art performance in fewer than 15 training epochs without relying on any external data, while consistently outperforming numerous baseline methods across a range of long-tail datasets, including ImageNet-LT, Places-LT, iNaturalist 2018, and CIFAR-100-LT. We emphasize the ease of training and hope that our approach serves as an inspiration for further advancements in the field of long-tail learning.

\textbf{Limitation}: \algo\ adopts semantic-aware initialization by harnessing semantic knowledge of CLIP. However, it faces challenges when applied to vision-only foundation models. In such cases, we opt for class mean features as a practical alternative. Fortunately, the results on ImageNet-21K pre-trained ViT indicate that class mean features serve as a viable substitute. As delineated in \Cref{table:comp_clf_init}, this approach yields notable improvements compared to other initialization methods; however, its efficacy on tail classes remains inferior to that of semantic-aware strategies. It remains an intriguing challenge regarding how to exploit available information for classifier initialization in visual-only foundation models. We identify this as a promising direction for future research.





\ifCLASSOPTIONcaptionsoff
  \newpage
\fi



\normalem
\bibliographystyle{IEEEtran}
\bibliography{reference}

\end{document}

%% file: math_commands.tex
\usepackage{amsmath,amsfonts,bm}

\def\eqref#1{equation~\ref{#1}}

\def\1{\bm{1}}

\def\vb{{\bm{b}}}
\def\vc{{\bm{c}}}

\def\vf{{\bm{f}}}

\def\vt{{\bm{t}}}

\def\vw{{\bm{w}}}
\def\vx{{\bm{x}}}
\def\vy{{\bm{y}}}
\def\vz{{\bm{z}}}

\def\mE{{\bm{E}}}

\def\mK{{\bm{K}}}

\def\mM{{\bm{M}}}

\def\mP{{\bm{P}}}
\def\mQ{{\bm{Q}}}

\def\mV{{\bm{V}}}
\def\mW{{\bm{W}}}
\def\mX{{\bm{X}}}

\DeclareMathAlphabet{\mathsfit}{\encodingdefault}{\sfdefault}{m}{sl}
\SetMathAlphabet{\mathsfit}{bold}{\encodingdefault}{\sfdefault}{bx}{n}

\def\gL{{\mathcal{L}}}

\def\gO{{\mathcal{O}}}

\def\sR{{\mathbb{R}}}










%% file: arxiv.bbl
\begin{thebibliography}{10}
\providecommand{\url}[1]{#1}
\csname url@samestyle\endcsname
\providecommand{\newblock}{\relax}
\providecommand{\bibinfo}[2]{#2}
\providecommand{\BIBentrySTDinterwordspacing}{\spaceskip=0pt\relax}
\providecommand{\BIBentryALTinterwordstretchfactor}{4}
\providecommand{\BIBentryALTinterwordspacing}{\spaceskip=\fontdimen2\font plus
\BIBentryALTinterwordstretchfactor\fontdimen3\font minus \fontdimen4\font\relax}
\providecommand{\BIBforeignlanguage}[2]{{%
\expandafter\ifx\csname l@#1\endcsname\relax
\typeout{** WARNING: IEEEtran.bst: No hyphenation pattern has been}%
\typeout{** loaded for the language `#1'. Using the pattern for}%
\typeout{** the default language instead.}%
\else
\language=\csname l@#1\endcsname
\fi
#2}}
\providecommand{\BIBdecl}{\relax}
\BIBdecl

\bibitem{wang2017learning}
Y.-X. Wang, D.~Ramanan, and M.~Hebert, ``Learning to model the tail,'' in \emph{Advances in Neural Information Processing Systems}, vol.~30, 2017, pp. 7029--7039.

\bibitem{van2018inaturalist}
G.~Van~Horn, O.~Mac~Aodha, Y.~Song, Y.~Cui, C.~Sun, A.~Shepard, H.~Adam, P.~Perona, and S.~Belongie, ``The inaturalist species classification and detection dataset,'' in \emph{Proceedings of the IEEE Conference on Computer Vision and Pattern Recognition}, 2018, pp. 8769--8778.

\bibitem{liu2019large}
Z.~Liu, Z.~Miao, X.~Zhan, J.~Wang, B.~Gong, and S.~X. Yu, ``Large-scale long-tailed recognition in an open world,'' in \emph{Proceedings of the IEEE/CVF Conference on Computer Vision and Pattern Recognition}, 2019, pp. 2537--2546.

\bibitem{gupta2019lvis}
A.~Gupta, P.~Dollar, and R.~Girshick, ``{LVIS}: A dataset for large vocabulary instance segmentation,'' in \emph{Proceedings of the IEEE/CVF Conference on Computer Vision and Pattern Recognition}, 2019, pp. 5356--5364.

\bibitem{wang2020devil}
T.~Wang, Y.~Li, B.~Kang, J.~Li, J.~Liew, S.~Tang, S.~Hoi, and J.~Feng, ``The devil is in classification: A simple framework for long-tail instance segmentation,'' in \emph{Proceedings of the 16th European Conference on Computer Vision}, 2020, pp. 728--744.

\bibitem{wang2021seesaw}
J.~Wang, W.~Zhang, Y.~Zang, Y.~Cao, J.~Pang, T.~Gong, K.~Chen, Z.~Liu, C.~C. Loy, and D.~Lin, ``Seesaw loss for long-tailed instance segmentation,'' in \emph{Proceedings of the IEEE/CVF Conference on Computer Vision and Pattern Recognition}, 2021, pp. 9695--9704.

\bibitem{li2020overcoming}
Y.~Li, T.~Wang, B.~Kang, S.~Tang, C.~Wang, J.~Li, and J.~Feng, ``Overcoming classifier imbalance for long-tail object detection with balanced group softmax,'' in \emph{Proceedings of the IEEE/CVF Conference on Computer Vision and Pattern Recognition}, 2020, pp. 10\,991--11\,000.

\bibitem{wang2021adaptive}
T.~Wang, Y.~Zhu, C.~Zhao, W.~Zeng, J.~Wang, and M.~Tang, ``Adaptive class suppression loss for long-tail object detection,'' in \emph{Proceedings of the IEEE/CVF Conference on Computer Vision and Pattern Recognition}, 2021, pp. 3103--3112.

\bibitem{li2022equalized}
B.~Li, Y.~Yao, J.~Tan, G.~Zhang, F.~Yu, J.~Lu, and Y.~Luo, ``Equalized focal loss for dense long-tailed object detection,'' in \emph{Proceedings of the IEEE/CVF Conference on Computer Vision and Pattern Recognition}, 2022, pp. 6990--6999.

\bibitem{zhou2020bbn}
B.~Zhou, Q.~Cui, X.-S. Wei, and Z.-M. Chen, ``{BBN}: Bilateral-branch network with cumulative learning for long-tailed visual recognition,'' in \emph{Proceedings of the IEEE/CVF Conference on Computer Vision and Pattern Recognition}, 2020, pp. 9719--9728.

\bibitem{chou2020remix}
H.-P. Chou, S.-C. Chang, J.-Y. Pan, W.~Wei, and D.-C. Juan, ``{Remix}: Rebalanced mixup,'' in \emph{Proceedings of the 16th European Conference on Computer Vision}, 2020, pp. 95--110.

\bibitem{yang2020rethinking}
Y.~Yang and Z.~Xu, ``Rethinking the value of labels for improving class-imbalanced learning,'' in \emph{Advances in Neural Information Processing Systems}, vol.~33, 2020, pp. 19\,290--19\,301.

\bibitem{he2021distilling}
Y.-Y. He, J.~Wu, and X.-S. Wei, ``Distilling virtual examples for long-tailed recognition,'' in \emph{Proceedings of the IEEE/CVF International Conference on Computer Vision}, 2021, pp. 235--244.

\bibitem{park2022majority}
S.~Park, Y.~Hong, B.~Heo, S.~Yun, and J.~Y. Choi, ``The majority can help the minority: Context-rich minority oversampling for long-tailed classification,'' in \emph{Proceedings of the IEEE/CVF Conference on Computer Vision and Pattern Recognition}, 2022, pp. 6887--6896.

\bibitem{shi2023re}
J.-X. Shi, T.~Wei, Y.~Xiang, and Y.-F. Li, ``How re-sampling helps for long-tail learning?'' in \emph{Advances in Neural Information Processing Systems}, vol.~36, 2023, pp. 75\,669--75\,687.

\bibitem{ahn2023cuda}
S.~Ahn, J.~Ko, and S.-Y. Yun, ``{CUDA}: Curriculum of data augmentation for long-tailed recognition,'' in \emph{International Conference on Learning Representations}, 2023.

\bibitem{gao2023enhancing}
J.~Gao, H.~Zhao, Z.~Li, and D.~Guo, ``Enhancing minority classes by mixing: An adaptative optimal transport approach for long-tailed classification,'' in \emph{Advances in Neural Information Processing Systems}, vol.~36, 2023, pp. 60\,329--60\,348.

\bibitem{kang2021exploring}
B.~Kang, Y.~Li, S.~Xie, Z.~Yuan, and J.~Feng, ``Exploring balanced feature spaces for representation learning,'' in \emph{International Conference on Learning Representations}, 2021.

\bibitem{wang2021contrastive}
P.~Wang, K.~Han, X.-S. Wei, L.~Zhang, and L.~Wang, ``Contrastive learning based hybrid networks for long-tailed image classification,'' in \emph{Proceedings of the IEEE/CVF Conference on Computer Vision and Pattern Recognition}, 2021, pp. 943--952.

\bibitem{cui2021parametric}
J.~Cui, Z.~Zhong, S.~Liu, B.~Yu, and J.~Jia, ``Parametric contrastive learning,'' in \emph{Proceedings of the IEEE/CVF International Conference on Computer Vision}, 2021, pp. 715--724.

\bibitem{zhu2022balanced}
J.~Zhu, Z.~Wang, J.~Chen, Y.-P.~P. Chen, and Y.-G. Jiang, ``Balanced contrastive learning for long-tailed visual recognition,'' in \emph{Proceedings of the IEEE/CVF Conference on Computer Vision and Pattern Recognition}, 2022, pp. 6908--6917.

\bibitem{liu2022selfsupervised}
H.~Liu, J.~Z. HaoChen, A.~Gaidon, and T.~Ma, ``Self-supervised learning is more robust to dataset imbalance,'' in \emph{International Conference on Learning Representations}, 2022.

\bibitem{yang2022inducing}
Y.~Yang, S.~Chen, X.~Li, L.~Xie, Z.~Lin, and D.~Tao, ``Inducing neural collapse in imbalanced learning: Do we really need a learnable classifier at the end of deep neural network?'' in \emph{Advances in Neural Information Processing Systems}, vol.~35, 2022, pp. 37\,991--38\,002.

\bibitem{peifeng2023feature}
G.~Peifeng, Q.~Xu, P.~Wen, Z.~Yang, H.~Shao, and Q.~Huang, ``Feature directions matter: Long-tailed learning via rotated balanced representation,'' in \emph{Proceedings of the 40th International Conference on Machine Learning}, 2023, pp. 27\,542--27\,563.

\bibitem{ma2023curvature}
Y.~Ma, L.~Jiao, F.~Liu, S.~Yang, X.~Liu, and L.~Li, ``Curvature-balanced feature manifold learning for long-tailed classification,'' in \emph{Proceedings of the IEEE/CVF Conference on Computer Vision and Pattern Recognition}, 2023, pp. 15\,824--15\,835.

\bibitem{kukleva2023temperature}
A.~Kukleva, M.~B{\"o}hle, B.~Schiele, H.~Kuehne, and C.~Rupprecht, ``Temperature schedules for self-supervised contrastive methods on long-tail data,'' in \emph{International Conference on Learning Representations}, 2023.

\bibitem{gao2024distribution}
J.~Gao, H.~Zhao, D.~D. Guo, and H.~Zha, ``Distribution alignment optimization through neural collapse for long-tailed classification,'' in \emph{Proceedings of the 41st International Conference on Machine Learning}, 2024, pp. 14\,969--14\,987.

\bibitem{cao2019learning}
K.~Cao, C.~Wei, A.~Gaidon, N.~Arechiga, and T.~Ma, ``Learning imbalanced datasets with label-distribution-aware margin loss,'' in \emph{Advances in Neural Information Processing Systems}, vol.~32, 2019, pp. 1565--1576.

\bibitem{ren2020balanced}
J.~Ren, C.~Yu, s.~sheng, X.~Ma, H.~Zhao, S.~Yi, and h.~Li, ``Balanced meta-softmax for long-tailed visual recognition,'' in \emph{Advances in Neural Information Processing Systems}, vol.~33, 2020, pp. 4175--4186.

\bibitem{menon2021longtail}
A.~K. Menon, S.~Jayasumana, A.~S. Rawat, H.~Jain, A.~Veit, and S.~Kumar, ``Long-tail learning via logit adjustment,'' in \emph{International Conference on Learning Representations}, 2021.

\bibitem{hong2021disentangling}
Y.~Hong, S.~Han, K.~Choi, S.~Seo, B.~Kim, and B.~Chang, ``Disentangling label distribution for long-tailed visual recognition,'' in \emph{Proceedings of the IEEE/CVF Conference on Computer Vision and Pattern Recognition}, 2021, pp. 6626--6636.

\bibitem{zhang2021distribution}
S.~Zhang, Z.~Li, S.~Yan, X.~He, and J.~Sun, ``Distribution alignment: A unified framework for long-tail visual recognition,'' in \emph{Proceedings of the IEEE/CVF Conference on Computer Vision and Pattern Recognition}, 2021, pp. 2361--2370.

\bibitem{samuel2021distributional}
D.~Samuel and G.~Chechik, ``Distributional robustness loss for long-tail learning,'' in \emph{Proceedings of the IEEE/CVF International Conference on Computer Vision}, 2021, pp. 9495--9504.

\bibitem{wei2022robust}
T.~Wei, H.~Wang, W.-W. Tu, and Y.-F. Li, ``Robust model selection for positive and unlabeled learning with constraints,'' \emph{Science China Information Sciences}, vol.~65, no.~11, p. 212101, 2022.

\bibitem{han2023wrapped}
B.~Han, ``Wrapped cauchy distributed angular softmax for long-tailed visual recognition,'' in \emph{Proceedings of the 40th International Conference on Machine Learning}, 2023, pp. 12\,368--12\,388.

\bibitem{radford2021clip}
A.~Radford, J.~W. Kim, C.~Hallacy, A.~Ramesh, G.~Goh, S.~Agarwal, G.~Sastry, A.~Askell, P.~Mishkin, J.~Clark \emph{et~al.}, ``Learning transferable visual models from natural language supervision,'' in \emph{Proceedings of the 38th International Conference on Machine Learning}, 2021, pp. 8748--8763.

\bibitem{dosovitskiy2021an}
A.~Dosovitskiy, L.~Beyer, A.~Kolesnikov, D.~Weissenborn, X.~Zhai, T.~Unterthiner, M.~Dehghani, M.~Minderer, G.~Heigold, S.~Gelly, J.~Uszkoreit, and N.~Houlsby, ``An image is worth 16x16 words: Transformers for image recognition at scale,'' in \emph{International Conference on Learning Representations}, 2021.

\bibitem{ma2021simple}
T.~Ma, S.~Geng, M.~Wang, J.~Shao, J.~Lu, H.~Li, P.~Gao, and Y.~Qiao, ``A simple long-tailed recognition baseline via vision-language model,'' \emph{arXiv preprint arXiv:2111.14745}, 2021.

\bibitem{tian2022vl}
C.~Tian, W.~Wang, X.~Zhu, J.~Dai, and Y.~Qiao, ``{VL-LTR}: Learning class-wise visual-linguistic representation for long-tailed visual recognition,'' in \emph{Proceedings of the 17th European Conference on Computer Vision}, 2022, pp. 73--91.

\bibitem{wang2023exploring}
Y.~Wang, Z.~Yu, J.~Wang, Q.~Heng, H.~Chen, W.~Ye, R.~Xie, X.~Xie, and S.~Zhang, ``Exploring vision-language models for imbalanced learning,'' \emph{International Journal of Computer Vision}, vol. 132, pp. 224--237, 2024.

\bibitem{long2022retrieval}
A.~Long, W.~Yin, T.~Ajanthan, V.~Nguyen, P.~Purkait, R.~Garg, A.~Blair, C.~Shen, and A.~van~den Hengel, ``Retrieval augmented classification for long-tail visual recognition,'' in \emph{Proceedings of the IEEE/CVF Conference on Computer Vision and Pattern Recognition}, 2022, pp. 6959--6969.

\bibitem{dong2023lpt}
B.~Dong, P.~Zhou, S.~Yan, and W.~Zuo, ``{LPT}: Long-tailed prompt tuning for image classification,'' in \emph{International Conference on Learning Representations}, 2023.

\bibitem{shi2024longtail}
J.-X. Shi, T.~Wei, Z.~Zhou, J.-J. Shao, X.-Y. Han, and Y.-F. Li, ``Long-tail learning with foundation model: Heavy fine-tuning hurts,'' in \emph{Proceedings of the 41st International Conference on Machine Learning}, 2024, pp. 45\,014--45\,039.

\bibitem{yang2022survey}
L.~Yang, H.~Jiang, Q.~Song, and J.~Guo, ``A survey on long-tailed visual recognition,'' \emph{International Journal of Computer Vision}, vol. 130, no.~7, pp. 1837--1872, 2022.

\bibitem{zhang2023deep}
Y.~Zhang, B.~Kang, B.~Hooi, S.~Yan, and J.~Feng, ``Deep long-tailed learning: A survey,'' \emph{IEEE Transactions on Pattern Analysis and Machine Intelligence}, vol.~45, no.~9, pp. 10\,795--10\,816, 2023.

\bibitem{xiang2020learning}
L.~Xiang, G.~Ding, and J.~Han, ``Learning from multiple experts: Self-paced knowledge distillation for long-tailed classification,'' in \emph{Proceedings of the 16th European Conference on Computer Vision}, 2020, pp. 247--263.

\bibitem{wang2021longtailed}
X.~Wang, L.~Lian, Z.~Miao, Z.~Liu, and S.~Yu, ``Long-tailed recognition by routing diverse distribution-aware experts,'' in \emph{International Conference on Learning Representations}, 2021.

\bibitem{cui2022reslt}
J.~Cui, S.~Liu, Z.~Tian, Z.~Zhong, and J.~Jia, ``{ResLT}: Residual learning for long-tailed recognition,'' \emph{IEEE Transactions on Pattern Analysis and Machine Intelligence}, vol.~45, no.~3, pp. 3695--3706, 2022.

\bibitem{zhang2022self}
Y.~Zhang, B.~Hooi, L.~Hong, and J.~Feng, ``Self-supervised aggregation of diverse experts for test-agnostic long-tailed recognition,'' in \emph{Advances in Neural Information Processing Systems}, vol.~35, 2022, pp. 34\,077--34\,090.

\bibitem{shi2024residual}
J.-X. Shi, T.~Wei, and Y.-F. Li, ``Residual diverse ensemble for long-tailed multi-label text classification,'' \emph{SCIENCE CHINA Information Sciences}, 2024.

\bibitem{kang2020decoupling}
B.~Kang, S.~Xie, M.~Rohrbach, Z.~Yan, A.~Gordo, J.~Feng, and Y.~Kalantidis, ``Decoupling representation and classifier for long-tailed recognition,'' in \emph{International Conference on Learning Representations}, 2020.

\bibitem{zhong2021improving}
Z.~Zhong, J.~Cui, S.~Liu, and J.~Jia, ``Improving calibration for long-tailed recognition,'' in \emph{Proceedings of the IEEE/CVF Conference on Computer Vision and Pattern Recognition}, 2021, pp. 16\,489--16\,498.

\bibitem{wei2023towards}
T.~Wei and K.~Gan, ``Towards realistic long-tailed semi-supervised learning: Consistency is all you need,'' in \emph{Proceedings of the IEEE/CVF Conference on Computer Vision and Pattern Recognition}, 2023, pp. 3469--3478.

\bibitem{nam2023decoupled}
G.~Nam, S.~Jang, and J.~Lee, ``Decoupled training for long-tailed classification with stochastic representations,'' in \emph{International Conference on Learning Representations}, 2023.

\bibitem{steiner2022how}
A.~P. Steiner, A.~Kolesnikov, X.~Zhai, R.~Wightman, J.~Uszkoreit, and L.~Beyer, ``How to train your vit? data, augmentation, and regularization in vision transformers,'' \emph{Transactions on Machine Learning Research}, 2022.

\bibitem{zhou2022learning}
K.~Zhou, J.~Yang, C.~C. Loy, and Z.~Liu, ``Learning to prompt for vision-language models,'' \emph{International Journal of Computer Vision}, vol. 130, no.~9, pp. 2337--2348, 2022.

\bibitem{zhou2022conditional}
K.~Zhou, J.~Yang, C.~Loy, and Z.~Liu, ``Conditional prompt learning for vision-language models,'' in \emph{Proceedings of the IEEE/CVF Conference on Computer Vision and Pattern Recognition}, 2022, pp. 16\,816--16\,825.

\bibitem{yu2023visual}
B.~X. Yu, J.~Chang, H.~Wang, L.~Liu, S.~Wang, Z.~Wang, J.~Lin, L.~Xie, H.~Li, Z.~Lin, Q.~Tian, and C.~W. Chen, ``Visual tuning,'' \emph{ACM Computing Surveys}, vol.~56, no.~12, 2024.

\bibitem{zhou2024decoop}
Z.~Zhou, M.~Yang, J.-X. Shi, L.-Z. Guo, and Y.-F. Li, ``{DeCoOp}: Robust prompt tuning with out-of-distribution detection,'' in \emph{Proceedings of the 41st International Conference on Machine Learning}, 2024.

\bibitem{iscen2023improving}
A.~Iscen, A.~Fathi, and C.~Schmid, ``Improving image recognition by retrieving from web-scale image-text data,'' in \emph{Proceedings of the IEEE/CVF Conference on Computer Vision and Pattern Recognition}, 2023, pp. 19\,295--19\,304.

\bibitem{xia2023lmpt}
P.~Xia, D.~Xu, L.~Ju, M.~Hu, J.~Chen, and Z.~Ge, ``{LMPT}: Prompt tuning with class-specific embedding loss for long-tailed multi-label visual recognition,'' \emph{arXiv preprint arXiv:2305.04536}, 2023.

\bibitem{he2023uniformly}
X.~He, S.~Fu, X.~Ding, Y.~Cao, and H.~Wang, ``Uniformly distributed category prototype-guided vision-language framework for long-tail recognition,'' in \emph{Proceedings of the 31st ACM International Conference on Multimedia}, 2023, pp. 5027--5037.

\bibitem{song2023long}
Y.~Song, M.~Li, and B.~Wang, ``Long-tailed visual recognition via improved cross-window self-attention and trivialaugment,'' \emph{IEEE Access}, vol.~11, pp. 49\,601--49\,610, 2023.

\bibitem{li2024rectify}
B.~Li, Y.~Yao, J.~Tan, R.~Gong, J.~Lu, and Y.~Luo, ``Rectify representation bias in vision-language models for long-tailed recognition,'' \emph{Neural Networks}, vol. 172, p. 106134, 2024.

\bibitem{vaswani2017attention}
A.~Vaswani, N.~Shazeer, N.~Parmar, J.~Uszkoreit, L.~Jones, A.~N. Gomez, {\L}.~Kaiser, and I.~Polosukhin, ``Attention is all you need,'' in \emph{Advances in Neural Information Processing Systems}, vol.~30, 2017, pp. 5998--6008.

\bibitem{he2022masked}
K.~He, X.~Chen, S.~Xie, Y.~Li, P.~Doll\'ar, and R.~Girshick, ``Masked autoencoders are scalable vision learners,'' in \emph{Proceedings of the IEEE/CVF Conference on Computer Vision and Pattern Recognition}, 2022, pp. 16\,000--16\,009.

\bibitem{kenton2019bert}
J.~Devlin, M.~Chang, K.~Lee, and K.~Toutanova, ``{BERT}: Pre-training of deep bidirectional transformers for language understanding,'' in \emph{Proceedings of NAACL-HLT}, 2019, pp. 4171--4186.

\bibitem{zhao2022adaptive}
Y.~Zhao, W.~Chen, X.~Tan, K.~Huang, and J.~Zhu, ``Adaptive logit adjustment loss for long-tailed visual recognition,'' in \emph{Proceedings of the AAAI Conference on Artificial Intelligence}, vol.~36, no.~3, 2022, pp. 3472--3480.

\bibitem{li2022long}
M.~Li, Y.-m. Cheung, and Y.~Lu, ``Long-tailed visual recognition via gaussian clouded logit adjustment,'' in \emph{Proceedings of the IEEE/CVF Conference on Computer Vision and Pattern Recognition}, 2022, pp. 6929--6938.

\bibitem{zaken2022bitfit}
E.~B. Zaken, Y.~Goldberg, and S.~Ravfogel, ``{BitFit}: Simple parameter-efficient fine-tuning for transformer-based masked language-models,'' in \emph{Proceedings of the 60th Annual Meeting of the Association for Computational Linguistics (Volume 2: Short Papers)}, 2022, pp. 1--9.

\bibitem{jia2022visual}
M.~Jia, L.~Tang, B.-C. Chen, C.~Cardie, S.~Belongie, B.~Hariharan, and S.-N. Lim, ``Visual prompt tuning,'' in \emph{Proceedings of the 17th European Conference on Computer Vision}, 2022, pp. 709--727.

\bibitem{houlsby2019parameter}
N.~Houlsby, A.~Giurgiu, S.~Jastrzebski, B.~Morrone, Q.~De~Laroussilhe, A.~Gesmundo, M.~Attariyan, and S.~Gelly, ``Parameter-efficient transfer learning for {NLP},'' in \emph{Proceedings of the 36th International Conference on Machine Learning}, 2019, pp. 2790--2799.

\bibitem{hu2022lora}
E.~J. Hu, yelong shen, P.~Wallis, Z.~Allen-Zhu, Y.~Li, S.~Wang, L.~Wang, and W.~Chen, ``{LoRA}: Low-rank adaptation of large language models,'' in \emph{International Conference on Learning Representations}, 2022.

\bibitem{chen2022adaptformer}
S.~Chen, C.~GE, Z.~Tong, J.~Wang, Y.~Song, J.~Wang, and P.~Luo, ``{AdaptFormer}: Adapting vision transformers for scalable visual recognition,'' in \emph{Advances in Neural Information Processing Systems}, vol.~35, 2022, pp. 16\,664--16\,678.

\bibitem{wei2021towards}
T.~Wei, W.-W. Tu, Y.-F. Li, and G.-P. Yang, ``Towards robust prediction on tail labels,'' in \emph{Proceedings of the 27th ACM SIGKDD Conference on Knowledge Discovery \& Data Mining}, 2021, pp. 1812--1820.

\bibitem{kumar2022finetuning}
A.~Kumar, A.~Raghunathan, R.~M. Jones, T.~Ma, and P.~Liang, ``Fine-tuning can distort pretrained features and underperform out-of-distribution,'' in \emph{International Conference on Learning Representations}, 2022.

\bibitem{simonyan2014very}
K.~Simonyan and A.~Zisserman, ``Very deep convolutional networks for large-scale image recognition,'' in \emph{International Conference on Learning Representations}, 2015.

\bibitem{he2016deep}
K.~He, X.~Zhang, S.~Ren, and J.~Sun, ``Deep residual learning for image recognition,'' in \emph{Proceedings of the IEEE Conference on Computer Vision and Pattern Recognition}, 2016, pp. 770--778.

\bibitem{cubuk2019autoaugment}
E.~D. Cubuk, B.~Zoph, D.~Mane, V.~Vasudevan, and Q.~V. Le, ``{AutoAugment}: Learning augmentation strategies from data,'' in \emph{Proceedings of the IEEE/CVF Conference on Computer Vision and Pattern Recognition}, 2019, pp. 113--123.

\bibitem{cubuk2020randaugment}
E.~D. Cubuk, B.~Zoph, J.~Shlens, and Q.~V. Le, ``{RandAugment}: Practical automated data augmentation with a reduced search space,'' in \emph{Advances in Neural Information Processing Systems}, vol.~33, 2020, pp. 18\,613--18\,624.

\bibitem{sun2020test}
Y.~Sun, X.~Wang, Z.~Liu, J.~Miller, A.~Efros, and M.~Hardt, ``Test-time training with self-supervision for generalization under distribution shifts,'' in \emph{Proceedings of the 37th International Conference on Machine Learning}, 2020, pp. 9229--9248.

\bibitem{wang2021tent}
D.~Wang, E.~Shelhamer, S.~Liu, B.~Olshausen, and T.~Darrell, ``{Tent}: Fully test-time adaptation by entropy minimization,'' in \emph{International Conference on Learning Representations}, 2021.

\bibitem{zhou2023ods}
Z.~Zhou, L.-Z. Guo, L.-H. Jia, D.~Zhang, and Y.-F. Li, ``{ODS}: Test-time adaptation in the presence of open-world data shift,'' in \emph{Proceedings of the 40th International Conference on Machine Learning}, 2023, pp. 42\,574--42\,588.

\bibitem{snell2024scaling}
C.~Snell, J.~Lee, K.~Xu, and A.~Kumar, ``Scaling llm test-time compute optimally can be more effective than scaling model parameters,'' \emph{arXiv preprint arXiv:2408.03314}, 2024.

\bibitem{chen2024expanding}
Z.~Chen, W.~Wang, Y.~Cao, Y.~Liu, Z.~Gao, E.~Cui, J.~Zhu, S.~Ye, H.~Tian, Z.~Liu \emph{et~al.}, ``Expanding performance boundaries of open-source multimodal models with model, data, and test-time scaling,'' \emph{arXiv preprint arXiv:2412.05271}, 2024.

\bibitem{li2022nested}
J.~Li, Z.~Tan, J.~Wan, Z.~Lei, and G.~Guo, ``Nested collaborative learning for long-tailed visual recognition,'' in \emph{Proceedings of the IEEE/CVF Conference on Computer Vision and Pattern Recognition}, 2022, pp. 6949--6958.

\bibitem{xu2023learning}
Z.~Xu, R.~Liu, S.~Yang, Z.~Chai, and C.~Yuan, ``Learning imbalanced data with vision transformers,'' in \emph{Proceedings of the IEEE/CVF Conference on Computer Vision and Pattern Recognition}, 2023, pp. 15\,793--15\,803.

\bibitem{suh2023long}
M.-K. Suh and S.-W. Seo, ``Long-tailed recognition by mutual information maximization between latent features and ground-truth labels,'' in \emph{Proceedings of the 40th International Conference on Machine Learning}, vol. 202, 2023, pp. 32\,770--32\,782.

\bibitem{lin2017focal}
T.-Y. Lin, P.~Goyal, R.~Girshick, K.~He, and P.~Doll{\'a}r, ``Focal loss for dense object detection,'' in \emph{Proceedings of the IEEE International Conference on Computer Vision}, 2017, pp. 2980--2988.

\bibitem{cui2019class}
Y.~Cui, M.~Jia, T.-Y. Lin, Y.~Song, and S.~Belongie, ``Class-balanced loss based on effective number of samples,'' in \emph{Proceedings of the IEEE/CVF Conference on Computer Vision and Pattern Recognition}, 2019, pp. 9268--9277.

\bibitem{lian2022scaling}
D.~Lian, D.~Zhou, J.~Feng, and X.~Wang, ``Scaling \&amp; shifting your features: A new baseline for efficient model tuning,'' in \emph{Advances in Neural Information Processing Systems}, vol.~35, 2022, pp. 109--123.

\bibitem{ba2016layer}
J.~L. Ba, J.~R. Kiros, and G.~E. Hinton, ``Layer normalization,'' \emph{arXiv preprint arXiv:1607.06450}, 2016.

\end{thebibliography}
